\documentclass[journal]{IEEEtran}
\usepackage{hyperref}       % hyperlinks
\usepackage{url}            % simple URL typesetting
\usepackage{booktabs}       % professional-quality tables
\usepackage{amsfonts}       % blackboard math symbols
\usepackage{nicefrac}       % compact symbols for 1/2, etc.
\usepackage{microtype}      % microtypography
\usepackage{amsmath}
\usepackage{amsthm}
\usepackage{multirow}
\usepackage{amssymb}
\usepackage{bm} 
\usepackage{graphicx}
\usepackage{subfigure}
\usepackage{colortbl,booktabs}
\usepackage{algorithm}
\usepackage[noend]{algpseudocode}
\usepackage{float}
\usepackage[margin=1.9cm]{geometry}
\usepackage{xcolor}
\definecolor{mygray}{gray}{.9}
\newtheorem{theorem}{Theorem}
\newtheorem{prop}{Proposition}
\newtheorem{remark}{Remark}
\newtheorem{lemma}{Lemma}

\begin{document}
\title{A universal framework for learning the elliptical mixture model}
\author{Shengxi~Li,~\IEEEmembership{Student~Member,~IEEE}, Zeyang~Yu,~Danilo~Mandic,~\IEEEmembership{Fellow,~IEEE}
\thanks{This is the author's version of an article that has been published accepted to IEEE Transactions on Neural Networks and Learning Systems with DOI: 10.1109/TNNLS.2020.3010198. Changes were made to this version by the publisher prior to publication. Please note that personal use is permitted. For any other purposes, permission must be obtained from the IEEE by emailing pubs-permissions@ieee.org. Shengxi Li, Zeyang Yu and Danilo Mandic are with the department of Electrical and Electronic of Imperial College London.}}
\maketitle
\begin{abstract}
	Mixture modelling using elliptical distributions promises enhanced robustness, flexibility and stability over the widely employed Gaussian mixture model (GMM). However, existing studies based on the elliptical mixture model (EMM) are restricted to several specific types of elliptical probability density functions, which are not supported by general solutions or systematic analysis frameworks; this significantly limits the rigour in the design and power of EMMs in applications. To this end, we propose a novel general framework for estimating and analysing the EMMs, achieved through Riemannian manifold optimisation. First, we investigate the relationships between Riemannian manifolds and elliptical distributions, and the so established connection between the original manifold and a reformulated one indicates a mismatch between these manifolds, a major cause of failure of the existing optimisation for solving general EMMs. We next propose a universal solver which is based on the optimisation of a re-designed cost and prove the existence of the same optimum as in the original problem; this is achieved in a simple, fast and stable way. We further calculate the influence functions of the EMM as theoretical bounds to quantify robustness to outliers. Comprehensive numerical results demonstrate the ability of the proposed framework to accommodate EMMs with different properties of individual functions in a stable way and with fast convergence speed. Finally, the enhanced robustness and flexibility of the proposed framework over the standard GMM are demonstrated both analytically and through comprehensive simulations.
\end{abstract}

\section{Introduction}
Finite mixture models have a prominent role in statistical machine learning, owing to their ability to enhance probabilistic awareness in many learning paradigms, including clustering, feature extraction and density estimation \cite{figueiredo2002unsupervised}. Among such models, the Gaussian mixture model (GMM) is the most widely used, with its popularity stemming from a simple formulation and the conjugate property of Gaussian distribution. Despite mathematical elegance, a standard GMM estimator is subject to robustness issues, as even a slight deviation from the Gaussian assumption or a single outlier in data can significantly degrade the performance or even break down the estimator \cite{zoubir2012robust}. Another issue with GMMs is their limited flexibility, which is prohibitive to their application in rapidly emerging scenarios based on multi-faceted data which are almost invariably unbalanced; sources of such imbalance may be due to different natures of the data channels involved, different powers in the constitutive channels, or temporal misalignment \cite{mandic2005data}.

An important class of flexible multivariate analysis techniques are elliptical distributions, which are quite general and include as special cases a range of standard distributions, such as the Gaussian distribution, the logistic distribution and the $t$-distribution \cite{fang2018symmetric}. The desired robustness to unbalanced multichannel data is naturally catered for in elliptical distributions; indeed estimating certain elliptical distribution types results in robust M-estimators \cite{huber2011robust}, thus making them a natural candidate for robust and flexible mixture modelling. In this work, we therefore consider mixtures of elliptical distributions, or elliptical mixture model (EMM), in probabilistic modelling. By virtue of the inherent flexibility of EMMs, it is possible to model a wide range of standard distributions under one umbrella, as EMM components may exhibit different properties, which makes EMMs both more suitable for capturing intrinsic data structures and more meaningful in interpreting data, as compared to the GMM. Another appealing property of EMMs is their identifiability in mixture problems, which has been proved by Holzmann \textit{et al.} \cite{holzmann2006identifiability}. In addition, it has also been reported that several members of the EMM family can effectively mitigate the singular covariance problem experienced in the GMM \cite{peel2000robust}. 

Existing mixture models related to elliptical distributions are most frequently based on the $t$-distribution \cite{peel2000robust, andrews2012model, lin2014capturing}, the Laplace distribution \cite{tan2007multivariate}, and the hyperbolic distribution \cite{browne2015mixture}; these are optimised by a specific generalised expectation-maximisation process, called the iteratively reweighting algorithm (IRA) \cite{kent1991redescending}. These elliptical distributions belong to the class of  \textit{scale mixture of normals} \cite{andrews1974scale}, where the IRA actually operates as an expectation maximisation (EM) algorithm, and such an EMM model is guaranteed to converge. However, for other types of elliptical distributions, the convergence of the IRA requires constraints on both the type of elliptical distributions and the data structure \cite{kent1991redescending, zhang2013multivariate, sra2013geometric}. Therefore, although beneficial and promising, the development of a universal method for estimating the EMM is non-trivial, owing to both theoretical and practical difficulties. 

To this end, we set out to rigorously establish a whole new framework for estimating and analysing the identifiable EMMs, thus opening an avenue for practical approaches based on general EMMs. More specifically, we first analyse the second-order statistical differential tensors to obtain the Riemannian metrics on the mean and the covariance of elliptical distributions. A reformulation trick is typically used to convert the mean-covariance-estimation problem into a covariance-estimation-only problem \cite{kent1991redescending}. We further investigate the relationship between the manifolds with and without the reformulation, and find that an equivalence of the two manifolds only holds in the Gaussian case. Since for general elliptical distributions, the equivalence is not guaranteed, this means that a direct optimisation on the reformulated manifold cannot yield the optimum for all EMMs. To overcome this issue, we propose a novel method with a modified cost of EMMs and optimise it on a matched Riemannian manifold via manifold gradient descent, where the same optimum as in the original problem is achieved in a fast and stable manner. The corresponding development of a gradient-based solver, rather than the EM-type solver (i.e., the IRA), is shown to be beneficial, as it offers more flexibility in model design, through various components and regularisations. We should point out that even for the EMMs where the IRA converges, our proposed method still outperforms the widely employed IRA. We finally systematically verify the robustness of EMMs by proving the influence functions (IFs) in closed-form, which serves as the theoretical bound.

The recent related work in \cite{hosseini2015matrix, hosseini2017alternative} adopts manifold optimisation for GMM problems by simply optimising on the intrinsic manifold. However, this strategy is inadequate in EMM problems due to a mismatch in manifolds for optimisation, which leads to a different optimum after reformulation. More importantly, as the flexibility of EMMs allows for inclusion of a wide range of distributions, this in turn requires the statistics of mixture modelling to be considered in the optimisation, whilst the work \cite{hosseini2015matrix, hosseini2017alternative} only starts from a manifold optimisation perspective.
The key contributions of this work are summarised as follows:

1) We justify the usage of Riemannian metrics from the statistics of elliptical distributions, and in this way connect the original manifold with the reformulated one, where the convergence can be highly accelerated.

2) A novel method for accurately solving general EMMs in a fast and stable manner is proposed, thus making the flexible EMM truly practically applicable.

3) We rigorously prove the IFs in closed-form as theoretical bounds to qualify the robustness of EMMs, thus providing a systematic framework for treating the flexibility of EMMs.

\section{Preliminaries and related works}
As our aim is to solve the EMMs from the perspective of manifold optimisation, we first provide the preliminaries on the manifold related to probability distributions in Section \ref{preMO}. Then, we introduce the preliminaries and notations of the elliptical distributions in Section \ref{preED}. We finally review the related EMM works in Section \ref{relatedworks}.

\subsection{Preliminaries on the Riemannian manifold}\label{preMO}
A Riemannian manifold ($\mathcal{M}$, $\rho$) is a smooth (differential) manifold $\mathcal{M}$ (i.e., locally homeomorphic to the Euclidean space) which is equipped with a smoothly varying inner product $\rho$ on its tangent space. The inner product also defines a Riemannian metric on the tangent space, so that the length of a curve and the angle between two vectors can be correspondingly defined. Curves on the manifold with the shortest paths are called \textit{geodesics}, which exhibit constant instantaneous speed and generalise straight lines in the Euclidean space. The distance between two points on $\mathcal{M}$ is defined as the minimum length of all geodesics connecting these two points. 

We shall use the symbol $T_{\mathbf{\Sigma}}\mathcal{M}$ to denote the \textit{tangent space} at the point $\mathbf{\Sigma}$, which is the first-order approximation of $\mathcal{M}$ at $\mathbf{\Sigma}$. Consequently, the \textit{Riemannian gradient} of a function $f$ is defined with regard to the equivalence between its inner product with an arbitrary vector $\xi$ on $T_{\mathbf{\Sigma}}\mathcal{M}$ and the Fr\'echet derivative of $f$ at $\xi$. Moreover, a smooth mapping from $T_{\mathbf{\Sigma}}\mathcal{M}$ into $\mathcal{M}$ is called the \textit{retraction}, whereby an exponential mapping obtains the point on geodesics in the direction of the tangent space. Because the tangent spaces vary across different points on $\mathcal{M}$, \textit{parallel transport} across different tangent spaces can be introduced on the basis of the Levi-Civita connection, which preserves the inner product and the norm. In this way, we can convert a complex optimisation problem on $\mathcal{M}$ into a more analysis friendly space, i.e.,  $T_{\mathbf{\Sigma}}\mathcal{M}$. For a comprehensive text on the optimisation on the Riemannian manifold, we refer to \cite{absil2009optimization}.
Therefore, on the basis of the above basic operations, the manifold optimisation can be performed by the Riemannian gradient descent \cite{bonnabel2013stochastic}. The retraction is then utilised to map a step descent from the tangent space to the manifold. To accelerate gradient descent optimisation, the parallel transport can also be utilised to accumulate the first-order moments \cite{liu2017accelerated, zhang2016riemannian, zhang2016first, reddi2016stochastic, zhang2018towards}.
%Although there are various metrics designed for measuring the distance between matrices \cite{jeuris2012survey, sra2012new, jayasumana2013kernel, faraki2018comprehensive}, not all of them arise from the smooth varying inner product (i.e., Riemannian manifold), which would consequently give a ``true'' geodesic distance. For covariance matrices, the metric $\rho_{\mathbf{{\Sigma}}}(\eta,\xi)\triangleq<\eta,\xi>\triangleq\mathrm{tr}(\eta\mathbf{{\Sigma}}^{-1}\xi\mathbf{{\Sigma}}^{-1})$ was adopted by to measure the Rao distance through covariance matrices of two multivariate Gaussian distributions. It is also possible to obtain a closed-form solution for the geodesic between two positive definite matrices $\mathbf{{{\Sigma}}}_0$ and $\mathbf{{{\Sigma}}}_1$, $\gamma(t)=\mathbf{{{\Sigma}}}_0^{1/2}(\mathbf{{{\Sigma}}}_0^{-1/2}\mathbf{{{\Sigma}}}_1\mathbf{{{\Sigma}}}_0^{-1/2})^t\mathbf{{{\Sigma}}}_0^{1/2}$, to yield its geodesic distance $d(\mathbf{{{\Sigma}}}_0,\mathbf{{{\Sigma}}}_1)=||\ln(\mathbf{{{\Sigma}}}_0^{-1/2}\mathbf{{{\Sigma}}}_1\mathbf{{{\Sigma}}}_0^{-1/2})||_F$ \cite{bhatia2009positive}. A geodesic convex function $f$ is defined as $f(\gamma(t))\leq(1-t)f(\mathbf{{{\Sigma}}}_0)+tf(\mathbf{{{\Sigma}}}_1)$ with $t\in[0,1]$. 

When restricted to the manifold of positive definite matrices, it is natural to define such a manifold via the statistics of Gaussian distributions because the covariance of the Gaussian distribution intrinsically satisfies the positive definiteness property. Pioneering in this direction is the work of Rao, which introduced the Rao distance to define the statistical difference between two multivariate Gaussian distributions \cite{rao1945information}. This distance was later generalised and calculated in closed-form \cite{skovgaard1984riemannian, atkinson1981rao, amari1982differential}, to obtain an explicit metric (also called the Fisher-Rao metric). However, with regard to other elliptical distributions, the corresponding Fisher-Rao metric is not guaranteed to be well suited for optimisation \cite{amari1998natural}. 

On the other hand, there is another type of distributions, named the exponential family, that overlaps with elliptical distributions; its Fisher-Rao metric can be explicitly determined by a second-order derivative of the potential function \cite{nielsen2009statistical}. However, the corresponding Riemannian gradient, Levi-Cevita connection, exponential mapping, parallel transport, etc., may not necessarily be obtained explicitly and in a general form for multivariate exponential families \cite{atkinson1981rao}. The existing literature mainly analyses the Gaussian distribution \cite{malago2015information} and the dually-flat affine geometry \cite{khan2018fast} in terms of $\alpha$-connections. More importantly, even though the optimisation can be formulated, a further obstacle is the lack of re-parametrisation property, addressed in this paper. As shown in the sequel, the absence of re-parametrisation could lead to extremely slow convergence.

\subsection{Preliminaries on the elliptical distributions}\label{preED}
A random variable $\mathcal{X}\in \mathbb{R}^M$ is said to have an elliptical distribution if and only if it admits the following stochastic representation \cite{cambanis1981theory},
\begin{equation}\label{RESsto}
\mathcal{X} =^d \bm \mu + \mathcal{R}\mathbf{\Lambda}\mathcal{U},
\end{equation}
where $\mathcal{R} \in \mathbb{R}^+$ is a non-negative real scalar random variable which models the tail properties of the elliptical distribution, $\mathcal{U} \in \mathbb{R}^{M'}$ is a random vector that is uniformly distributed on a unit spherical surface with the probability density function (pdf) within the class of ${\Gamma(\nicefrac{M}{2})}/{(2\pi^{\nicefrac{M}{2}})}$, $\bm \mu \in \mathbb{R}^M$ is a mean (location) vector, while $\mathbf{\Lambda} \in \mathbb{R}^{M \times M'}$ is a matrix that transforms $\mathcal{U}$ from a sphere to an ellipse, and the symbol ``$=^d$" designates ``the same distribution''. 
For a comprehensive review, we refer to \cite{fang2018symmetric, frahm2004generalized}.

Note that an elliptical distribution does not necessarily possess an explicit pdf, but can always be formulated by its characteristic function. However, when $M' = M$ and $\bm{\Lambda}$ has a full row-rank \footnote{We assume these two conditions throughout this paper to ensure an explicit pdf in formulating EMMs.}, that is, for a non-singular scatter matrix $\mathbf{\Sigma} = \mathbf{\Lambda}\mathbf{\Lambda}^T$, the pdf for elliptical distributions does exist and has the following form
\begin{equation}\label{RES}
p_{\mathcal{X}}(\mathbf{x}) =\mathrm{det}(\mathbf{\Sigma})^{-1/2} \cdot c_M \cdot  g\big((\mathbf{x} - \bm{\mu})^T \mathbf{\Sigma}^{-1} (\mathbf{x} - \bm{\mu})\big),
\end{equation}
where the term $c_{M} = \frac{\Gamma(\nicefrac{M}{2})}{2\pi^{\nicefrac{M}{2}}}$ serves as a normalisation term and solely relates to $M$. We also denote the Mahalanobis distance $(\mathbf{x} - \bm{\mu})^T \mathbf{\Sigma}^{-1} (\mathbf{x} - \bm{\mu})$ by the symbol $t$ for simplicity. Then, the density generator, $g(\cdot)$, can be explicitly expressed as $t^{-\nicefrac{(M-1)}{2}}p_\mathcal{R}(\sqrt{t})$, where $t>0$ and $p_\mathcal{R}(t)$ denotes the pdf of $\mathcal{R}$. For example, when $\mathcal{R}=^d \sqrt{\chi_M^2}$, where $\chi_M^2$ denotes the chi-squared distribution of dimension $M$, $g(t)$ in \eqref{RES} is then proportional to $\mathrm{exp}(-t/2)$, which formulates the multivariate Gaussian distribution. 
For simplicity, the elliptical distribution in \eqref{RES} will be denoted by $\mathcal{E}(\mathbf{x}|\bm{\mu}, \mathbf{\Sigma}, g)$. We also need to point out that typical EMMs are identifiable \cite{holzmann2006identifiability}, which is important in order to uniquely estimate mixture models. 

\begin{remark}
	Before proceeding further, we shall emphasise the importance of the stochastic representation of \eqref{RESsto} in analysing elliptical distributions:\\
	1) Since $\mathcal{R}$ is independent of ${\mathcal{U}}$, the Mahalanobis distance $(\mathbf{x} - \bm{\mu})^T \mathbf{\Sigma}^{-1} (\mathbf{x} - \bm{\mu})$ ($=^d \mathcal{R}^2$) is thus independent of the normalised random variable $\nicefrac{\mathbf{\Sigma}^{-1/2}(\mathbf{x}-\bm{\mu})}{\sqrt{(\mathbf{x} - \bm{\mu})^T \mathbf{\Sigma}^{-1} (\mathbf{x} - \bm{\mu})}}$ ($=^d {\mathcal{U}}$), which is an important property in many proofs in this paper.\\
	2) The stochastic representation provides an extremely simple way to generate samples, because only two random variables, i.e., the one-dimensional $\mathcal{R}$ and uniform ${\mathcal{U}}$, can be easily generated.\\
	3) When $\mathcal{R}$ is composed from a scale mixture of normal distributions, the IRA method converges \cite{arslan2004convergence}. For a general EMM, however, the convergence is not ensured.\\
	4) Elliptical distributions can be easily generalised via the stochastic representation. For example, when replacing the uniform distribution, $\mathcal{U}$, with a general directional distribution, i.e., the von Mises-Fisher distribution, parametrised by a mean direction, $\bm{\mu}_v$, and a concentration parameter, $\tau$, we obtain a generalised elliptical distribution. When $\tau \rightarrow 0$, the von Mises-Fisher distribution degenerates into the uniform distribution, $\mathcal{U}$, on the sphere, and the generalised distribution becomes the symmetric elliptical distribution discussed in this paper.
\end{remark}

\subsection{Related works on EMMs}\label{relatedworks}
The GMM based estimation is well established in the machine learning community. Since our focus is on the EMM, we omit the review of GMM and the readers are referred to  \cite{lindsay1995mixture} for a comprehensive review. To robustify the mixture model, the mixtures of the $t$-distributions have been thoroughly studied \cite{peel2000robust, andrews2012model, lin2014capturing}, on the basis of the IRA method. A more general mixture model has been proposed in \cite{sun2010robust} based on the Pearson type VII distribution (includes the $t$-distribution as a special case). Moreover, as the transformed coefficients in the wavelet domain tend to be Laplace distributed, a mixture of the Laplace distributions has been proposed in \cite{tan2007multivariate} for image denoising. Its more general version, a mixture of hyperbolic distributions, has also been recently introduced in \cite{browne2015mixture}. The above distributions belong to the \textit{scale mixture of normals} class, which can be regarded as a multiplication with a Gamma distribution, and ensures the convergence of the IRA. Another recent work proposed a Fisher-Gaussian distribution as mixing components to better accommodate the curvature of data, with the Markov chain Monte Carlo used to solve a Bayesian model \cite{mukhopadhyay2019estimating}. This distribution has a closed-form representation and mainly consists of a von Mises-Fisher distribution convolved with  Gaussian noise, which belongs to generalised skew normal distributions \cite{genton2005generalized}. More importantly, when the concentration parameter $\tau \rightarrow 0$, this distribution then belongs to the mixture of symmetric elliptical distributions.

On the other hand,  Wiesel proved the convergence of the IRA in \cite{wiesel2012geodesic} via the concept of geodesic convexity of Riemannian manifold, and Zhang \textit{et al.} further relaxed the convergence conditions in \cite{zhang2013multivariate}. The work in \cite{sra2013geometric} proves similar results from another perspective of the Riemannian manifold, which also states that the IRA cannot ensure a universal convergence for all EMMs. In other words, for other elliptical distributions, the convergence is no longer guaranteed. Despite several attempts, current EMMs, including \cite{karlis2007finite, lopez2011stochastic, browne2012model}, are rather of an \textit{ad hoc} nature.
%Recently, a toolbox in \cite{hosseini2015mixest} which was originally designed for the GMM, has already included several types of elliptical distributions. However, the existing toolbox has not been generalised to the EMM. 

Besides the IRA method for solving several EMMs, gradient-based numerical algorithms typically rest upon additional techniques that only work in particular situations (e.g., gradient reduction \cite{redner1984mixture}, s re-parametrisation \cite{jordan1994hierarchical} and Cholesky decomposition \cite{naim2012convergence, salakhutdinov2003optimization}). Recently, Hosseini and Sra directly adopted a Riemannian manifold method for estimating the GMM, which provided an alternative to the traditional EM algorithm in the GMM problem \cite{hosseini2015matrix, hosseini2017alternative}. However, their method fails to retain the optimum in the EMM problem. To this end, we propose a universal scheme to consistently and stably achieve the optimum at a fast speed, which acts as a ``necessity'' instead of an ``alternative'' in the EMM problem, as the IRA algorithm may not converge.
%In general, according to the stochastic representation of \eqref{RESsto}, there are two random variables that relate to an arbitrary elliptical distribution, i.e., $\mathcal{R}$ and $\mathbf{u}$\footnote{$\bm\mu$ and $\mathbf{{{\Sigma}}}$ are constant parameters about the mean and the covariance.}. As the uniformly distributed $\mathbf{u}$ only relates to the dimension $M$, different types of elliptical distributions are mainly related to the random variable $\mathcal{R}$, or equivalently, $\mathcal{R}^2$. We then provide a unifying summary of typical elliptical distributions in terms of the pdf of $\mathcal{R}^2$ in Table \ref{typicalRES}\footnote{The term $\mathcal{R}^2$ is frequently used in practice because $\mathcal{R}^2 =^d (\mathbf{x} - \bm{\mu})^T \mathbf{\Sigma}^{-1} (\mathbf{x} - \bm{\mu})$ (i.e., the Mahalanobis distance).}. 

\section{Manifold optimisation for the EMM}\label{MEMM}
In this section, we shall first justify the Riemannian metrics of elliptical distributions in Section \ref{secmetric}, followed by a layout of the EMM problem, and the introduction of the proposed method in Section \ref{FEMM}. Finally, a novel type of regularisation on the EMMs is introduced in Section \ref{REMM}, which includes the \textit{mean-shift} algorithm as a special case.
\subsection{Statistical metrics for elliptical distributions}\label{secmetric}
Although there are various metrics designed for measuring the distance between matrices \cite{jeuris2012survey, sra2012new, jayasumana2013kernel, faraki2018comprehensive}, not all of them arise from the smooth varying inner product (i.e., Riemannian metrics), which would consequently give a ``true'' geodesic distance. One of the widely employed Riemannian metrics is the intrinsic metric $\mathrm{tr}(d\mathbf{\Sigma}\mathbf{{\Sigma}}^{-1}d\mathbf{\Sigma}\mathbf{{\Sigma}}^{-1})$, which can be obtained via the ``entropy differential metric'' \cite{burbea1982entropy} of two multivariate Gaussian distributions. The entropy related metric was later used by Hiai and Petz to define the Riemannian metric for positive definite matrices \cite{hiai2009riemannian}. In this paper, we follow the work of \cite{hiai2009riemannian} to calculate the corresponding Riemannian metrics for the elliptical distributions.
\begin{lemma}\label{proprie}
	Consider the class of elliptical distributions $\mathcal{E}(\mathbf{x}|\bm{\mu}, \mathbf{\Sigma}, g)$. Then, the Riemannian metric for the covariance is given by 
	\begin{equation}
	ds^2 = \frac{1}{2}\mathrm{tr}(d\mathbf{{\Sigma}}\mathbf{{\Sigma}}^{-1}d\mathbf{{\Sigma}}\mathbf{{\Sigma}}^{-1}).
	\vspace{-0.5em}
	\end{equation}
\end{lemma}

\begin{proof}
	Please see Appendix-\ref{prooflemma1}.
\end{proof}
More importantly, as the metric is related to $\mathbf{\Sigma}$, the Levi-Civita connection is given by $\nabla_{\mathbf{X}}\mathbf{Y}=-\frac{1}{2}(\mathbf{X}\mathbf{\Sigma}^{-1}\mathbf{Y}+ \mathbf{Y}\mathbf{\Sigma}^{-1}\mathbf{X})$ \cite{moakher2011riemannian}, where $\mathbf{X}, \mathbf{Y}$ are vector fields on the manifold of $\mathbf{\Sigma}$. The corresponding exponential mapping, which moves along with the geodesics given the direction from a tangent vector, can be explicitly obtained as $\mathrm{Exp}_{\mathbf{\Sigma}}(\mathbf{U}) = \mathbf{\Sigma}^{\frac{1}{2}}\exp(\mathbf{\Sigma}^{-\frac{1}{2}}\mathbf{U}\mathbf{\Sigma}^{-\frac{1}{2}})\mathbf{\Sigma}^{\frac{1}{2}},$ where $\mathbf{U}\in T_{\mathbf{\Sigma}}\mathcal{M}$ \cite{moakher2011riemannian}.

When estimating parameters of elliptical distributions, the mean vector and the covariance matrix need to be estimated simultaneously. An elegant strategy would be to incorporate the mean and the covariance into an augmented matrix with one extra  dimension \cite{kent1991redescending}. Such a strategy has also been successfully employed in the work of \cite{hosseini2015matrix, hosseini2017alternative}, which is called the ``reformulation trick''. Thus, based on the metrics of Lemma \ref{proprie}, we can introduce the following relationship related to the reformulation. 
\begin{lemma}\label{ellipticalmetric}
	Consider the class of elliptical distributions, $\mathcal{E}(\mathbf{y}|\bm{0}, \mathbf{\tilde\Sigma}, g)$. Then, upon reformulating $\mathbf{y}$ and $\mathbf{\tilde\Sigma}$ as
	\begin{equation}\label{refor}
	\mathbf{y} = [\mathbf{x}^T, ~1]^T, ~~~~ \mathbf{\tilde{\Sigma}} = \left({{\mathbf{\Sigma}+\lambda\bm{\mu}\bm{\mu}^T} \atop {\lambda\bm{\mu}^T}}~~ {\lambda{\bm{\mu}} \atop {\lambda}}\right)
	\end{equation}
	the subsequent Riemannian metric follows
	\begin{equation}
	\begin{aligned}
	&ds^2 = \mathrm{tr}(d\mathbf{{\tilde\Sigma}}\mathbf{{\tilde\Sigma}}^{-1}d\mathbf{{\tilde\Sigma}}\mathbf{{\tilde\Sigma}}^{-1})\!\\
	&=\!\lambda d\bm\mu^T\mathbf{\Sigma}^{-1}d\bm\mu\! +\! \frac{1}{2}\mathrm{tr}(d\mathbf{{\Sigma}}\mathbf{{\Sigma}}^{-1}d\mathbf{{\Sigma}}\mathbf{{\Sigma}}^{-1}) \!+\! \frac{1}{2}(\lambda^{-1}d\lambda)^2.
	\end{aligned}
	\end{equation}
	%where $\lambda = \frac{4}{M}\int_{0}^{\infty}t(\frac{g'(t)}{g(t)})^2p_{\mathcal{R}^2}(t)dt$.
\end{lemma}

\begin{proof}
	The proof is a direct extension of that in \cite{calvo1990distance}, where only the Gaussian case is proved, and will thus be omitted.
\end{proof}
As $d\bm\mu^T\mathbf{\Sigma}^{-1}d\bm\mu$ and $\frac{1}{2}\mathrm{tr}(d\mathbf{{\Sigma}}\mathbf{{\Sigma}}^{-1}d\mathbf{{\Sigma}}\mathbf{{\Sigma}}^{-1})$ exactly formulate the two manifolds of EMMs without reformulation, we can inspect the relationship between the manifolds of EMMs with and without the reformulation from Lemma \ref{ellipticalmetric}, which provides another perspective in understanding the reformulation. % and enhances physical interpretability.
\begin{remark}
	In Lemma \ref{ellipticalmetric}, there is a mismatch between the two manifolds, due to the term $\frac{1}{2}(\lambda^{-1}d\lambda)^2$. 
	When restricted to the Gaussian case, we show in the sequel that the gradient of $\lambda$ vanishes when optimising $\mathbf{{\tilde\Sigma}}$, i.e., $d\lambda = 0$. In this case, manifold optimisation on $\mathbf{{\tilde\Sigma}}$ is  performed under the same metric as a simultaneous optimisation on a product manifold of the mean and the covariance, which leads to the success of \cite{hosseini2015matrix, hosseini2017alternative} in solving GMMs. However, this property does not hold for general EMMs.
\end{remark}

\subsection{Manifold optimisation on the EMM}\label{FEMM}
%Since there is no closed-form solution to EMM problems, we therefore proceed to introduce the manifold optimisation. 
Generally, we assume that the EMM consists of $K$ mixing components, each elliptically distributed. To make the proposed EMM flexible enough to capture inherent structures in data, in our framework it is not necessary for every elliptical distribution within the mixture to have the same density generator (denoted by $\mathcal{E}_k(\mathbf{x}|\bm{\mu}_k, \mathbf{\Sigma}_k, g_k)$). 
In finite mixture models, the probability of choosing the $k$-th mixing component is denoted by $\pi_k$, so that $\sum_{k = 1}^K \pi_k = 1$. For a set of i.i.d samples $\mathbf{x}_n,$ $n = 1,2,3,\cdots, N$, the negative log-likelihood can be obtained as
%In finite mixture models, latent variables $\mathcal{Z}_k \in \{0,1\}$ are binary, and the probability of choosing the $k$-th mixture is denoted by $p(\mathcal{Z}_k = 1) = \pi_k$, so that $\sum_{k = 1}^K \mathcal{Z}_k = 1$ and $\sum_{k = 1}^K \pi_k = 1$. For a set of i.i.d samples $\mathbf{x}_n,$ $n = 1,2,3,\cdots, N$, the negative log-likelihood can be obtained as
\begin{equation}\label{MLLx}
J \!=\! -\!\!\sum_{n=1}^N\!\mathrm{ln}\!\sum_{k = 1}^K\!\pi_k c_M \mathrm{det}(\mathbf{\Sigma}_k)^{-\frac{1}{2}} g_k\!\big((\mathbf{x}_n - \bm{\mu}_k)^T \mathbf{\Sigma}_k^{-1} (\mathbf{x}_n - \bm{\mu}_k)\big).
\end{equation}

The estimation of $\pi_k$, $\bm{\mu}_k$ and $\mathbf{\Sigma}_k$ therefore requires the minimisation of $J$ in \eqref{MLLx}. By setting the derivatives of $J$ to 0, we arrive at the following equations,
\begin{equation}\label{solution}
\begin{aligned}
&\pi_k = \frac{\sum_{n=1}^N \xi_{nk}}{N}, ~\bm{\mu}_k = \frac{\sum_{n=1}^N \xi_{nk}\psi_k(t_{nk})\mathbf{x}_n}{\sum_{n=1}^N \xi_{nk}\psi_k(t_{nk})}, \\
&\mathbf{\Sigma}_k  \!= \!-2\frac{\sum_{n=1}^N \xi_{nk}\psi_k(t_{nk}) (\mathbf{x}_n-\bm{\mu}_k)(\mathbf{x}_n-\bm{\mu}_k)^T}{\sum_{n=1}^N\xi_{nk}},
\end{aligned}
\end{equation}
where $	\xi_{nk} = \frac{\mathcal{E}_k(\mathbf{x}_n|\bm{\mu}_k, \mathbf{\Sigma}_k, g_k)\pi_k}{\sum_{k = 1}^K\mathcal{E}_k(\mathbf{x}_n|\bm{\mu}_k, \mathbf{\Sigma}_k, g_k)\pi_k}$ is the posterior distribution of latent variables; $t_{nk} = (\mathbf{x}_n - \bm{\mu}_k)^T \mathbf{\Sigma}_k^{-1} (\mathbf{x}_n - \bm{\mu}_k)$ is the Mahalanobis distance; $\psi_k(t_{nk}) = \nicefrac{g'_k(t_{nk})}{g_k(t_{nk})}$  acts almost as an M-estimator for most heavily-tailed elliptical distributions, which decreases to $0$ when the Mahalanobis distance $t_{nk}$ increases to infinity. It is obvious that the solutions $\pi_k$, $\bm{\mu}_k$ and $\mathbf{\Sigma}_k$ are intertwined with $\xi_{nk}$ and $t_{nk}$, which prevents a closed-form solution\footnote{It should be pointed out that there are multiple solutions to \eqref{solution} and the goal here is to find a local stationary point. Finding the global optima is difficult in mixture problems \cite{jin2016local} and beyond the scope of this paper.} of \eqref{MLLx}. By iterating \eqref{solution}, this results to an EM-type solver, which is exactly the IRA algorithm. However, the convergence of the IRA is not guaranteed for general EMMs \cite{kent1991redescending}.

%It is well-known that the Mahalanobis distance is a scale-free metric, which is particularly suited to measure outliers. Thus, $\mathbf{x}_n$ with large Mahalanobis distance result in small values of $\psi_k(t_{nk})$, and would have little impact on the final $\bm{\mu}_k$, $\mathbf{{\Sigma}}_k$ and $\pi_k$, which therefore generates the robustness of the EMM. We will show more on the robustness of EMMs in Sec. \ref{secIFEMM}. Furthermore, the existence of $\psi_k$ also mitigate the problem of the singular solutions during estimation.

On the other hand, when directly estimating the reformulated EMM, i.e., $\mathcal{E}_k(\mathbf{y}|\bm{0}, \mathbf{\tilde\Sigma}_k, g_k)$, similarly, we arrive at 
\begin{equation}
\widetilde{\pi}_k = \frac{\sum_{n=1}^N \widetilde{\xi}_{nk}}{N}, \mathbf{\tilde{\Sigma}}_k = -2\frac{\sum_{n=1}^N \widetilde{\xi}_{nk}\psi_k(\widetilde{t}_{nk}) \mathbf{y}_n\mathbf{y}_n^T}{\sum_{n=1}^N\widetilde{\xi}_{nk}}, 
\end{equation}
where $\widetilde{t}_{nk}=\mathbf{y}_n^T \mathbf{\widetilde{\Sigma}}_k^{-1} \mathbf{y}_n$ and $	\widetilde{\xi}_{nk} = \frac{\mathcal{E}_k(\mathbf{y}_n|\bm{0}, \mathbf{\widetilde{\Sigma}}_k, g_k)\pi_k}{\sum_{k = 1}^K\mathcal{E}_k(\mathbf{y}_n|\bm{0}, \mathbf{\widetilde{\Sigma}}_k, g_k)\pi_k}$. It needs to be pointed out that the directly reformulated EMM optimises on the augmented space $\mathbf{y}_n=[\mathbf{x}_n^T, 1]^T$ of $\mathbb{R}^{M+1}$, which is typically a mismatch to the original problem within the dimension $M$. This intrinsic difference becomes clear after decomposing $\mathbf{\tilde{\Sigma}}_k$ to obtain the corresponding solutions, $\bm{\mu}_k$ and $\mathbf{\Sigma}_k$, as well as $\lambda_k$. This is achieved in the form
\begin{equation}\label{directEMM}
\begin{aligned}
&\left({{\mathbf{\Sigma}_k+\lambda_k\bm{\mu}_k\bm{\mu}_k^T} \atop {\lambda_k\bm{\mu}_k^T}}~~ {\lambda_k{\bm{\mu}_k} \atop {\lambda}_k}\right) \\
&~~~~= -\frac{2}{\sum_{n=1}^N\widetilde{\xi}_{nk}} \sum_{n=1}^N \widetilde{\xi}_{nk}\psi_k(\widetilde{t}_{nk}) \left({{\mathbf{x}_n\mathbf{x}_n^T} \atop {\mathbf{x}_n}}~~ {\mathbf{x}_n^T \atop 1}\right).
\end{aligned}
\end{equation}
Because $(\widetilde{t}_{nk}\! =\! \mathbf{y}_n^T\mathbf{\tilde{\Sigma}}_k\mathbf{y}_n) \!\neq\! (t_{nk}\!=\!(\mathbf{x}_n \!-\! \bm{\mu}_k)^T \mathbf{\Sigma}_k^{-1} (\mathbf{x}_n\! - \!\bm{\mu}_k))$, $\psi_k(\widetilde{t}_{nk})$ in \eqref{directEMM} does not equal $\psi_k(\widetilde{t}_{nk})$ in \eqref{solution}. The only exception is when $\psi_k(\cdot)$ is a constant, e.g.,  $\psi_k(\cdot)\!\equiv\!-\frac{1}{2}$ for the Gaussian distribution. In this case, $\lambda_k\! \equiv\! 1$, and the manifold with the reformulation is same as the original one. 

To retain the same optimum as the original problem, we introduce a new parameter $c_k$, which aims to mitigate the mismatch of the reformulated manifold brought by $\lambda_k$. The same optimum is ensured in the following theorem.
\begin{theorem}\label{reformu}
	The optimisation of $\pi_k$, $\mathbf{\tilde{\Sigma}}_k$ and $c_k$ based on the following re-designed cost 
	\begin{equation}\label{reformuJ}
	\tilde J \!=\! -\!\sum_{n=1}^N\!\mathrm{ln}\!\sum_{k = 1}^K\!\pi_k\cdot c_M \cdot (c_k\mathrm{det}(\mathbf{\tilde{\Sigma}}_k))^{-\nicefrac{1}{2}}g_k\left(\mathbf{y}_n^T \mathbf{\tilde{\Sigma}}_k^{-1} \mathbf{y}_n \!- \!c_k\right)
	\vspace{-0.5em}
	\end{equation}
	has the same optimum as those in \eqref{solution}:
	\begin{equation}\label{lambdaoptimum}
	\begin{aligned}
	&\pi_k \!=\! \frac{\sum_{n=1}^N \!\xi_{nk}}{N},	c_k = \frac{1}{\lambda_k} = -\frac{\sum_{n=1}^N\xi_{nk}}{2\sum_{n=1}^N\xi_{nk}\psi_k(t_{nk})}\\
	& \mathbf{\tilde{\Sigma}}_k \!=\! -2\frac{\sum_{n=1}^N \!\xi_{nk}\psi_k(t_{nk}) \mathbf{y}_n\mathbf{y}_n^T}{\sum_{n=1}^N\xi_{nk}}
	\end{aligned}
	\vspace{-0.5em}
	\end{equation} 
\end{theorem}

\begin{proof}
	Please see Appendix-\ref{prooftheorem1}.
\end{proof}

We optimise \eqref{reformuJ} on a product manifold of  $\mathbf{\tilde{\Sigma}}_k$, $\pi_k$ and $c_k$. For optimising $\mathbf{\tilde{\Sigma}}_k$, on the basis of the metric in Section \ref{secmetric}, we calculate Riemannian gradient as $\nabla_R\tilde{J} = \mathbf{\tilde{\Sigma}}_k (\nabla_E\tilde{J}) \mathbf{\tilde{\Sigma}}_k$, where $\nabla_E\tilde{J}$ is the Euclidean gradient of cost $\tilde J$ via $\nicefrac{\partial \tilde{J}}{\partial \mathbf{\tilde{\Sigma}}_k}$. Furthermore, although explicitly formulated, it is important to mention that the exponential mapping provided after Lemma \ref{proprie} operates on a matrix, which comes with an extremely high computational complexity (typically $\mathcal{O}(M^4)$) and even needs a certain degree of approximation \cite{miyajima2019verified}. A common way to approximate the exponential mapping is via the retraction, of which the accuracy is up to the first order to the exponential mapping \cite{absil2012projection}. Thus, we employ the Taylor series expansion of $\exp(\mathbf{\Sigma}^{-\frac{1}{2}}\mathbf{U}\mathbf{\Sigma}^{-\frac{1}{2}})$ as a way of the approximation, via 
\begin{equation}
\begin{aligned}
&\mathrm{Exp}_{\mathbf{\Sigma}}(\mathbf{U}) = \mathbf{\Sigma}^{\frac{1}{2}}\exp(\mathbf{\Sigma}^{-\frac{1}{2}}\mathbf{U}\mathbf{\Sigma}^{-\frac{1}{2}})\mathbf{\Sigma}^{\frac{1}{2}} \\
&~~~~\approx \mathbf{\Sigma}^{\frac{1}{2}}(\mathbf{0} + \mathbf{\Sigma}^{-\frac{1}{2}}\mathbf{U}\mathbf{\Sigma}^{-\frac{1}{2}} + \frac{1}{2} \mathbf{\Sigma}^{-\frac{1}{2}}\mathbf{U}\mathbf{\Sigma}^{-1}\mathbf{U}\mathbf{\Sigma}^{-\frac{1}{2}})\mathbf{\Sigma}^{\frac{1}{2}} \\
&~~~~= \mathbf{\Sigma} + \mathbf{U} + \frac{1}{2} \mathbf{U}\mathbf{\Sigma}^{-1}\mathbf{U} = R_{\mathbf{\Sigma}}(\mathbf{U}), 
\end{aligned}
\end{equation}
where the approximation is performed up to the cubic (third-order) term. It can be easily verified that $R_{\mathbf{\Sigma}}(\mathbf{U})$ is a retraction (Chapter 4.1 of \cite{absil2009optimization}), which significantly reduces the computational complexity from a matrix exponential to simple linear operations on matrices. Finally, we employ the conjugate gradient descent \cite{boumal2015lowrank} as a manifold solver, with a pseudo-code for our method given in Algorithm \ref{Alg1}. %Thus, the optimum of \eqref{reformuJ} can be finally achieved by gradient descent on the defined Riemannian manifold of Sec. \ref{secmetric} in a fast speed and a stable way.
\begin{algorithm}
	{
	\caption{The proposed method for optimising the EMM}\label{Alg1}
	\begin{algorithmic}[1]
		\Require $N$ observed samples: $\mathbf{x}_1, \mathbf{x}_2, \ldots, \mathbf{x}_N$;
		\State \textbf{initialize:} $\{\bm{\mu}_k^0\}_{k=1}^K$, $\{\mathbf{\Sigma}_k^0\}_{k=1}^K$, $\{\pi_k^0\}_{k=1}^K$,  $\{c_k^0\}_{k=1}^K$, and $\{\lambda_{k}^{ini}\}_{k=1}^K$
		\For{$k = 1 ~to~ K$}
		\For{$n = 1 ~to~ N$}
		\State $\mathbf{y}_n, \mathbf{\tilde \Sigma}_k^0$ $\gets$ REPARAMET($\mathbf{x}_n$, $\bm{\mu}_k^0$, $\mathbf{\Sigma}_k^0$, $\lambda_{k}^{ini}$);
		\EndFor
		\EndFor
		\While{not converged (at the $t$-th iteration):}
		\For{$k = 1 ~to~ K$}
		\State Calculate Euclidean gradients $\nabla_E(\mathbf{\tilde \Sigma}_k^t)$, $\nabla_E(\pi_k^{t})$ and $\nabla_E(c_k^{t})$, by differentiating $\tilde J$ of \eqref{reformu};
		\State $\pi_k^{t+1} \gets $ Step descent based on $\nabla_E(\pi_k^{t})$;
		\State $c_k^{t+1} \gets $ Step descent based on $\nabla_E(c_k^{t})$;
		\State Update $\mathbf{\tilde \Sigma}_k^t$:
		\State ~~~~~~$\nabla_R(\mathbf{\tilde \Sigma}_k^t) \gets$ RGRADIENT($\mathbf{\tilde \Sigma}_k^t$, $\nabla_E(\mathbf{\tilde \Sigma}_k^t)$);
		\State ~~~~~~$\mathbf{U}_k^{t} \gets$ Step descent based on $\nabla_R(\mathbf{\tilde \Sigma}_k^t)$;
		\State ~~~~~~$\mathbf{\tilde \Sigma}_k^{t+1} \gets$ RETRACTION($\mathbf{\tilde \Sigma}_k^t$, $\mathbf{U}_k^{t}$);
		\EndFor
		\EndWhile
		\For{$k = 1 ~to~ K$}
		\State $\bm{\mu}_k^*$, $\mathbf{\Sigma}_k^*$ $\gets$ DECOMPOSITION($\mathbf{\tilde \Sigma}_k^*$) 
		\EndFor
		\Ensure $\{\bm{\mu}_k^*\}_{k=1}^K$, $\{\mathbf{\Sigma}_k^*\}_{k=1}^K$ and $\{\pi_k^*\}_{k=1}^K$.
		\vspace{.5em}
		\Procedure{Reparamet}{$\mathbf{x}$, $\bm{\mu}$, $\mathbf{\Sigma}$, $\lambda$}
		\State Re-parametrisation via \eqref{refor}.
		\State \textbf{return} $\mathbf{y}, \mathbf{\tilde{\Sigma}}$
		\EndProcedure
		\Procedure{Rgradient}{$\mathbf{\Sigma}$, $\nabla_E$}
		\State \textbf{return} $\nabla_R = (\mathbf{\Sigma}\nabla_E\mathbf{\Sigma})$
		\EndProcedure	
		\Procedure{Retraction}{$\mathbf{\Sigma}$, $\mathbf{U}$}
		\State \textbf{return} $\mathbf{\Sigma}_{new} = (\mathbf{\Sigma} + \mathbf{U} + \frac{1}{2} \mathbf{U}\mathbf{\Sigma}^{-1}\mathbf{U})$
		\EndProcedure	
		\Procedure{Decomposition}{$\mathbf{\tilde \Sigma}$, $c$}
		\State Decompose via inverting \eqref{refor}: $ \left({{\mathbf{\Sigma}+\frac{1}{c}\bm{\mu}\bm{\mu}^T} \atop {\frac{1}{c}\bm{\mu}^T}}~~ {\frac{1}{c}{\bm{\mu}} \atop {\frac{1}{c}}}\right)=\mathbf{\tilde{\Sigma}}$
		\State \textbf{return} $\bm{\mu}$ and $\mathbf{\Sigma}$
		\EndProcedure			
	\end{algorithmic}}
\end{algorithm}

The advantages of our algorithm, by virtue of its inherent reformulation, can be understood from two aspects. First, through the reformulation, our method is capable of providing a relatively global descent in terms of the re-parametrised $\mathbf{\tilde\Sigma}_k$, whereas optimisation without the reformulation requires a sophisticated incorporated step descent on both $\bm{\mu}_k$ and $\mathbf{\Sigma}_k$, to ensure a well-behaved convergence. On the other hand, one typical singularity case is when certain $\bm{\mu}_k$ moves to the boundary of the data during optimisation, in which the cluster is likely to model a small set of data samples (e.g., one or two samples). In contrast, the proposed reformulated EMMs can be regarded as zero-mean mixtures, which to some extent relieves this singularity issue.

\subsection{Regularisation}\label{REMM}
We impose the inverse-Wishart prior distribution (i.e., $p_{\mathbf{\Sigma}_k}(\mathbf{\Sigma}_k) \propto \frac{1}{\mathrm{det}(\mathbf{\Sigma}_k)^{\nicefrac{v}{2}}}\exp(-\frac{v\mathrm{tr}(\mathbf{\Sigma}_k^{-1}\mathbf{S})}{2})$) to regularise the EMM, where $v$ controls the freedom and $\mathbf{S}$ is the prior matrix. The advantages of using a form of $\mathrm{tr}(\mathbf{{{\Sigma}}}_k^{-1}\mathbf{S})$ are two-fold: i) it is strictly geodesic convex in $\mathbf{{{\Sigma}}}_k$ and ii) the solutions are ensured to exist for any data configuration \cite{ollila2014regularized}. By utilising maximising a posterior on covariance matrices, we obtain the same solutions of $\pi_k$ and $\bm{\mu}_k$ as those of \eqref{solution}, whereas $\mathbf{\Sigma}_k$ now becomes 
\begin{equation}\label{solutionr}
\mathbf{\Sigma}_k = \frac{-2\cdot\sum_{n=1}^N \xi_{nk}\psi_k(t_{nk}) (\mathbf{x}_n-\bm{\mu}_k)(\mathbf{x}_n-\bm{\mu}_k)^T + v\mathbf{S}}{\sum_{n=1}^N\xi_{nk} + v}.
\end{equation}

Similar to Theorem \ref{reformu}, the following proposition can be obtained for the reformulation with regularisations.
\begin{prop}\label{regularpro}
	The optimisation of $\pi_k$, $\mathbf{\tilde{\Sigma}}_k$ and $c_k$ based on the following cost function 
	\begin{equation}
	\tilde J_r \!=\! \tilde J + \sum_{k = 1}^K{(c_k\mathrm{det}(\mathbf{\tilde{\Sigma}}_k))^{-\nicefrac{v}{2}}}\exp\left(-\frac{v\mathrm{tr}(\mathbf{\tilde{\Sigma}}_k^{-1}\mathbf{\tilde{S}})}{2}\right),
	\end{equation}
	achieves the same optimal $\mathbf{\Sigma}_k$ as in \eqref{solutionr} and the same $\pi_k$ and $\bm{\mu}_k$ as in \eqref{solution}, where $\mathbf{\tilde{S}} = [{{\mathbf{S}} \atop {\mathbf{0}}} \hspace{.5em} {{\mathbf{0}} \atop {0}}]$. The optimal $c_k$ and $\lambda_k$ are
	\begin{equation}
	c_k = \frac{1}{\lambda_k} = -\frac{\sum_{n=1}^N\xi_{nk} + v}{2\sum_{n=1}^N\xi_{nk}\psi_k(t_{nk})}.
	\end{equation} 
\end{prop}

\begin{proof}
	The proof is analogous to that of Theorem \ref{reformu} and is therefore omitted.
\end{proof}
\begin{remark}
	In \eqref{solutionr}, it can be seen that when $v\rightarrow \infty$, $\mathbf{\Sigma}_k \rightarrow \mathbf{S}$. Furthermore, when $\mathbf{S}=\mathbf{I}_M$, $\mathbf{\Sigma}_k = \sigma^2\mathbf{I}_M$, the estimation of $\bm{\mu}_k$ in \eqref{solution} becomes $\frac{\sum_{n=1}^N \xi_{nk}\psi_k(\sigma^{-2}||\mathbf{x}_n-\bm{\mu}_k||^2)\mathbf{x}_n}{\sum_{n=1}^N \xi_{nk}\psi_k(\sigma^{-2}||\mathbf{x}_n-\bm{\mu}_k||^2)}$, which is the basic \textit{mean-shift} algorithm with soft thresholds. Furthermore, when $\mathbf{S}=\mathbf{I}_M$, $\mathbf{\Sigma}_k = \mathbf{I}_M$ and $\psi_k(t_{nk}) = -\nicefrac{1}{2}$ (the GMM), it then turns to a soft version of the basic \textit{k-means} algorithm. This all demonstrates that the EMM is a flexible framework in our regularisation settings and that we can choose $v$ and $\mathbf{S}$ to achieve different models. 
\end{remark}

It needs to be pointed out that although the inverse-Wishart prior is one of the popular priors (typically $\mathbf{S} = \mathbf{I}$), there are also other priors which suit different requirements. For example, there is also work using the Wishart prior, which is less informative but requires a particular setting of the parameters \cite{chung2015weakly}. Instead of controlling the degrees of freedom by the scalar $v$, a generalised inverse-Wishart distribution has been applied to flexibly control the degrees of freedom \cite{rajaratnam2008flexible}. Another pragmatic solution would be to decompose the covariance matrix into its standard deviation and correlation matrix components (inverse-Wishart distribution) so that the standard deviation can be treated in a flexible way \cite{barnard2000modeling}. Moreover, probabilistic graphical models can be used as a prior to explicitly control the sparsity of the matrices \cite{fop2019model}, where e.g., graphical LASSO can be applied. Furthermore, a robust distribution for positive definite matrices, named $F$-distribution, has become a popular choice for priors, which generalises the half-Cauchy and half-$t$ distributions \cite{mulder2018matrix}. Recently, a Riemannian Gaussian distribution for the positive definite matrix has been proposed by replacing the Mahalanobis term with the Fisher-Rao metric of positive definite matrices \cite{said2017riemannian}. A similar strategy can be extended to the Laplacian \cite{hajri2016riemannian} and even to the elliptical distributions, which has a significant potential to generate a rich class of priors on positive definite matrices. This paper investigates the inverse-Wishart prior as an example of regularisation, because it can further emphasise the flexibility of EMMs and also the compatibility of our re-parametrisation technique. The investigation on other priors is part of our future work.
\begin{figure*}[!h]
	\begin{center}
		\subfigure[Gaussian]{\includegraphics[width=0.24\textwidth]{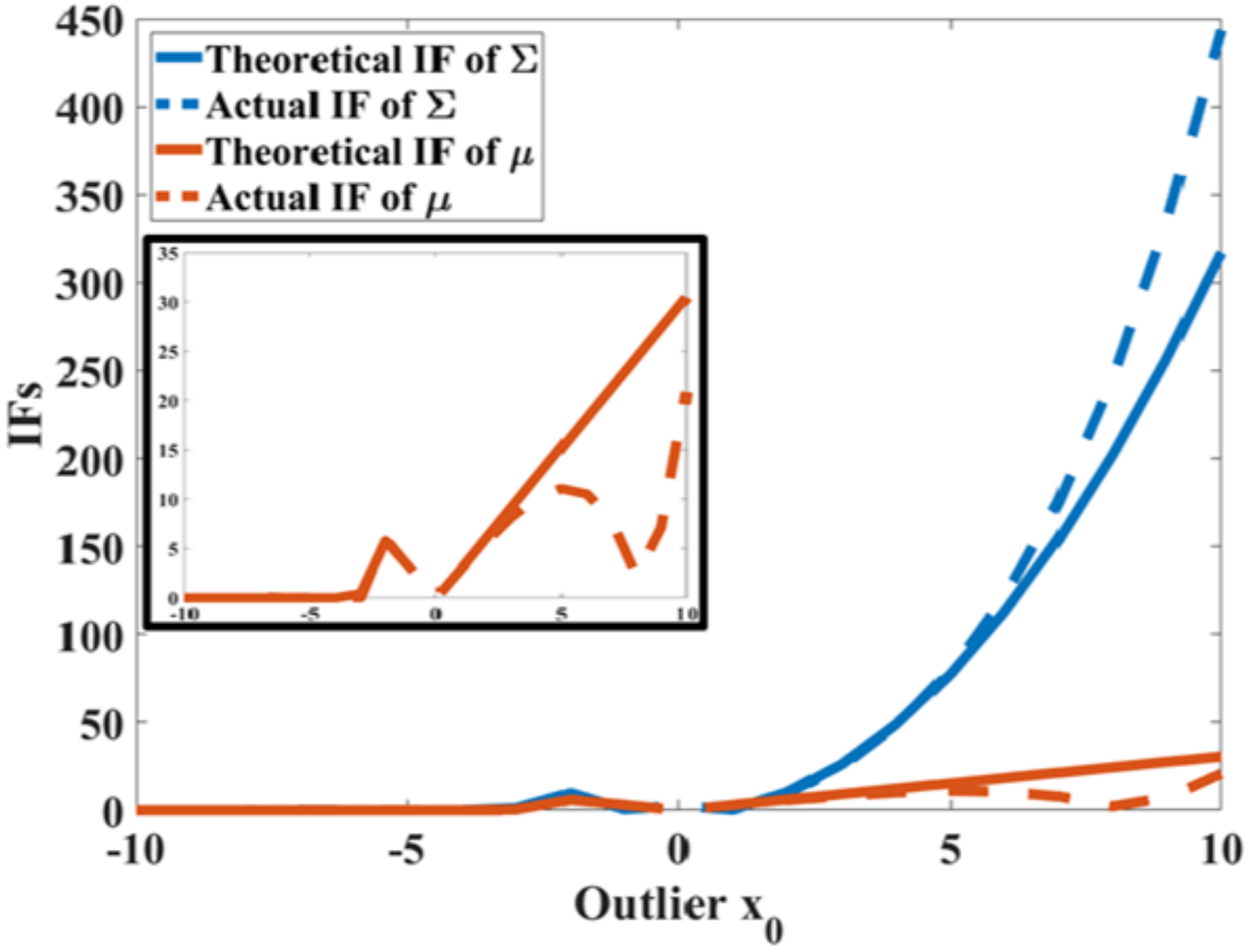}}
		\subfigure[Cauchy]{\includegraphics[width=0.24\textwidth]{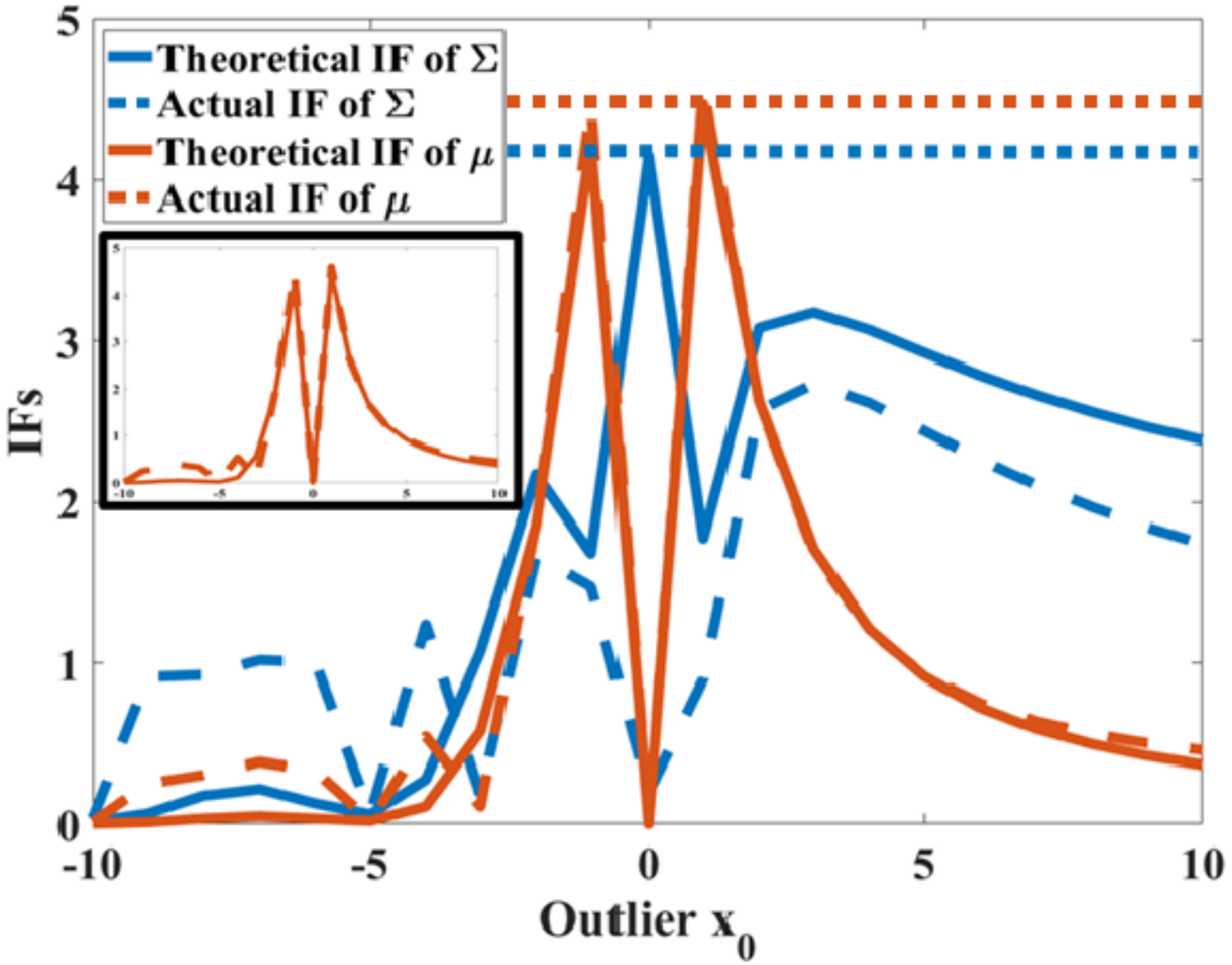}}
		\subfigure[Laplace]{\includegraphics[width=0.24\textwidth]{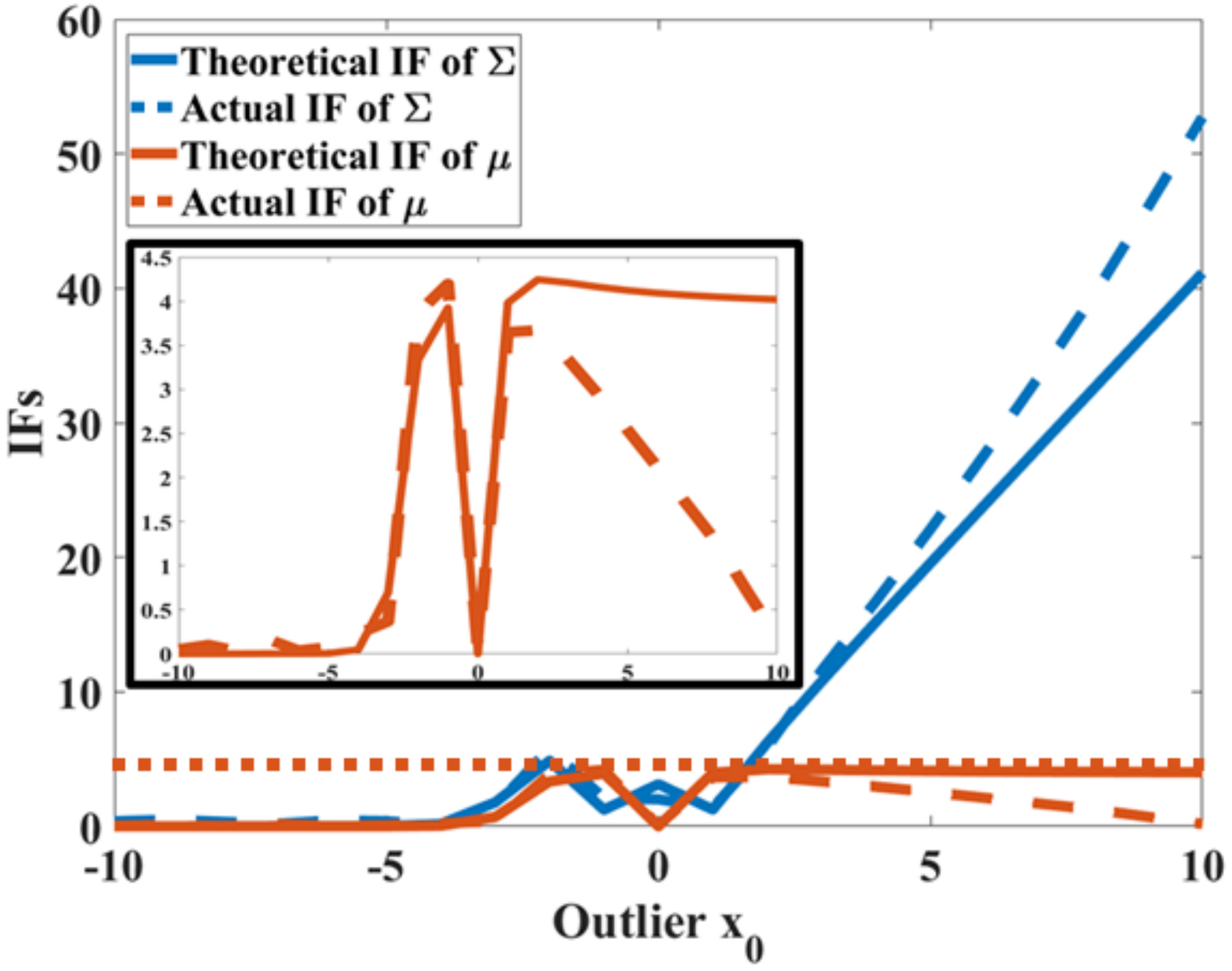}}
		\subfigure[GG1.5]{\includegraphics[width=0.24\textwidth]{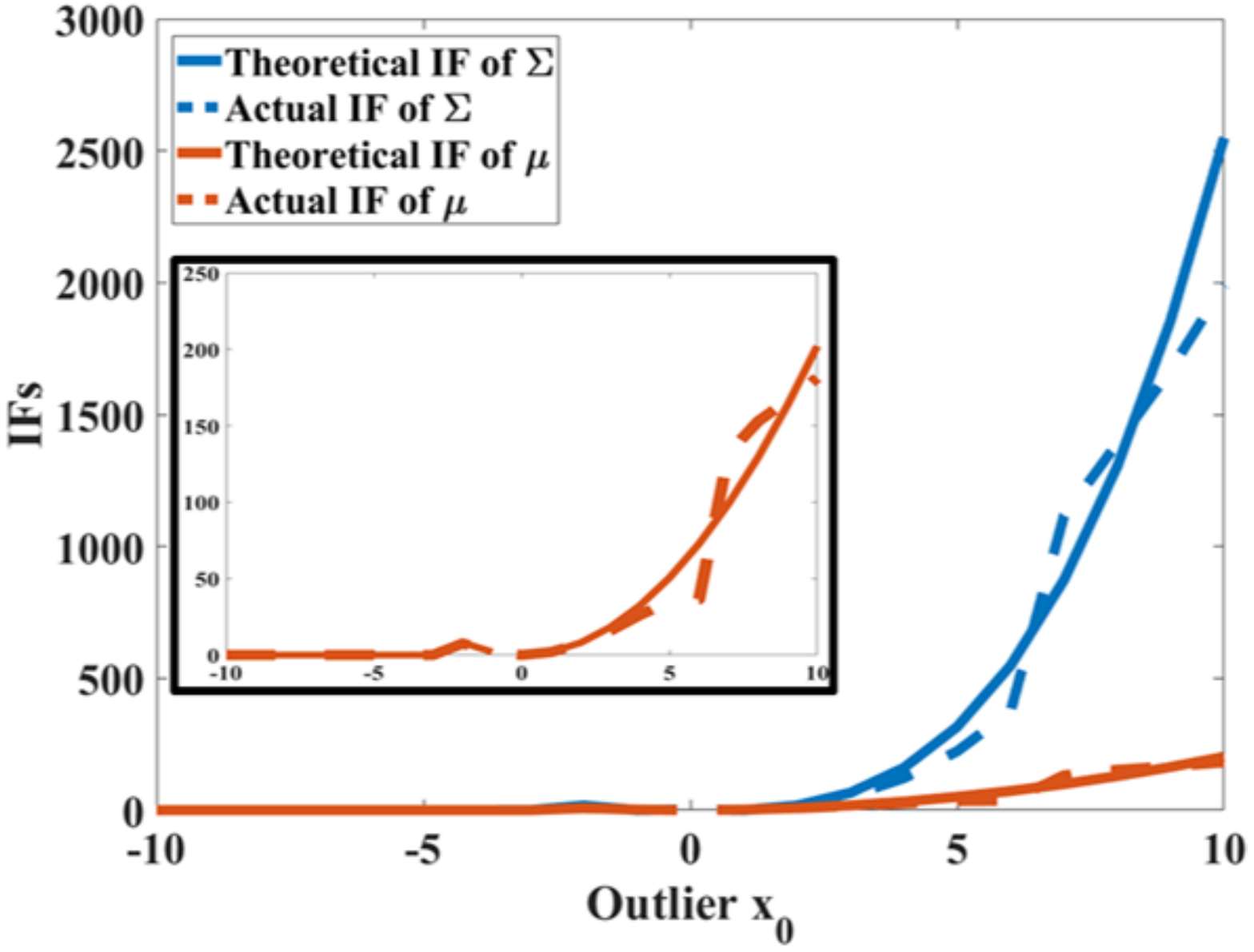}}
	\end{center}
	\vspace{-1.2em}
	\caption{\footnotesize For reproducibility, we follow the work of \cite{hennig2004breakdown} to generate three one-dimensional clusters using the inverse of the Gaussian cumulative distribution function. These three clusters are centred respectively at $\bm{\mu}_1 = 0$, $\bm{\mu}_2 = -5$ and $\bm{\mu}_3 = -10$, and IF curves for $\bm{\mu}_1$ are illustrated. The mixture distributions are the Gaussian, Cauchy, Laplace and GG1.5, as shown in Table \ref{distde}. The theoretical bounds are plotted in solid lines and the actual IFs are plotted in dotted lines. The horizontal dotted lines represent the boundedness (upper bounds) where the mixtures exhibit robustness. The zoomed versions of the IFs of the mean are given in the black box of each figure.}\label{IFcurve}
	\vspace{-1em}
\end{figure*}

\section{Influence functions of the EMM}\label{secIFEMM}
Robustness properties of a single elliptical distribution (or more generally, an M-estimator) have been extensively studied \cite{maronna1976robust, tyler2014breakdown, zoubir2012robust, ollila2012complex}, typically from the perspective of influence functions (IFs) \cite{hampel2011robust}. The IF is an important metric for quantifying the impact of an infinitesimal fraction of outliers on the estimations, which captures the local robustness. However, to the best of our knowledge, there exists no work on the IF of mixture models, especially for the EMMs. %The key difficulty is because 
%%, whilst Hennig provided the BP of location-scale mixture models in \cite{hennig2004breakdown} and proved in the discrete case, a single outlier approaching infinity could breakdown the estimation. He also points out that because the outlier cannot be unrealistically large, it is more interesting to investigate how far this outlier is from the centre of data that can drive away the estimation, which is later defined as ``dissolution point'' \cite{hennig2008dissolution}. Therefore, in this section, we first obtain the IF of the EMM and then provide an approximation to the dissolution point of the EMM.
%estimations on $\bm{\mu}_k$ and $\mathbf{\Sigma}_k$ over all clusters are dependent with each other, IFs of $\bm{\mu}_k$ and $\mathbf{\Sigma}_k$ are also intertwined in an implicit manner. To explicitly model the relationships, we here calculate the IF of one $\bm{\mu}_k$ with other $\bm{\mu}_k$ and $\mathbf{\Sigma}_k$ fixed, and similarly, that of one $\mathbf{\Sigma}_k$ with others fixed. 
To calculate the IFs, we utilise $\mathbf{x}_0$ to denote point-mass outliers, which means that these outliers are point-mass distributed at $\mathbf{x}_0$ \cite{zoubir2012robust}. We also explicitly write the posterior distribution of latent variables as a function of $\mathbf{x}$ ($\xi_j(\mathbf{x}) = \frac{\mathcal{E}_j(\mathbf{x}|\bm{\mu}_j, \mathbf{\Sigma}_j, g_j)\pi_j}{\sum_{k = 1}^K\mathcal{E}_k(\mathbf{x}|\bm{\mu}_k, \mathbf{\Sigma}_k, g_k)\pi_k}$), because in robustness analysis, we need to quantify it with respect to outliers. For simplicity, $t_j$ is also defined as the Mahalanobis distance $(\mathbf{x}-\bm{\mu}_j)^T\mathbf{\Sigma}_j^{-1}(\mathbf{x}-\bm{\mu}_j)$ and $\mathbb{E}[\cdot]$ is the expectation over the true distribution of $\mathbf{x}$. Then, our analysis on the IFs is based on the following two lemmas.
\begin{lemma}\label{IFcovari}
	Consider the mixture of elliptical distributions, $\mathcal{E}_k(\mathbf{x}|\bm{\mu}_k, \mathbf{\Sigma}_k, g_k)$. When data are well separated, upon denoting $(\mathbf{x}_0 - \bm{\mu}_j)$ by $\overline{\mathbf{x}}_0$ for the $j$-th cluster, its IF is given by,
	\small{
		\begin{equation}\label{IFcov}
		\begin{aligned}
		&\mathcal{I}_{\mathbf{\Sigma}_j}(\mathbf{x}_0)\! =\! - \frac{\xi_j(\mathbf{x}_0)\psi_j(\overline{\mathbf{x}}_0^T\mathbf{\Sigma}_j^{-1}\overline{\mathbf{x}}_0)}{w_2}\overline{\mathbf{x}}_0\overline{\mathbf{x}}_0^T \\
		& + \mathbf{\Sigma}_j^{\frac{1}{2}}\!\left[\frac{2w_1\cdot \xi_j(\mathbf{x}_0)\psi_j(\overline{\mathbf{x}}_0^T\mathbf{\Sigma}_j^{-1}\overline{\mathbf{x}}_0)\overline{\mathbf{x}}_0^T\mathbf{\Sigma}_j^{-1}\overline{\mathbf{x}}_0 + w_2 \cdot\xi_j(\mathbf{x}_0)\mathbf{I} }{2(Mw_1-w_2)w_2}\right]\!\mathbf{\Sigma}_j^{\frac{1}{2}},
		\end{aligned}
		\end{equation}}
	\normalsize
	where $w_1$ and $w_2$ are constants (irrelevant to the outlier $\mathbf{x}_0$) given by
	\small
	\begin{equation}\label{defini}
	\begin{aligned}
	w_1 & = \frac{\mathbb{E}[(\xi_j(\mathbf{x})-\xi_j^2(\mathbf{x}))\psi_j^2(t)t^2]+\mathbb{E}[\xi_j(\mathbf{x})\psi'_j(t)t^2]}{M(M+1)}\\
	&~~~~+\frac{\mathbb{E}[(\xi_j(\mathbf{x})-\xi_j^2(\mathbf{x}))\psi_j(t)t]}{M}+ \frac{\mathbb{E}[\xi_j(\mathbf{x})-\xi_j^2(\mathbf{x})]}{4},\\
	w_2 & \!= \!\frac{\pi_j}{2} \!-\! \frac{\mathbb{E}[(\xi_j(\mathbf{x})-\xi_j^2(\mathbf{x}))\psi_j^2(t)t^2]\!+\!\mathbb{E}[\xi_j(\mathbf{x})\psi'_j(t)t^2]}{M(M+1)}.
	\end{aligned}
	\end{equation}
\end{lemma}
\begin{lemma}\label{IFmean}
	Consider the mixture of elliptical distributions, $\mathcal{E}_k(\mathbf{x}|\bm{\mu}_k, \mathbf{\Sigma}_k, g_k)$. When data are well separated, for the $j$-th cluster, its IF on the mean is given by
	\begin{equation}
	\mathcal{I}_{\bm{\mu}_j}(\mathbf{x}_0) = \frac{1}{w_3}\xi_j(\mathbf{x}_0)\psi_j(\mathbf{x}_0)(\mathbf{x}_0-\bm{\mu}_j),
	\end{equation}
	where $w_3$ is a constant (irrelevant to the outlier $\mathbf{x}_0$) given by
	\begin{equation}
	\begin{aligned}
	w_3 = &\frac{2\mathbb{E}[\xi_j(\mathbf{x})\psi'_j(t)t]}{M}\\
	& + \mathbb{E}[\xi_j(\mathbf{x})\psi_j(t)]+\frac{2\mathbb{E}[(\xi_j(\mathbf{x})-\xi_j^2(\mathbf{x}))\psi_j^2(t)t]}{M}.
	\end{aligned}
	\end{equation}
\end{lemma}

Proofs of the two lemmas are provided in the Appendices-\ref{prooflemma3} and \ref{prooflemma4}. The actual\footnote{The actual IF is obtained via numerical tests on the actual difference between the estimated parameter and the ground truth when increasing the absolute value of a single outlier, to establish whether an outlier could totally break down the estimation; this is cumbersome and requires extensive repeated estimations to obtain the curve.} and theoretical IF curves of the four EMMs are plotted in Fig. \ref{IFcurve}, showing that in practice the robustness of the EMMs can be well captured by our theoretical bounds. 

More importantly, the robustness of the EMM can also be analysed from Lemmas \ref{IFcovari} and \ref{IFmean},  and is determined by $\psi_k(\cdot)$ (or $g_k(\cdot)$) of each cluster. Specifically, when $\mathbf{x}_0 \rightarrow \infty$, $\mathcal{I}_{\mathbf{{\Sigma}}_j}(\mathbf{x}_0)$ is bounded (defined as \textit{covariance robust}) only when $\psi_j(t)t$ is bounded for $t\!\rightarrow\! \infty$, which leads to bounded $\psi_j(\overline{\mathbf{x}}_0^T\mathbf{\Sigma}_j^{-1}\overline{\mathbf{x}}_0)\overline{\mathbf{x}}_0^T\mathbf{\Sigma}_j^{-1}\overline{\mathbf{x}}_0$ in \eqref{IFcov}. Likewise, bounded $\mathcal{I}_{\bm{\mu}_j}(\mathbf{x}_0)$ (defined as \textit{mean robust}) requires bounded $\psi_j(t)\sqrt{t}$, which is slightly more relaxed than the requirement of \textit{covariance robust}. For example, in Fig. \ref{IFcurve}, by inspecting the boundedness of the curves, we find that the Gaussian and GG1.5 mixtures are neither \textit{covariance robust} and \textit{mean robust}, while the Cauchy mixtures are both \textit{covariance robust} and \textit{mean robust}. For the Laplace mixtures, they are not \textit{covariance robust} but are \textit{mean robust}, which shows that the \textit{covariance robust} is more stringent than the \textit{mean robust}. 
Thus, the developed bounds provide an extremely feasible and convenient treatment for qualifying or designing the robustness within EMMs.

\section{Experimental results}
Our experimental settings are first detailed in Section \ref{parameterssetiing}. We then employ in Section \ref{toyexamples} two toy examples to illustrate the flexibility of EMMs in capturing different types of data. This also highlights the virtues of our method in universally solving EMMs. In Section \ref{syntheticeva}, we systematically compare our EMM solver with other baselines on the synthetic dataset, followed by a further evaluation on the image data of BSDS500 in Section \ref{imagedata}.

\subsection{Parameter settings and environments}\label{parameterssetiing}
\textbf{Baselines:} We compared the proposed method (\textbf{Our}) with the regular manifold optimisation (\textbf{RMO}) method without reformulation (i.e., updating $\bm{\mu}_k$ and $\mathbf{\Sigma}_k$ separately) and the \textbf{IRA} method, by optimising different EMMs over various data structures. It should be pointed out that the IRA includes a range of existing works on solving certain EMMs, e.g., the standard EM algorithm for the Gaussian distribution, \cite{peel2000robust} for the $t$-distribution and \cite{browne2015mixture} for the hyperbolic distribution. Besides, the convergence criterion in all the experiments was set by the cost decrease of adjacent iterations of less than $10^{-10}$. For our method and the RMO, that involved manifold optimisation, we have utilised the default conjugate gradient solver in the Manopt toolbox \cite{manopt}. We should also point out that we evaluated all the methods on original EMM problems and due to the fact that priors are highly data-dependent, we leave the reasonable and comprehensive evaluations on regularised EMMs as part of our future work.

\textbf{Performance objectives:} We compared our method with the RMO and IRA methods by comprehensively employing 9 different elliptical distributions as components within EMMs. These are listed in Table \ref{distde}, with their properties provided in Table \ref{alldist}, where the Cauchy distribution is a special case of the student-$t$ distribution with $v=1$. We should also point out that the non-geodesic elliptical distributions cannot be solved by the IRA method \cite{zhang2013multivariate,sra2013geometric}. In contrast, as shown below, our method can provide a stable and fast solution even for the non-geodesic elliptical distributions.

\begin{table*}[!h]
	\renewcommand\arraystretch{2}
	\caption{Details of the 9 elliptical distributions used for assessments}\label{distde}
	\centering
	\scriptsize{{
			\begin{tabular}{c|cccc}\hline \hline
				\multicolumn{1}{c}{}   & Gaussian &  Student-$t$ ($v=1$ and $v=10$)  & GG1.5      & Logistic  \\\hline
				\multirow{1}{*}{$g(t)$ of \eqref{RES}}      &\multirow{1}{*}{$g(t) \propto \exp(-0.5t)$}  &   \multirow{1}{*}{$g(t) \propto (1+\nicefrac{t}{v})^{-\nicefrac{(M+v)}{2}}$}           & \multirow{1}{*}{$g(t) \propto \exp(-0.5t^{1.5})$}         & \multirow{1}{*}{$g(t) \propto \frac{\exp(-t)}{(1+\exp(-t))^2}$}    \\\hline
				\multicolumn{1}{c}{} &  Laplace    & Weib0.9 & Weib1.1    & Gamma1.1 \\\hline
				\multirow{1}{*}{$g(t)$ of \eqref{RES}} &  \multirow{1}{*}{$g(t) \propto \frac{\mathcal{K}_{(1-0.5M)}(\sqrt{2t})}{\sqrt{0.5t}^{0.5M-1}}$}  & \multirow{1}{*}{$g(t) \propto t^{-0.1}\exp(-0.5t^{0.9})$}       & \multirow{1}{*}{$g(t) \propto= t^{0.1}\exp(-0.5t^{1.1})$}    & \multirow{1}{*}{$g(t) \propto t^{0.1}\exp(-0.5t)$}     \\\hline\hline
				\multicolumn{1}{c|}{Note:} &\multicolumn{4}{l}{$\mathcal{K}_x(y)$ is the modified Bessel function of the second kind. Student-$t$ with $v=1$ is the Cauchy distribution.}\\\hline
	\end{tabular}}}
\vspace{-1em}
\end{table*}
\begin{table}[!htb]
	\centering
	\caption{Properties of the elliptical distributions used for evaluations}\label{alldist}
	\scriptsize{
		\begin{tabular}{ccccc}\hline\hline
			& Gaussian        & Student-$t$                & Laplace              & GG1.5                    \\\hline
			Covariance Robust                 & No              & Yes                   & No                   & No             \\
			Mean Robust                 & No              & Yes                   & Yes                  & No                     \\
			Heavily Tailed                 & No              & Yes                   & Yes                  & No                       \\
			Geodesic Convex                 & Yes             & Yes                   & Yes                  & Yes           \\\hline\hline 
			&  Logistic & Weib0.9         & Weib1.1         & Gamma1.1                            \\\hline
			Covariance Robust                   & No       & No              & No              & No                                  \\
			Mean Robust                    & No       & No              & No              & No                             \\
			Heavily Tailed                    & No       & Yes             & No              & No                                 \\
			Geodesic Convex                       & Yes      & Yes             & No              & No                   \\\hline\hline        
	\end{tabular}}
\vspace{-1em}
\end{table}
\textbf{Synthetic datasets:}
We generated the synthetic dataset via randomly choosing the mean and the covariance, except for the \textit{separation} $c$ and \textit{eccentricity} $e$ \cite{lindsay1995mixture,dasgupta1999learning}, which were controlled to comprehensively evaluate the proposed method under various types of data structures. 
The separation, $c$, of two clusters $k_1$ and $k_2$ is defined as  $||\bm{\mu}_{k_1}-\bm{\mu}_{k_2}||^2\geqslant c\cdot \max\{\mathrm{tr}(\mathbf{{{\Sigma}}}_{k_1}),\mathrm{tr}(\mathbf{{{\Sigma}}}_{k_2})\}$, and the eccentricity, $e$, is defined as a ratio of the largest and the smallest eigenvalue of the covariance matrix within one cluster. 
The smaller value of $c$ indicates the larger overlaps between clusters; the smaller value of $e$ means more spherically distributed clusters. In total, we generated $3\times 2=6$ types of synthetic datasets, whereby $M$ and $K$ were set in pairs to $\{8,8\}$, $\{16,16\}$, and $\{64,64\}$; $c$ and $e$ were set in pairs to  $\{10,10\}$ and $\{0.1,1\}$ to represent the two extreme cases. Each synthetic dataset contained $10,000$ samples in total ($N = 10,000$) drawn from different mixtures of Gaussian distributions. For each test case (i.e., for each method and for each EMM), we repeatedly ran the optimisation over $50$ trials, with random initialisations. Finally, we recorded average values of the iterations, the computational time and the final cost. When the optimisation failed, i.e., converging to singular covariance matrices or infinite negative likelihood, we also recorded and calculated the optimisation fail ratio within the $50$ initialisations for each test case, to evaluate the stability in optimisation.

\textbf{BSDS500 dataset:} 
Finally, we evaluated our method on the image data, over two typical tasks. The first was related to image segmentation, where all the 500 pictures in the Berkeley segmentation dataset BSDS500 \cite{amfm_pami2011} were tested and reported in our results. We set $K=2$ in this task in order to clearly show the effects of different EMMs in segmentation (as shown in Fig. \ref{esifigus}). Evaluation over multiple parameters $K$ was included in the second task. Moreover, each optimisation was initialised by the \textit{k-means++} using the vl-feat toolbox \cite{vedaldi08vlfeat}, which is a typical initialisation method in clustering tasks such as the \textit{k-means} clustering. The cost, iterations and computational time were recorded for all the 500 pictures. In the second task, our evaluation was implemented on another challenging task, by modelling and clustering $3\times3$ and $5\times5$ image patches from the image dataset. It needs to be pointed out that this task is a core part of many applications, such as image denoising, image super-resolution and image/video compression, where similarities of image patches play an important role. Specifically, we randomly extracted $100$ patches ($3\times3$ and $5\times5$) from each image in the BSDS500, and vectorised those patches as data samples. Thus, we finally obtained the test data with sizes $50,000\times27$ and $50,000\times75$, where $K$ was set to $3$ and $9$. Also, we ran each optimisation with $50$ times random initialisations, and recorded the average final cost and the standard deviation.

\subsection{Toy examples}\label{toyexamples}
Before comprehensively evaluating our method, we first provide some intuition behind its performance based on flower-shaped data with 4 clusters ($N = 10,000$) (shown in Fig. \ref{toypre}-(a)). The flexibility of the EMMs via our method is illustrated by: (i) adding 100$\%$ uniform noise (i.e., $10,000$ noisy samples), as shown in Fig. \ref{toypre}-(b); (ii) replacing two clusters by the Cauchy samples with the same mean and covariance matrices, as shown in Fig. \ref{toypre}-(c). The five distributions that were chosen as components in EMMs are shown in Fig. \ref{toypre}-(d).

The optimised EMMs are shown in Fig. \ref{toyfinalresult}. From this figure, we find that the GMM is inferior in modelling noisy or unbalanced data. This is mainly due to its lack of robustness, and thus similar results can be found in another non-robust EMM, i.e., the GG1.5. In contrast, for a robust EMM, such as the Cauchy and the Laplace, the desirable level of estimation is ensured in both cases. Therefore, a universal solver is crucial as it enables flexible EMMs can be well optimised for different types of data.

\begin{figure}[!h]
	\begin{center}
		\subfigure[Ground truth]{\includegraphics[width=0.21\textwidth]{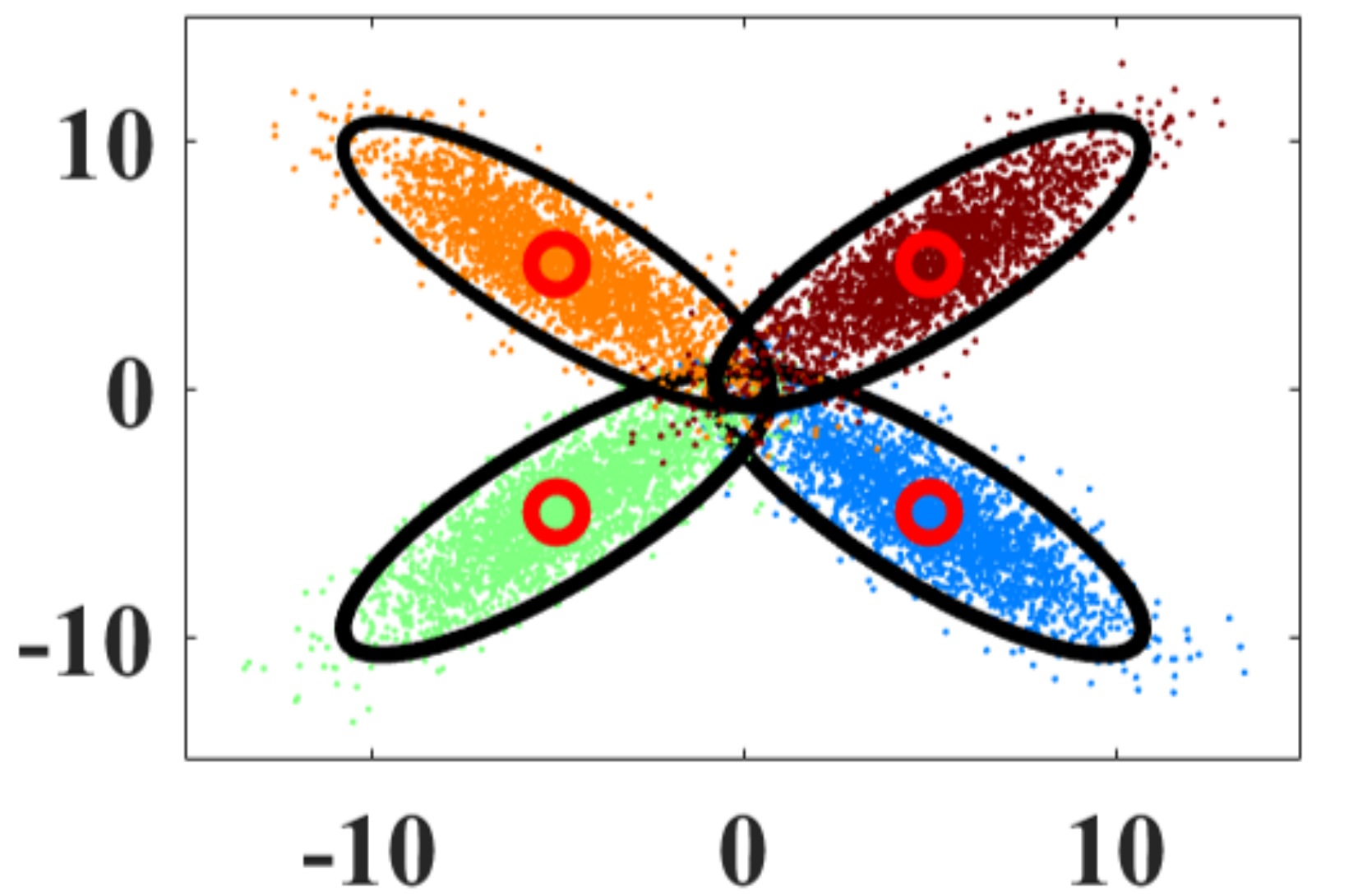}}
		\subfigure[Case 1: $100\%$ Noise]{\includegraphics[width=0.21\textwidth]{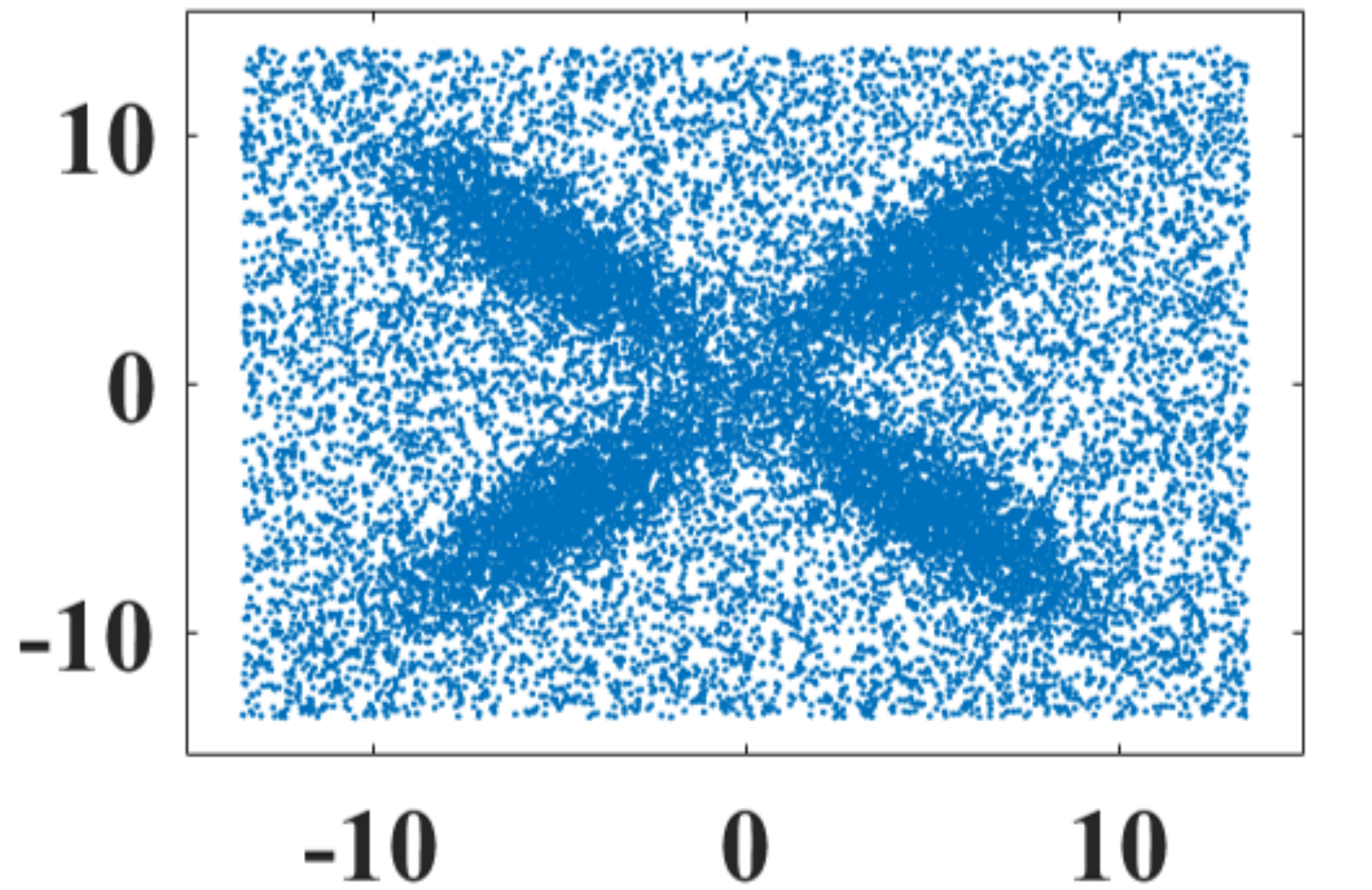}}
		\subfigure[Case 2: Cauchy clusters]{\includegraphics[width=0.21\textwidth]{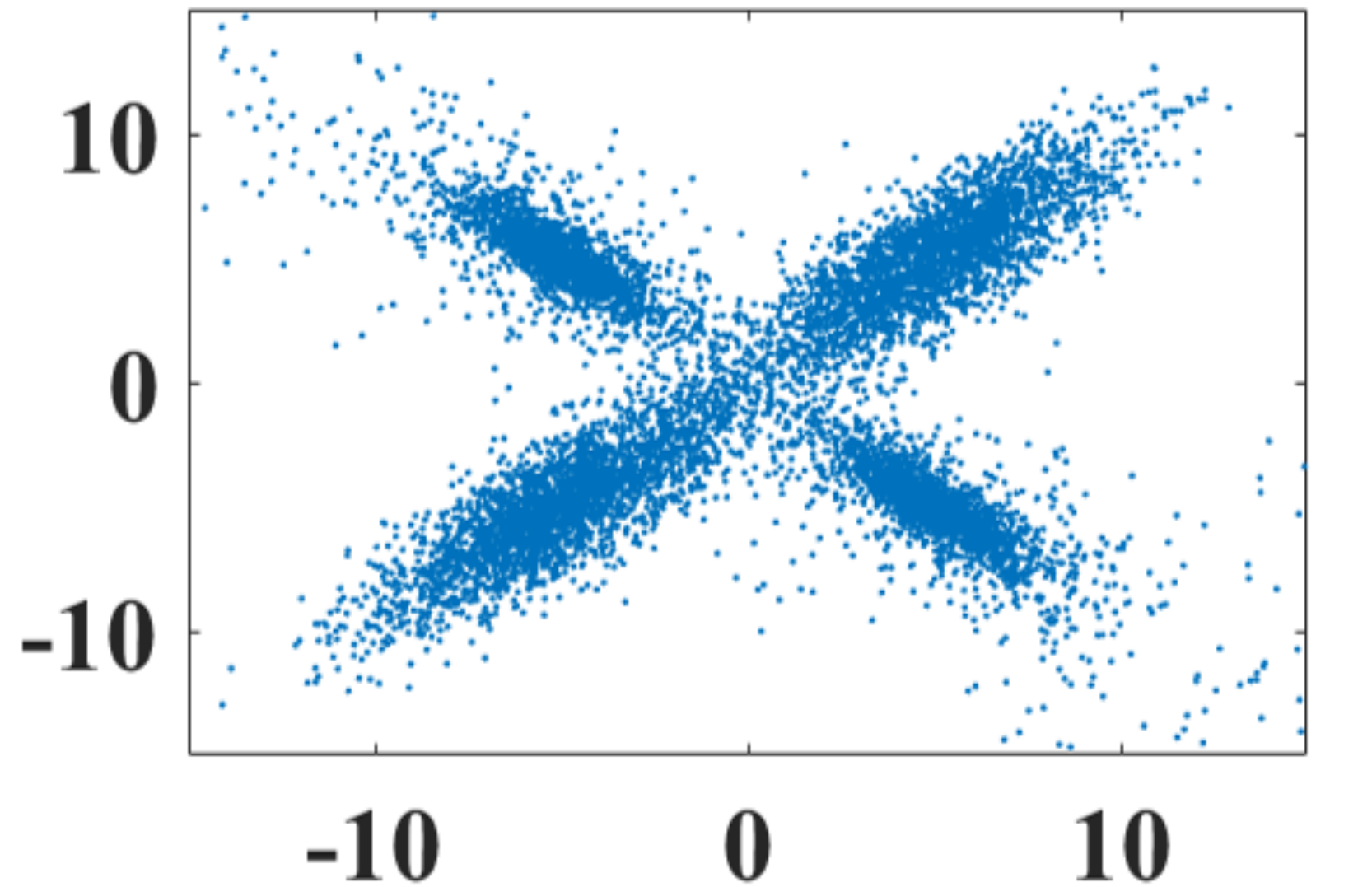}}
		\subfigure[Tails of distributions]{\includegraphics[width=0.21\textwidth]{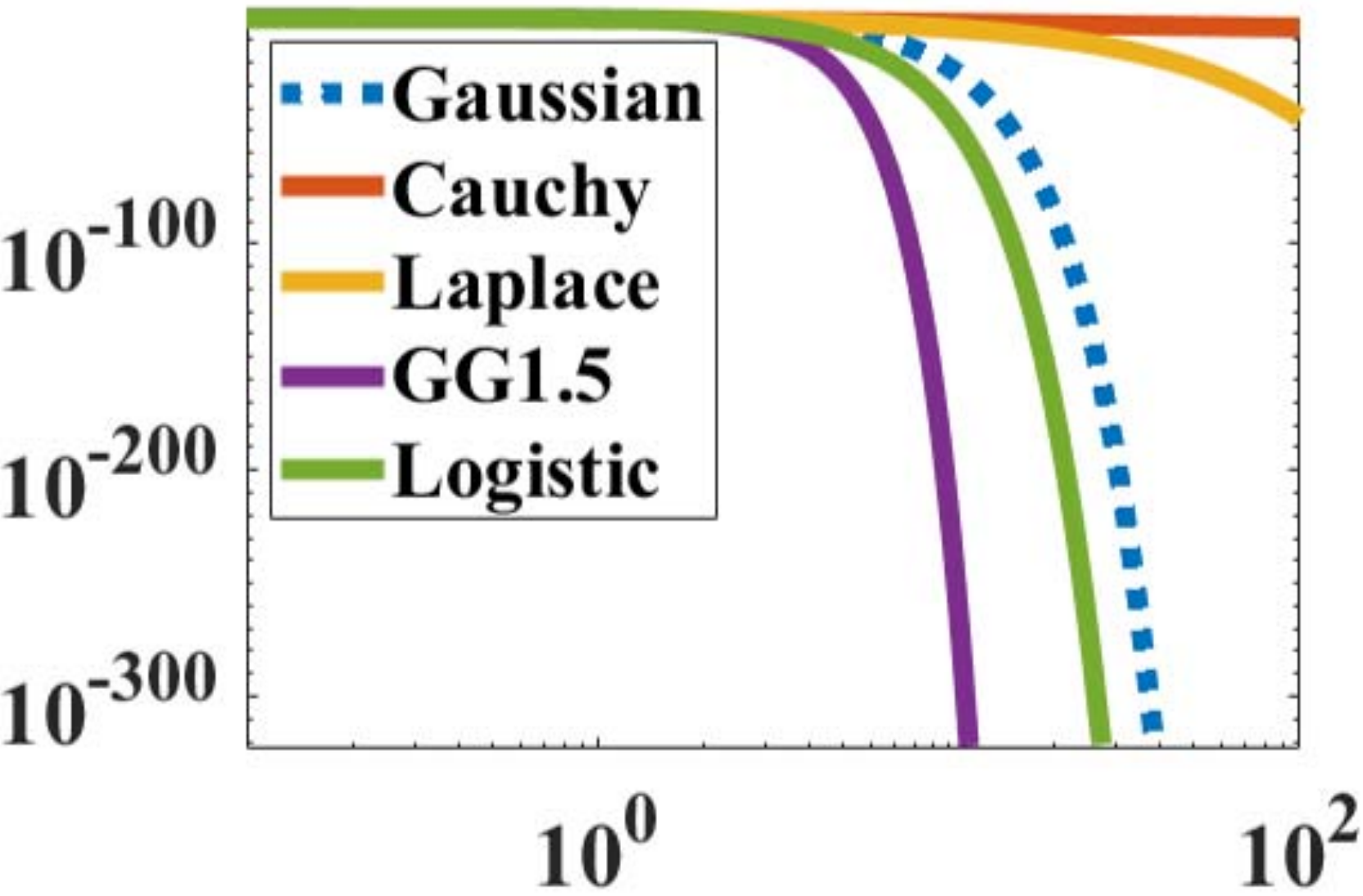}}
	\end{center}
\vspace{-1em}
	\caption{Toy examples consisting of four Gaussian sets. (a) Data structure: The red circles represent the mean values at $(5,5)$, $(5,-5)$, $(-5,5)$ and $(-5,-5)$, and the black circles denote covariance ellipses including 95$\%$ data samples of each Gaussian distribution. (b) Test case adding 100$\%$ uniform noisy samples to the data. (c) Test case of data that consist of two Gaussian and two Cauchy sets. (d) Tails of the distributions which were utilised in this test. It needs to be pointed out that the Cauchy samples in (c) are spread over a wide range, so that we show (c) within $(\pm15,\pm15)$ for illustration convenience.}\label{toypre}
\end{figure}

\begin{figure*}[!h]
	\begin{center}
		\subfigure{\includegraphics[width=0.17\textwidth]{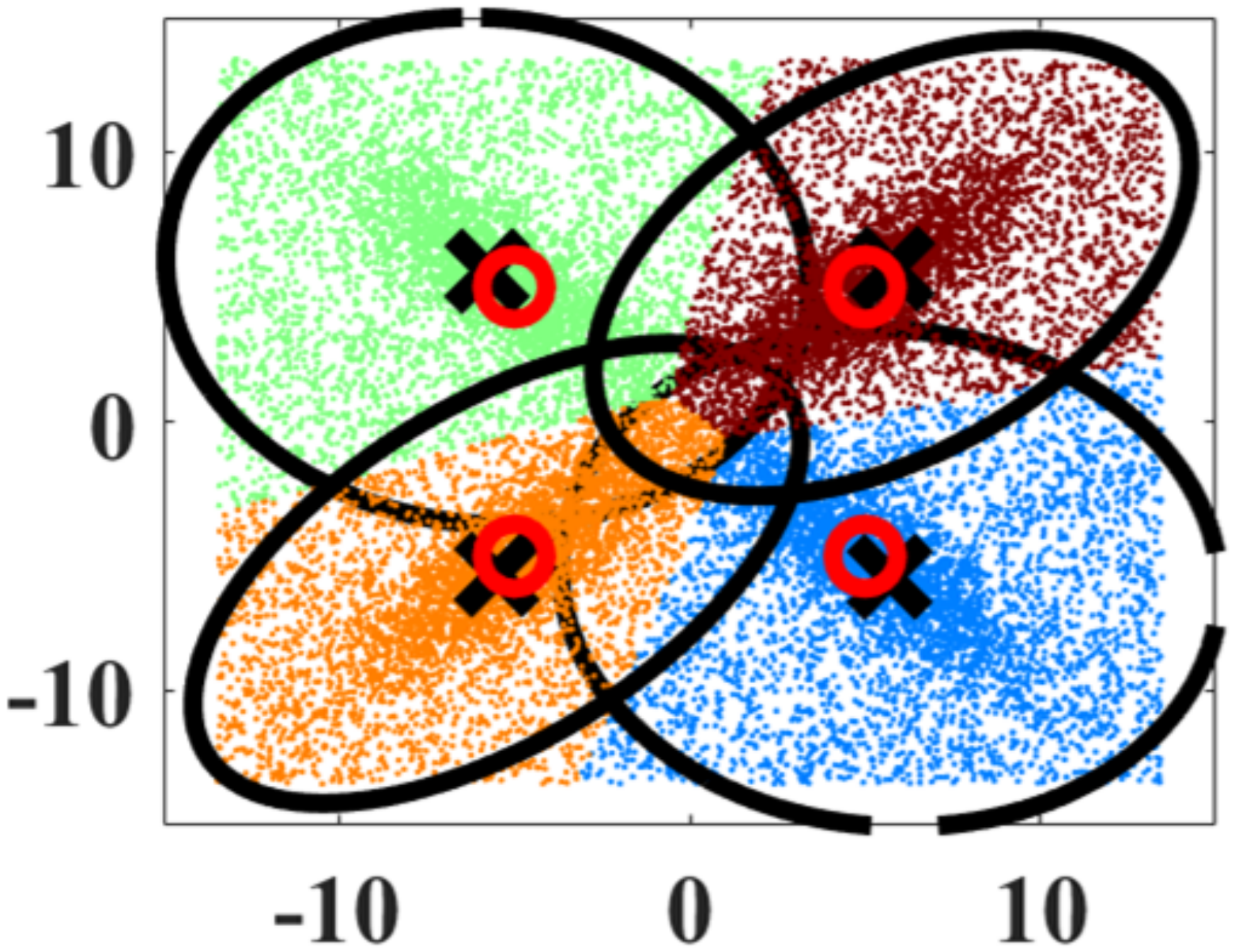}}
		\subfigure{\includegraphics[width=0.17\textwidth]{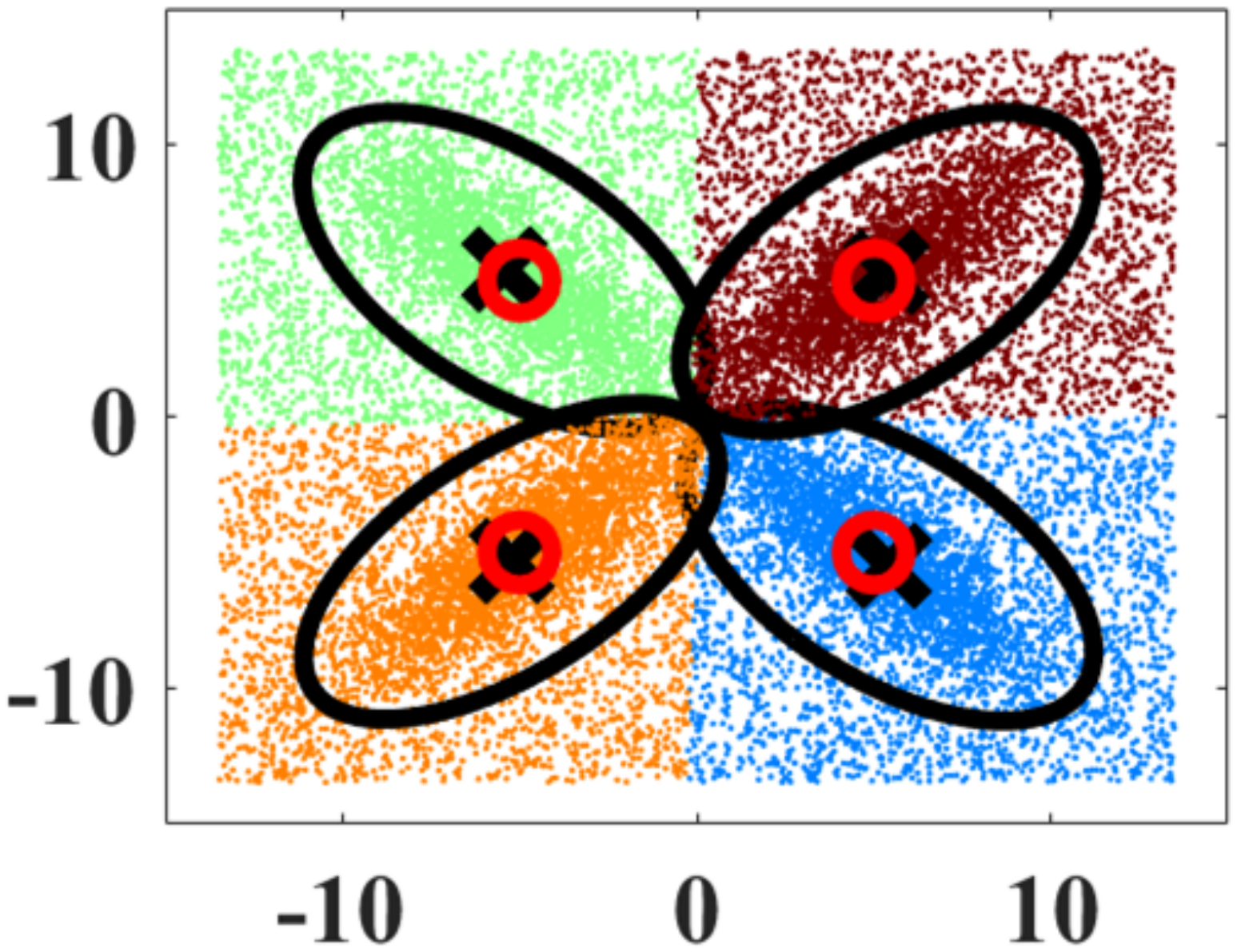}}
		\subfigure{\includegraphics[width=0.17\textwidth]{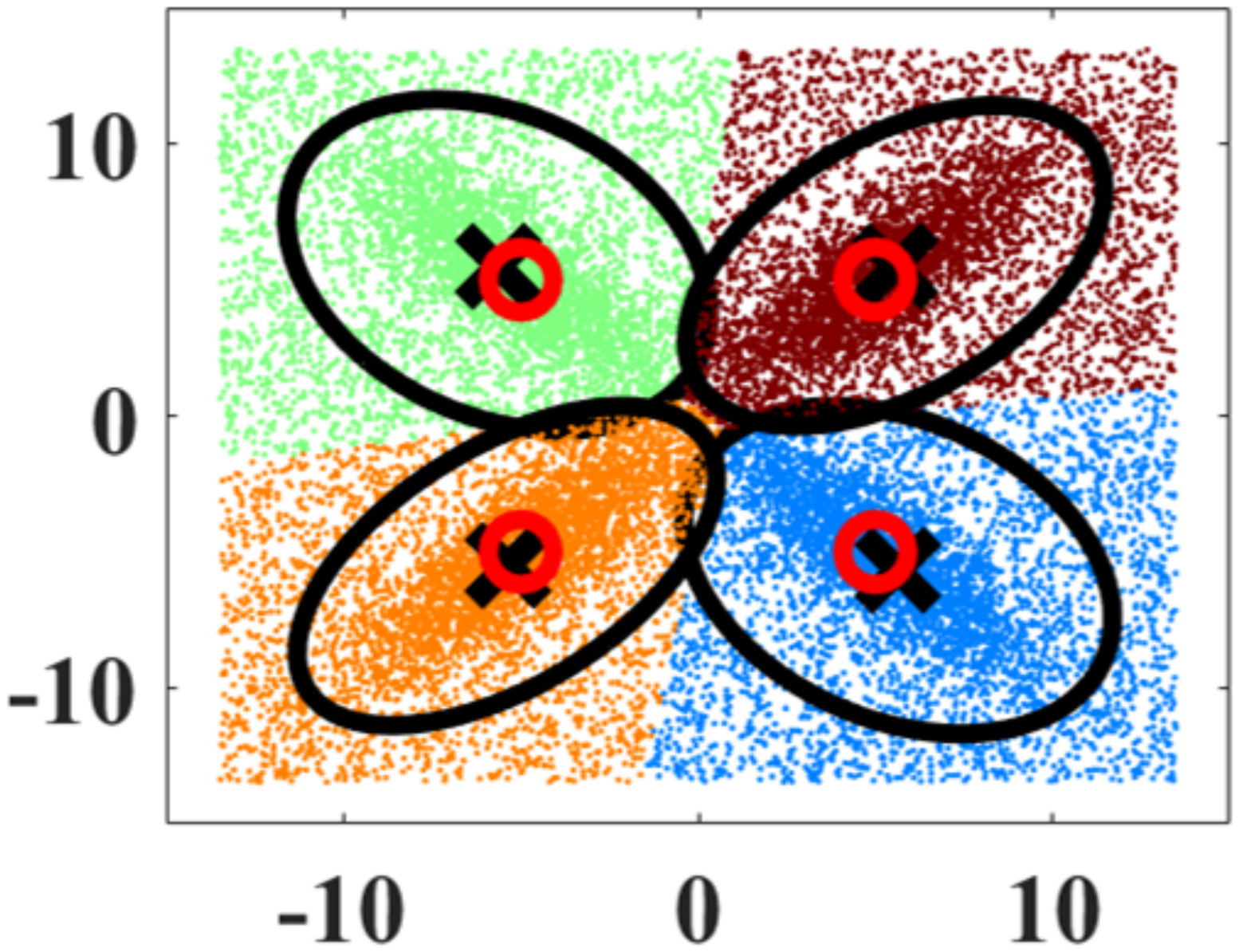}}
		\subfigure{\includegraphics[width=0.17\textwidth]{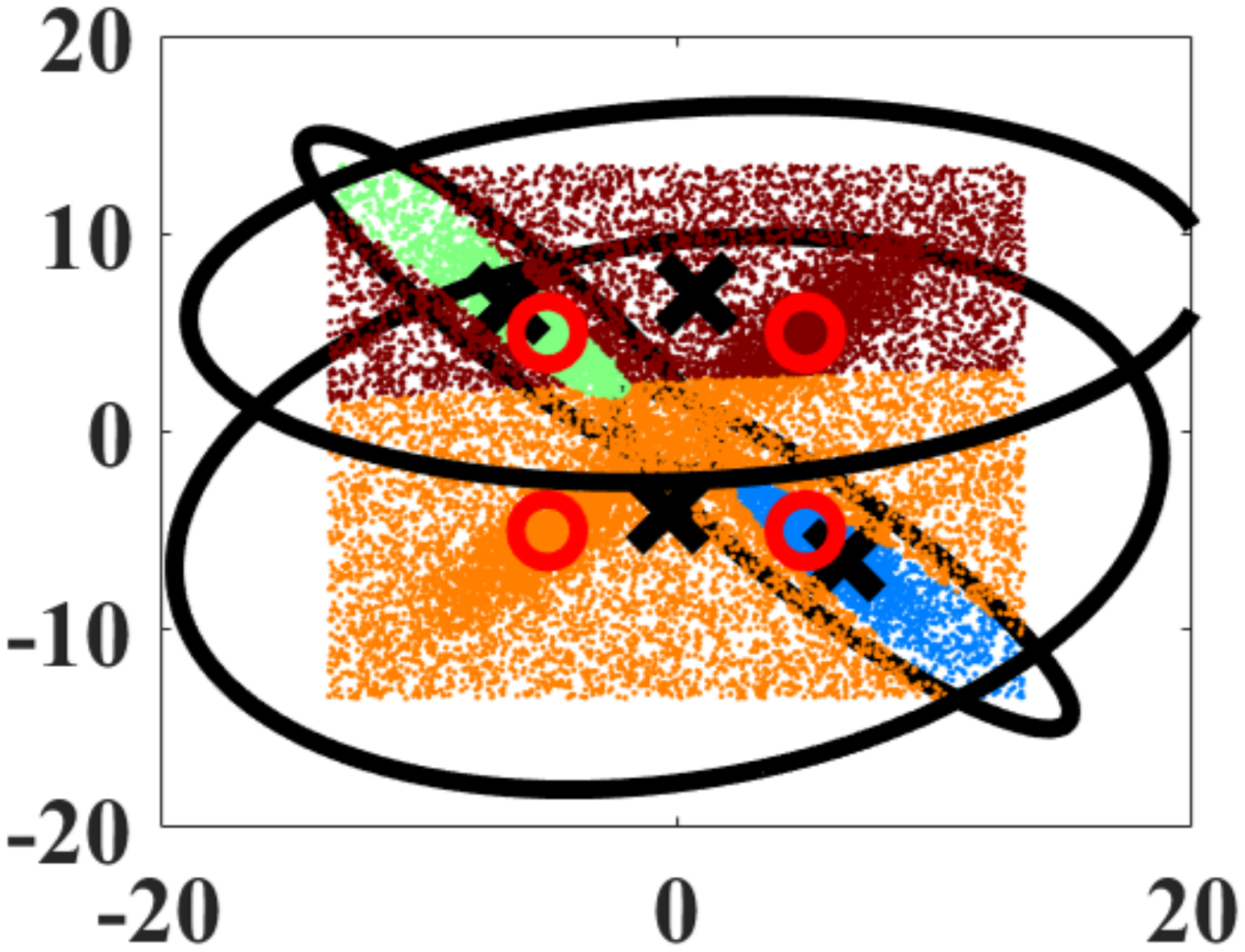}}
		\subfigure{\includegraphics[width=0.17\textwidth]{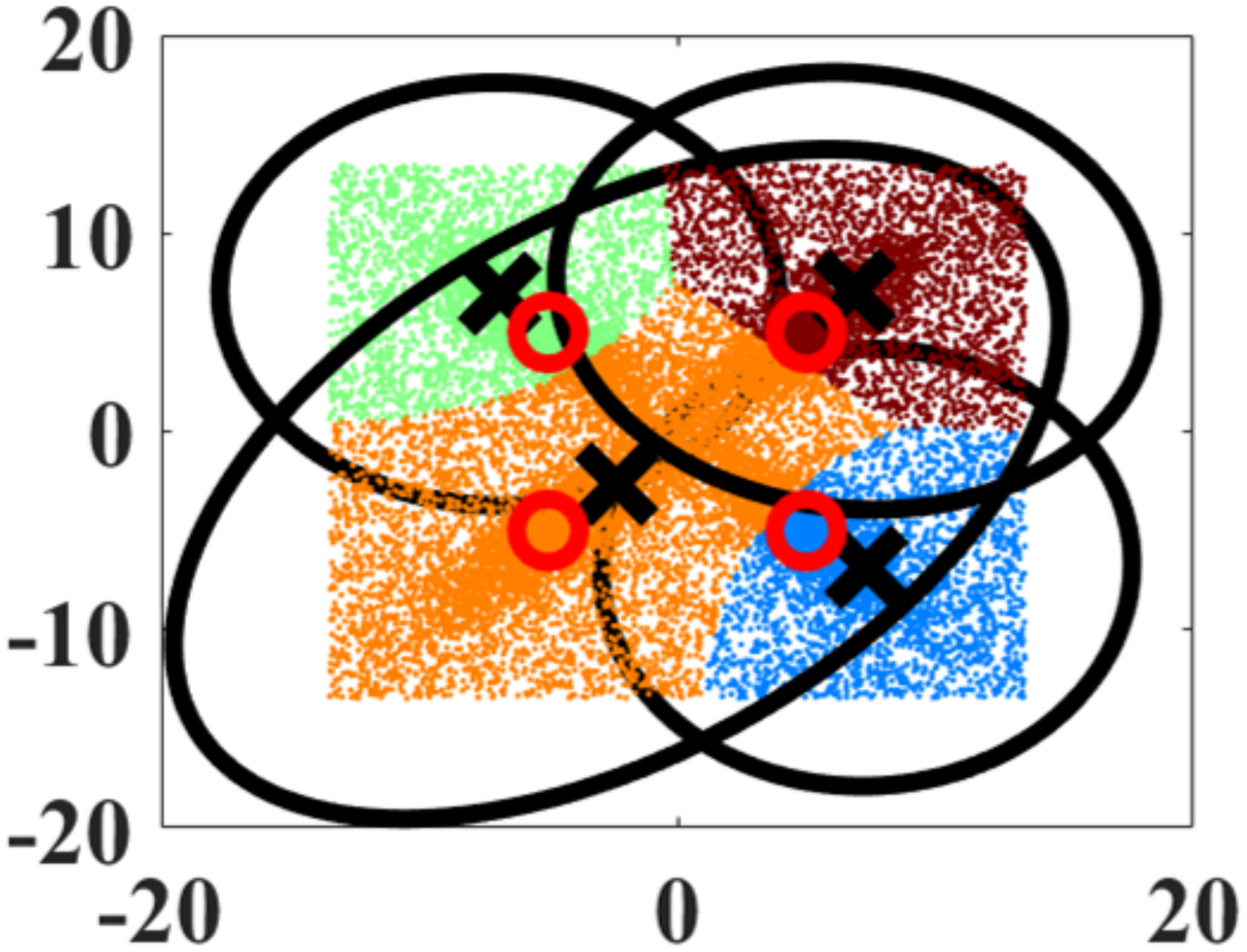}}
		\setcounter{subfigure}{0}
		\subfigure[Gaussian]{\includegraphics[width=0.17\textwidth]{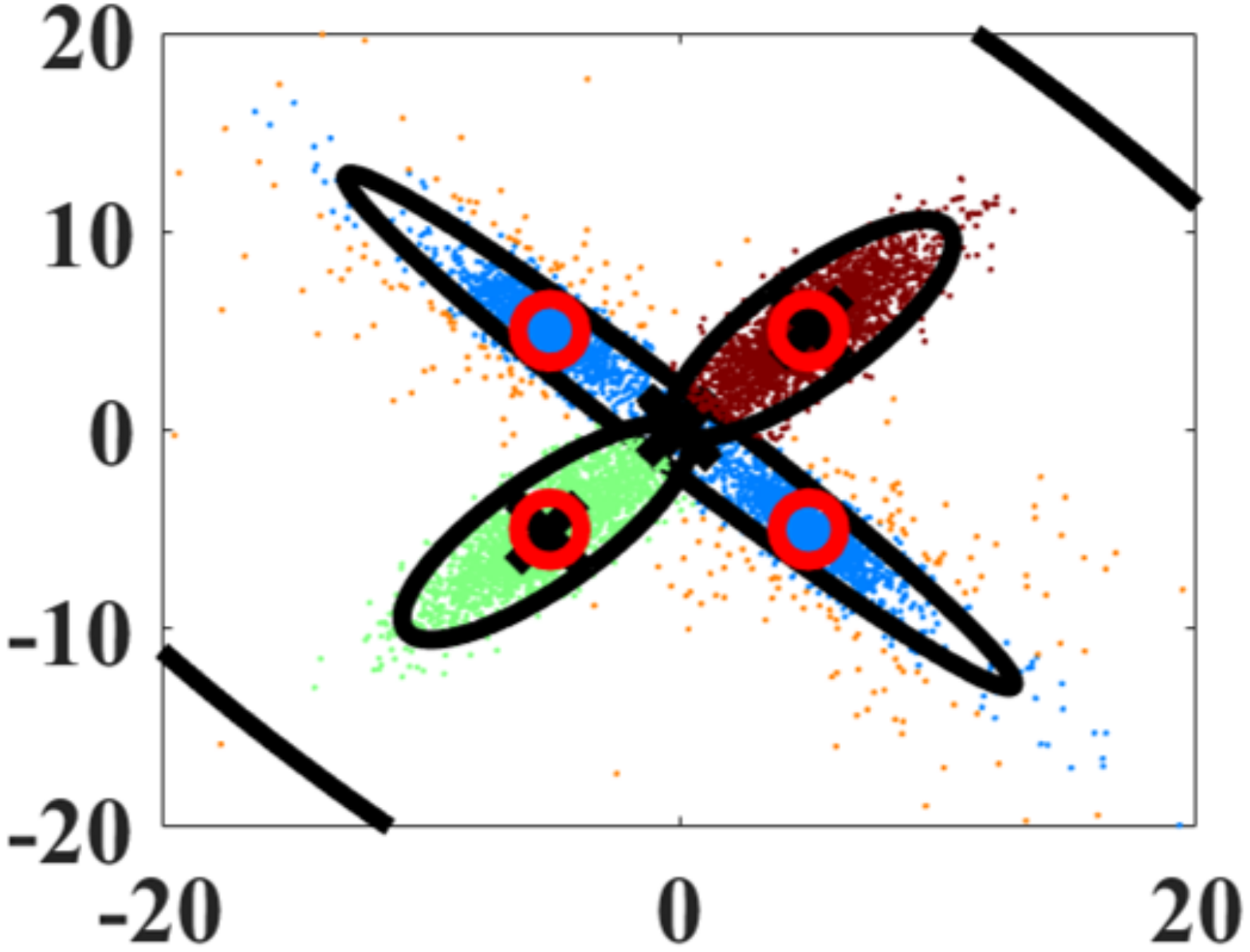}}
		\subfigure[Cauchy]{\includegraphics[width=0.17\textwidth]{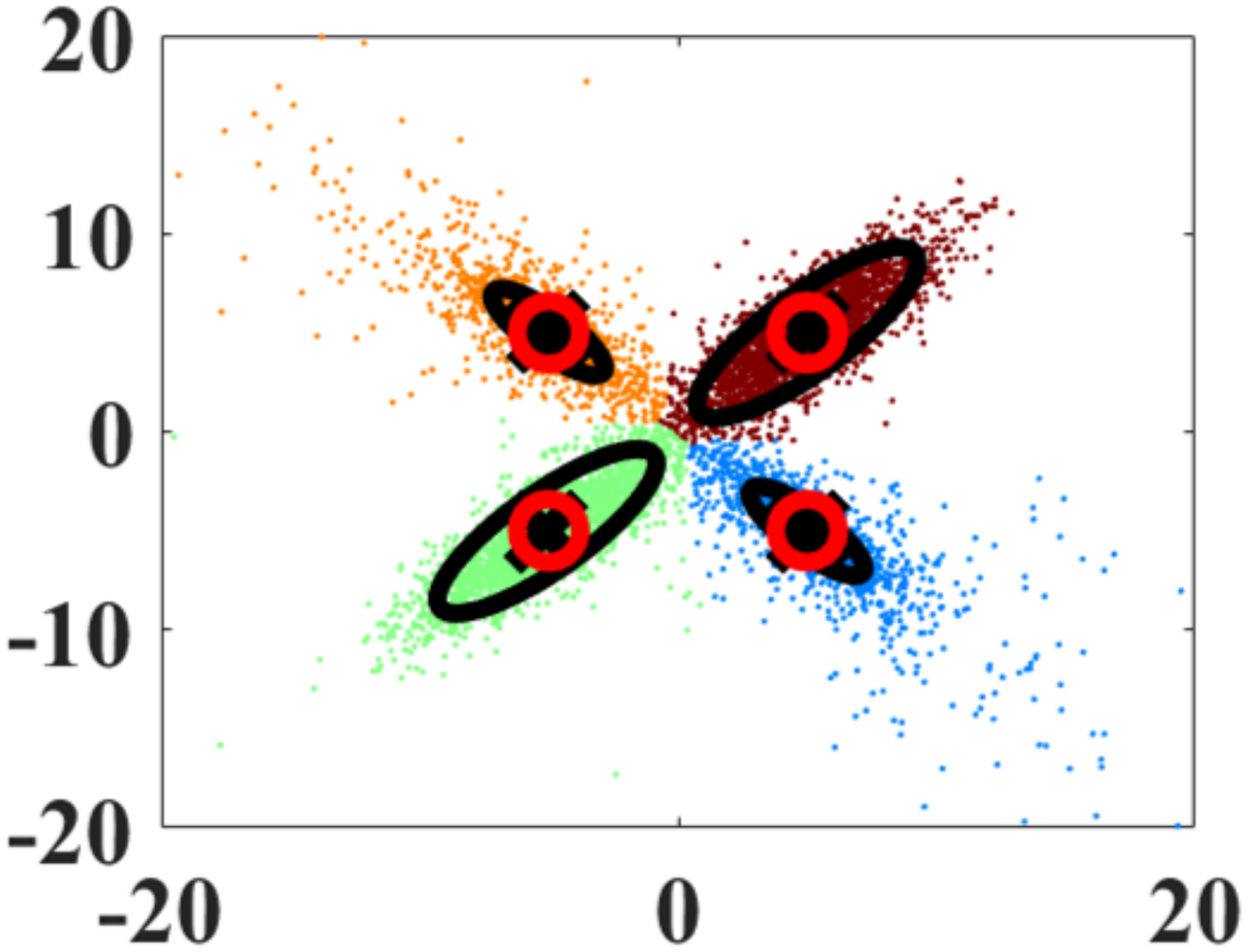}}
		\subfigure[Laplace]{\includegraphics[width=0.17\textwidth]{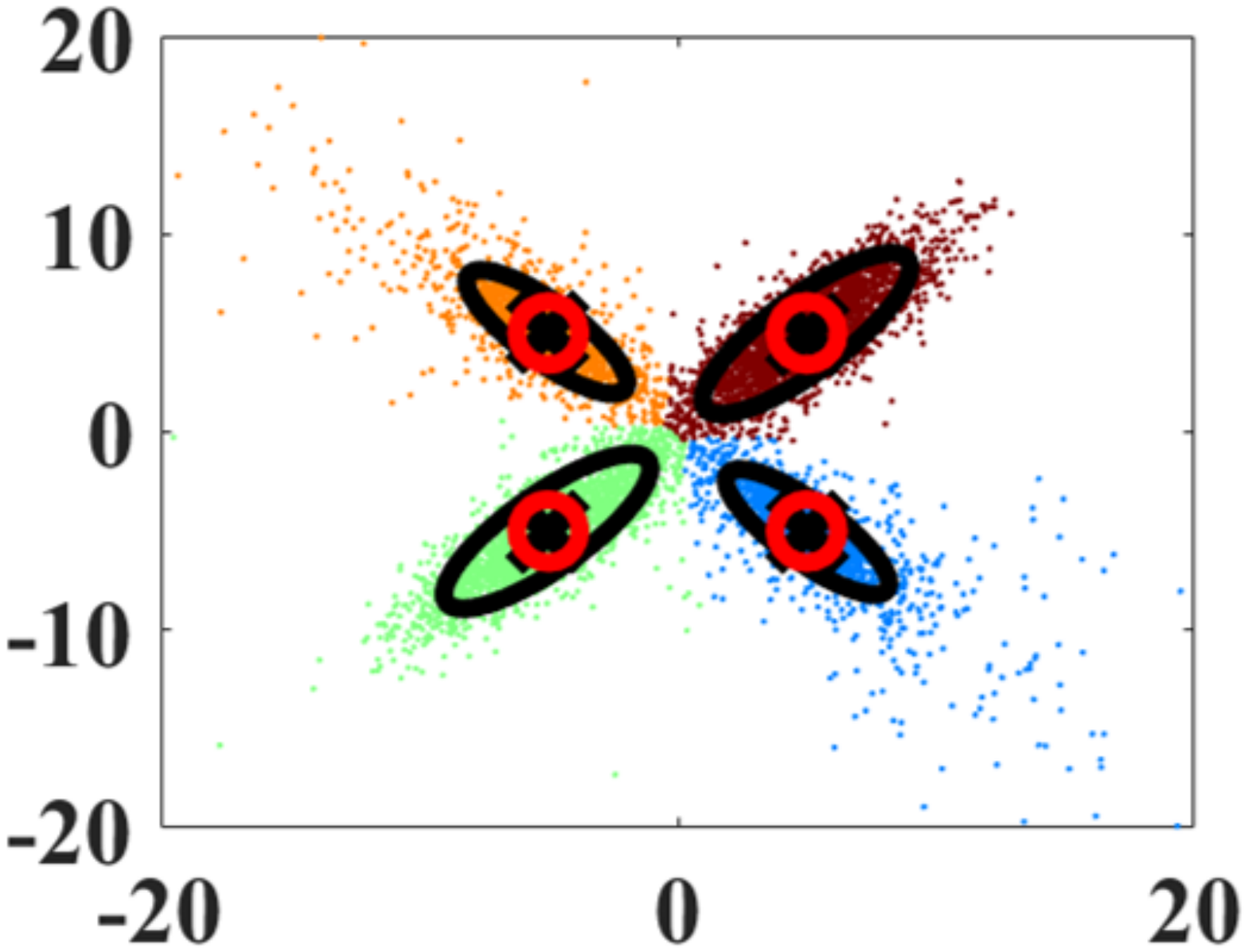}}
		\subfigure[Logistic]{\includegraphics[width=0.17\textwidth]{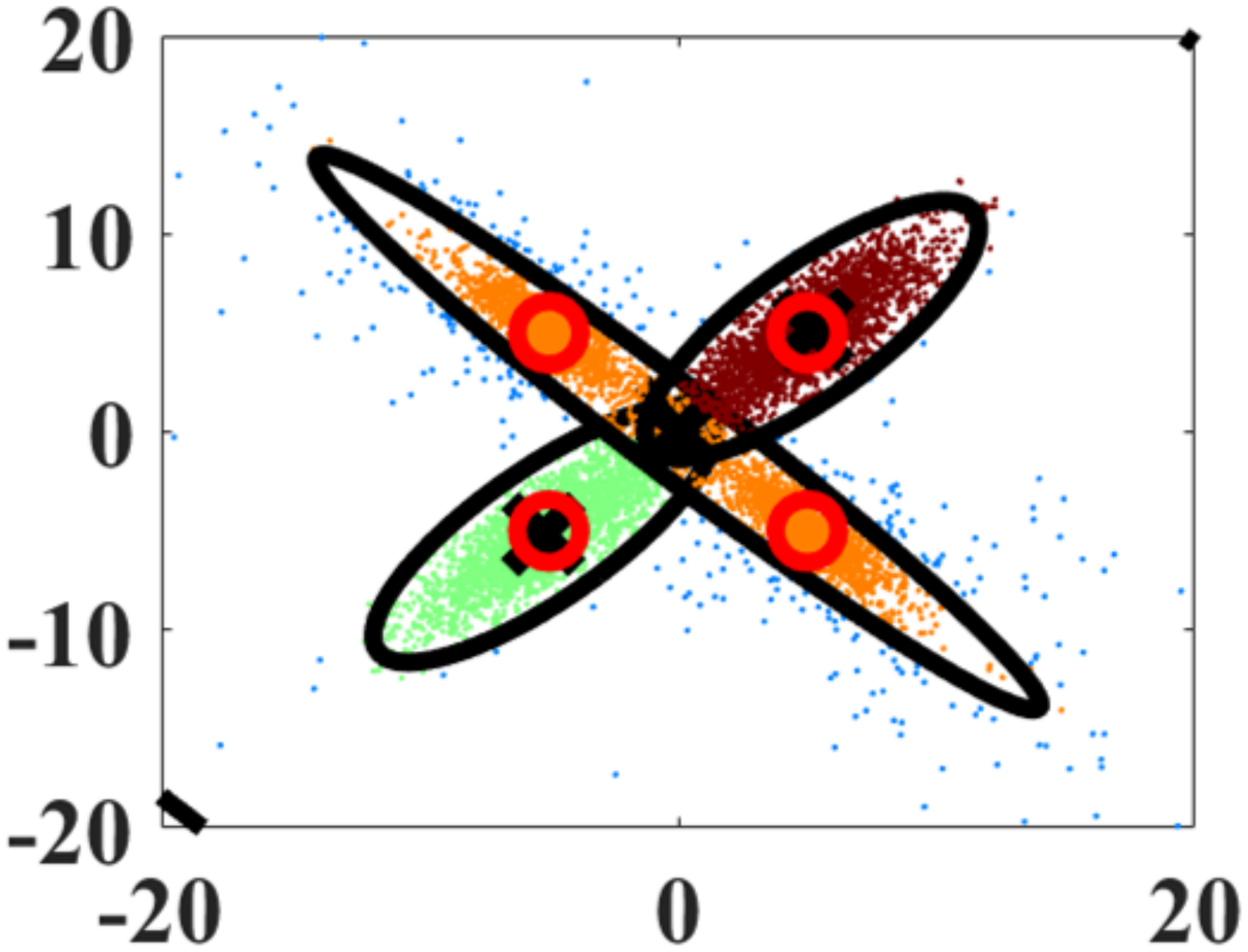}}
		\subfigure[GG1.5]{\includegraphics[width=0.17\textwidth]{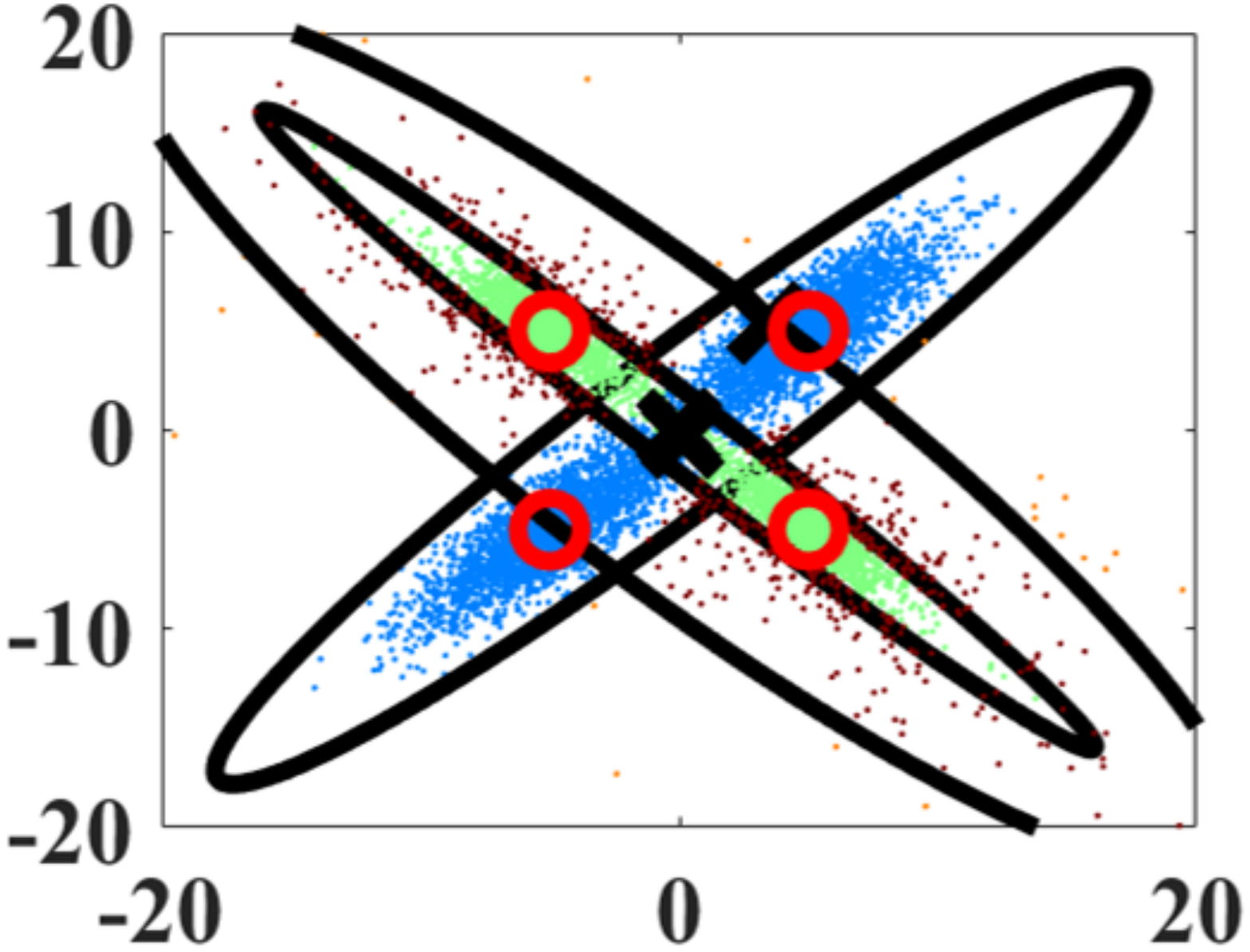}}
	\end{center}
	\vspace{-1em}
	\caption{Optimisation results of the proposed method across the 5 EMMs. The top row shows the results for Case 1 and the bottom row shows the results for Case 2. The red circles denote the ground-truth as shown in Fig. \ref{toypre}-(a), whilst the black crosses and ellipses represent the estimated mean and covariance matrices. The colour of each sample is corresponding to that of Fig. \ref{toypre}-(a), and is classified by selecting the maximum posterior among the clusters.}\label{toyfinalresult}
	\vspace{-1em}
\end{figure*}
\begin{figure*}[!h]
	\begin{center}
		\subfigure[Gaussian]{\includegraphics[width=0.18\textwidth]{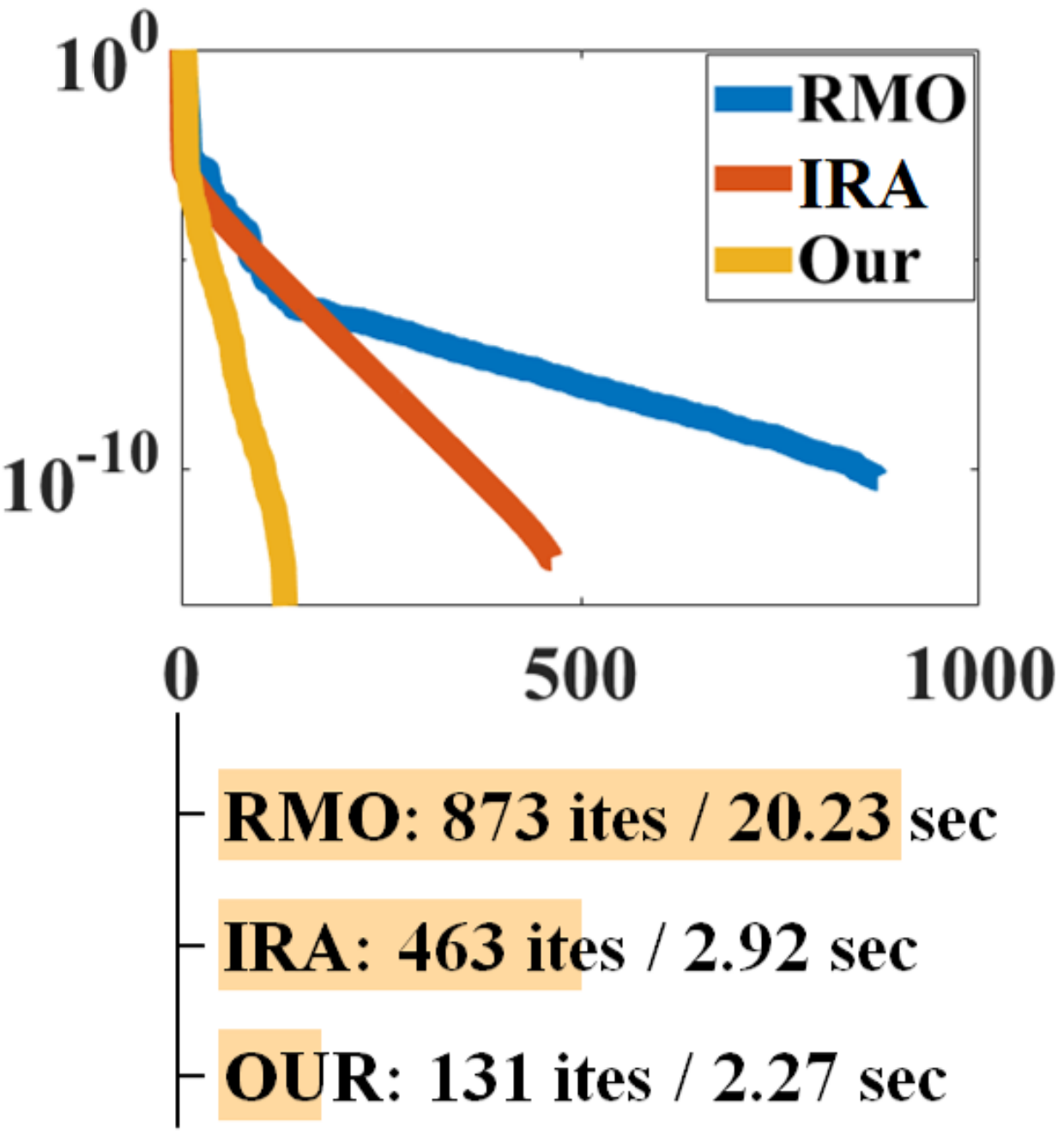}}
		\subfigure[Cauchy]{\includegraphics[width=0.18\textwidth]{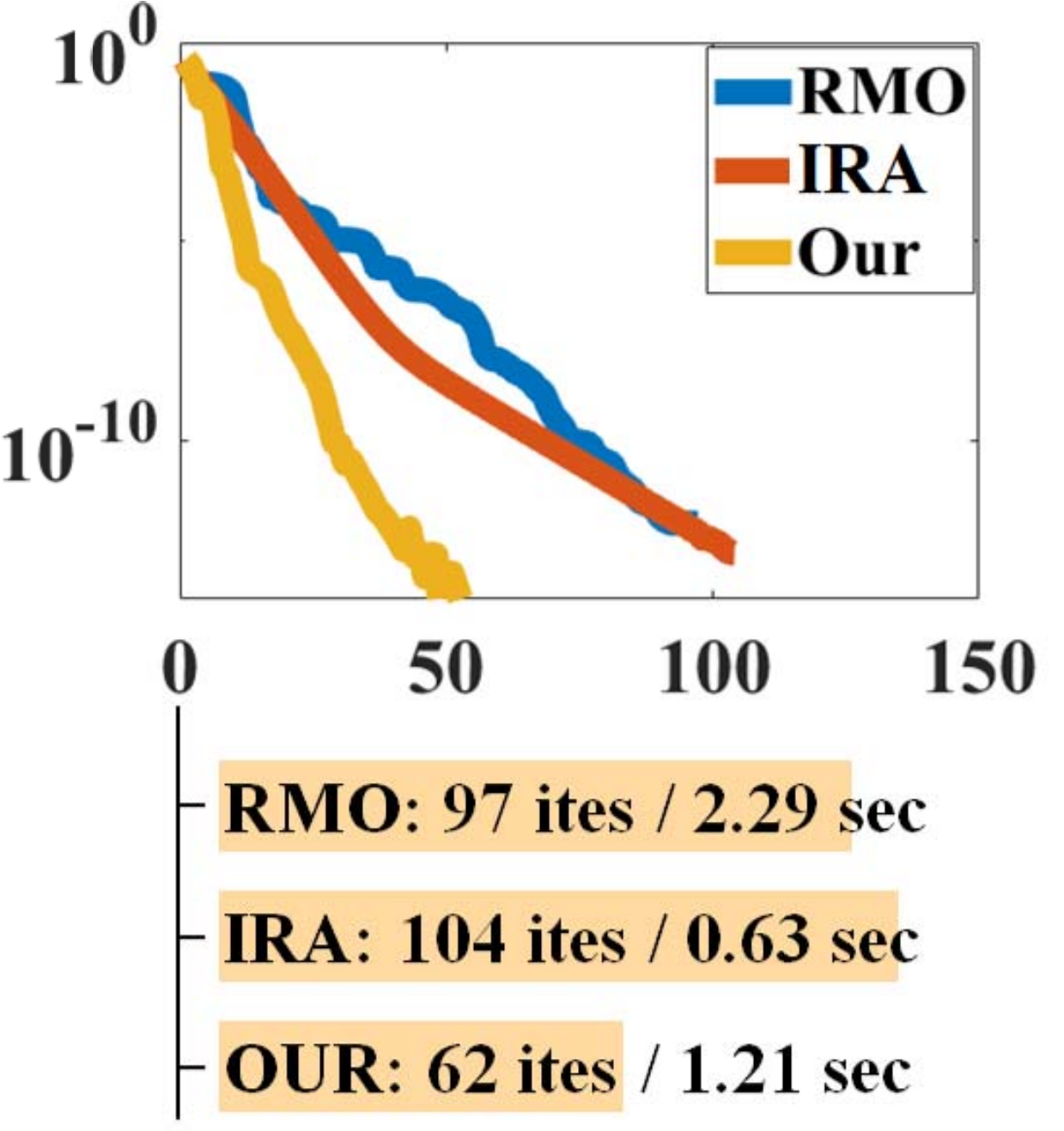}}
		\subfigure[Laplace]{\includegraphics[width=0.18\textwidth]{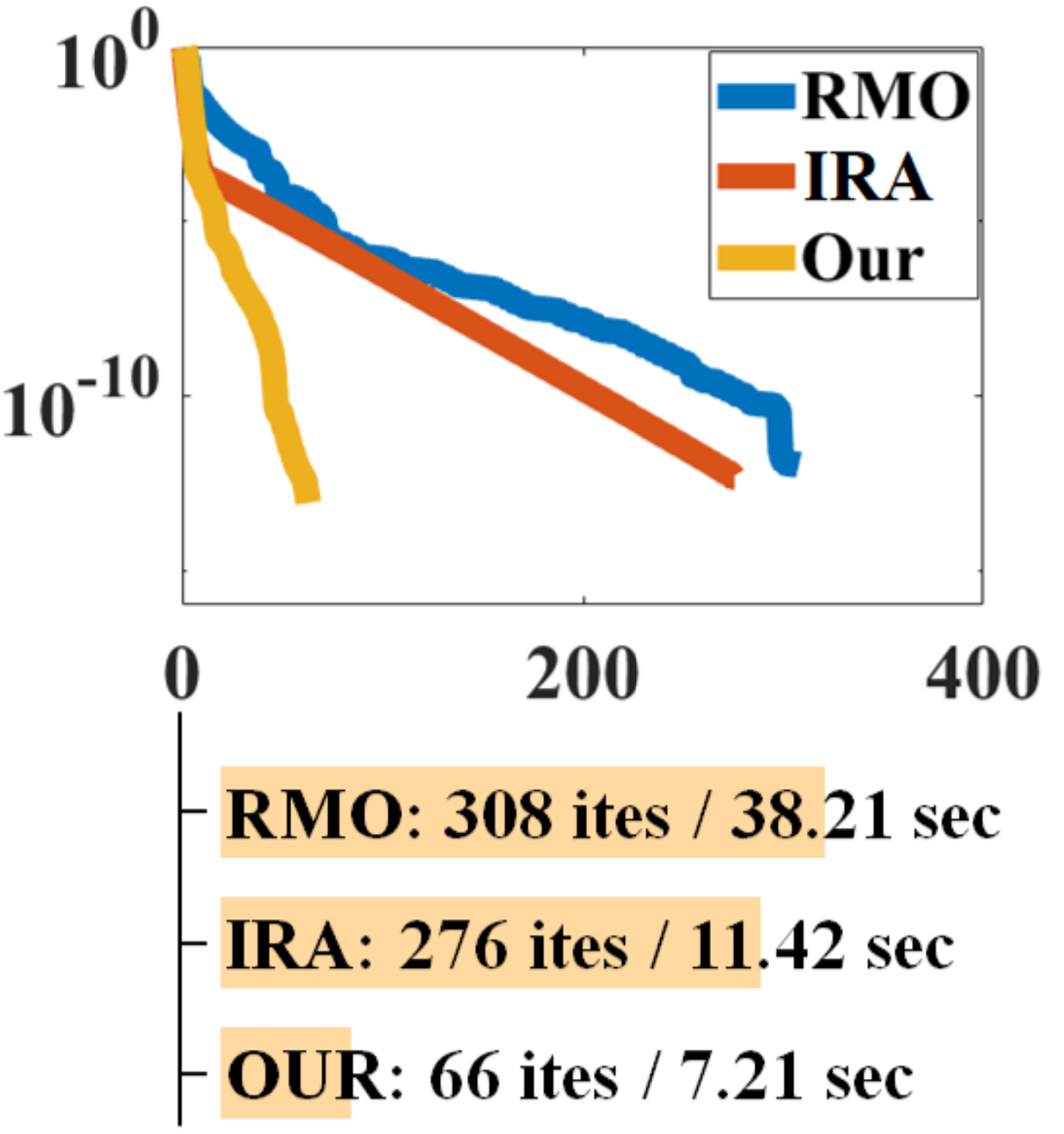}}
		\subfigure[Logistic]{\includegraphics[width=0.18\textwidth]{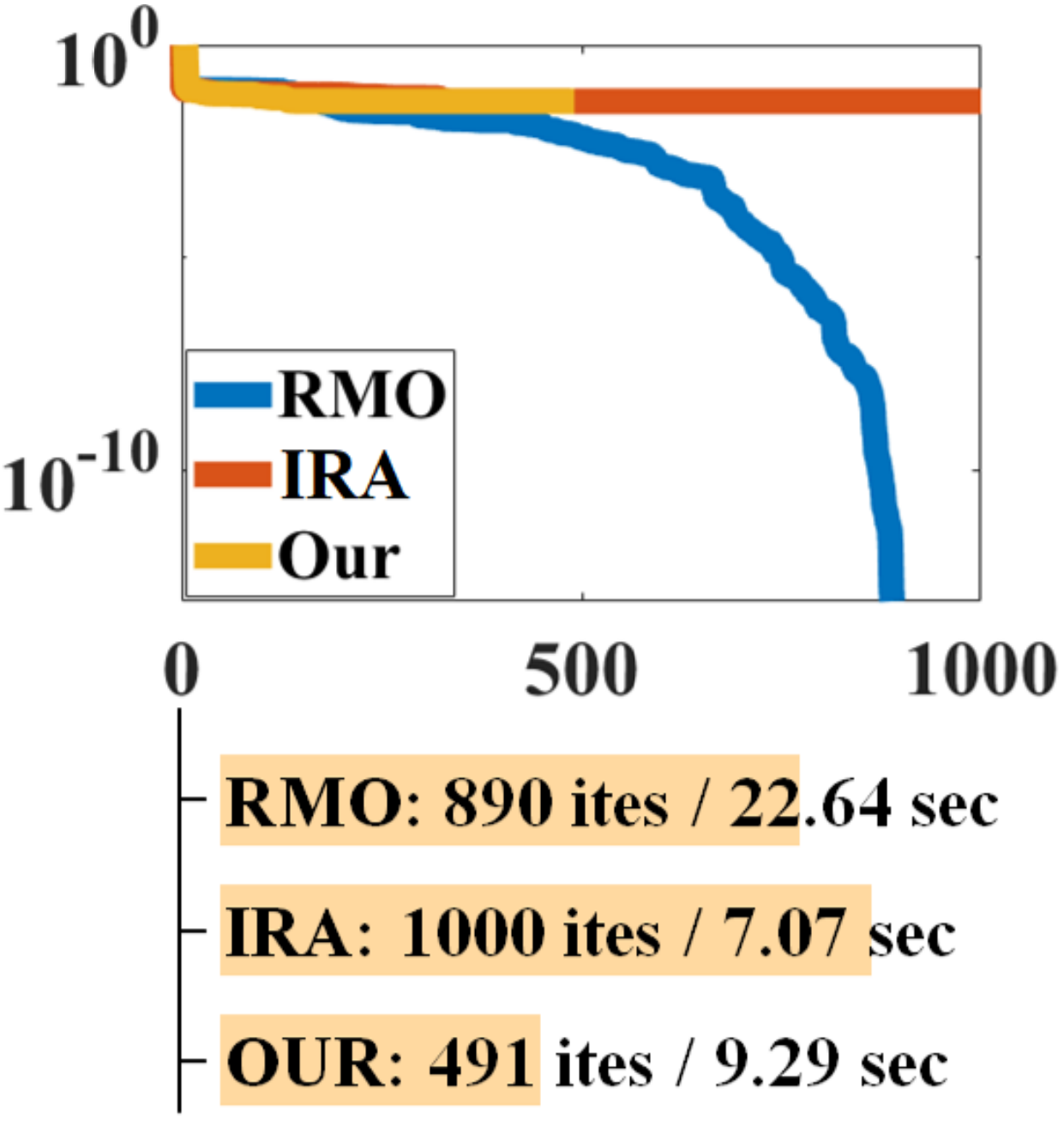}}
		\subfigure[GG1.5]{\includegraphics[width=0.18\textwidth]{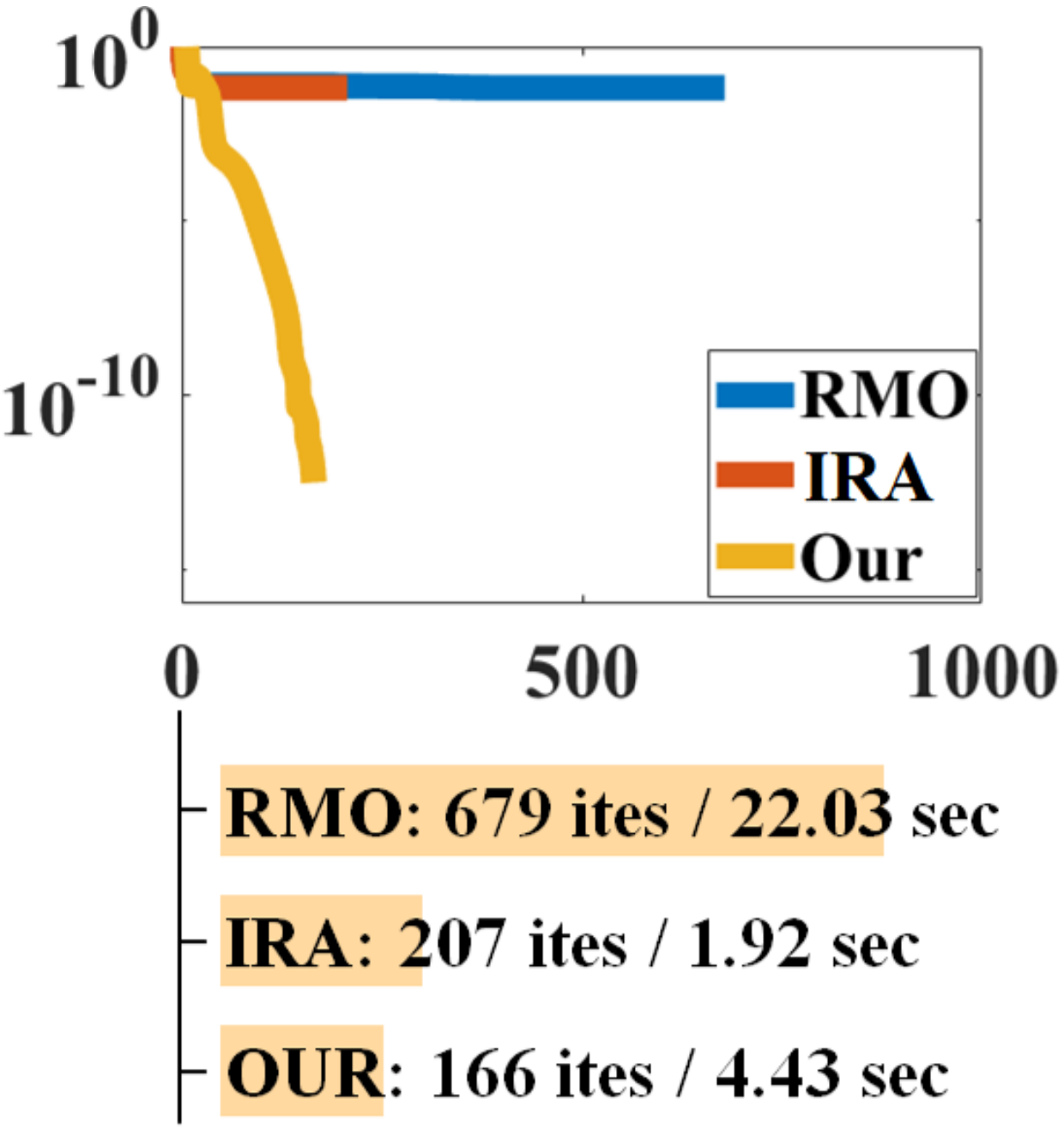}}
	\end{center}
	\vspace{-1em}
	\caption{Cost difference against the number of iterations of the 5 EMMs for Case 1. The top figures show the cost difference against the iterations optimised via Our, IRA and RMO methods. The bottom quantities are the final convergence speed in terms of the number of iterations and execution time (ites/sec).} \label{noisyite}
	\vspace{-1em}
\end{figure*}

\begin{figure*}[!h]
	\begin{center}
		\subfigure[Gaussian]{\includegraphics[width=0.18\textwidth]{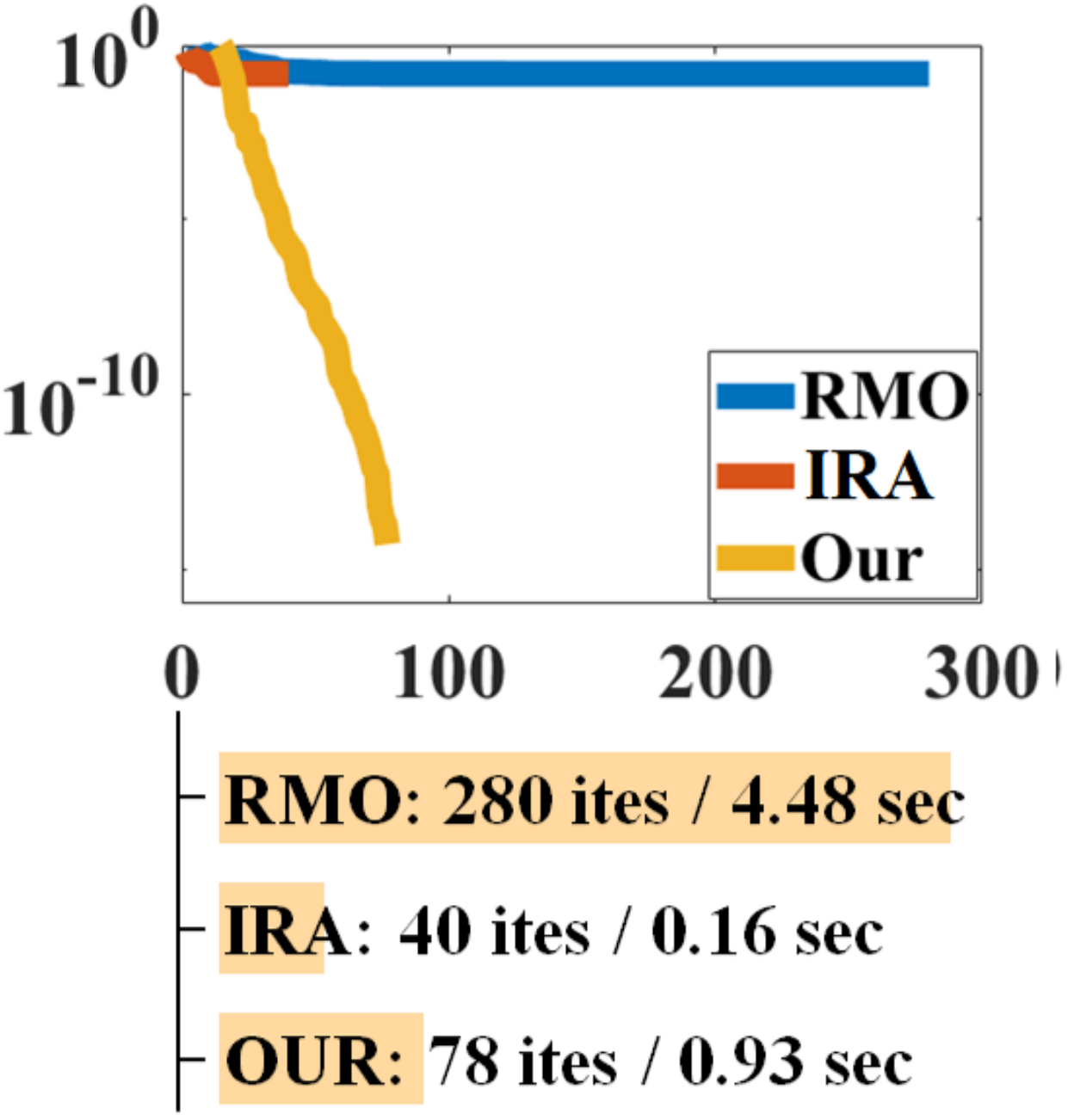}}
		\subfigure[Cauchy]{\includegraphics[width=0.18\textwidth]{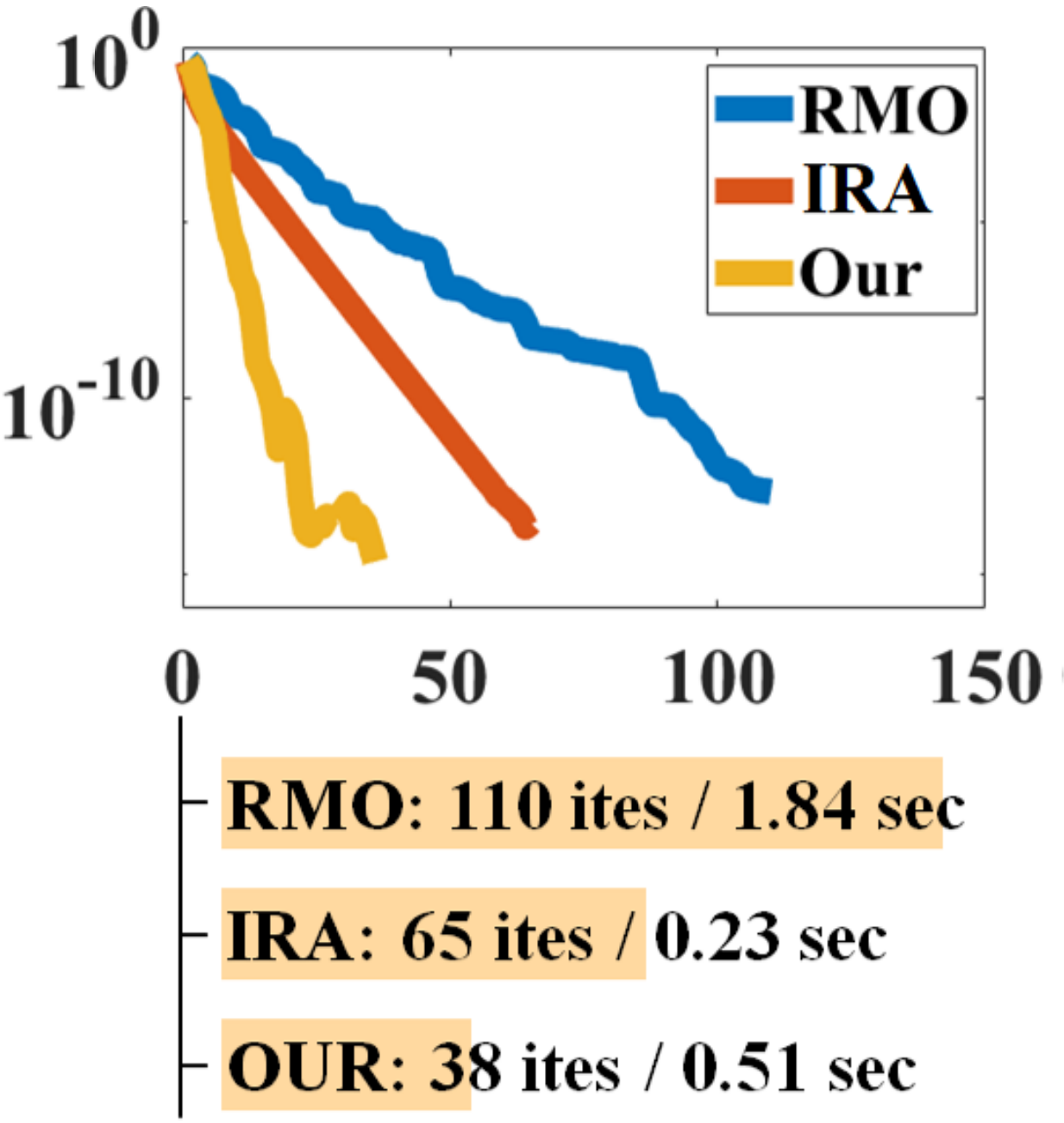}}
		\subfigure[Laplace]{\includegraphics[width=0.18\textwidth]{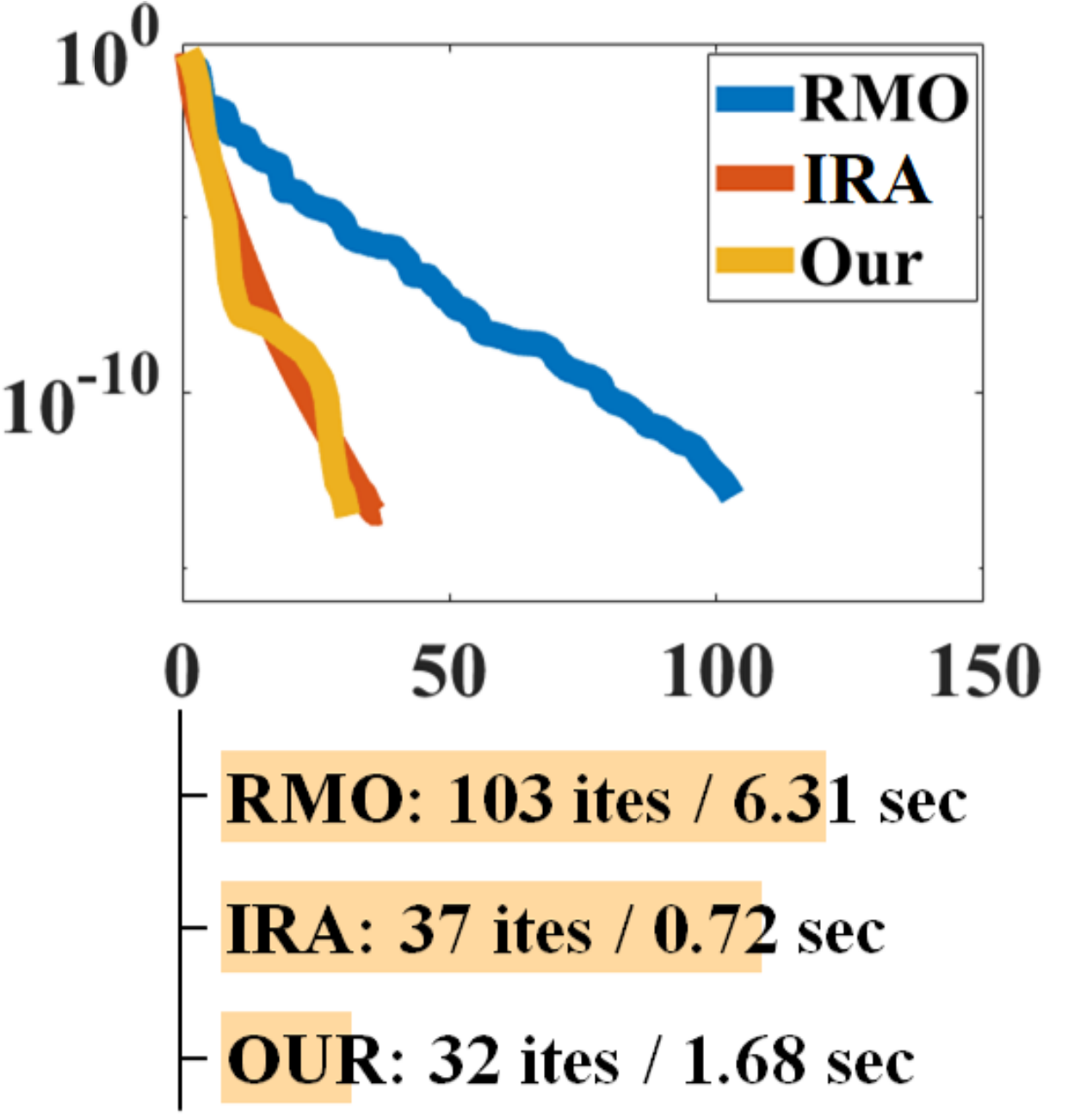}}
		\subfigure[Logistic]{\includegraphics[width=0.18\textwidth]{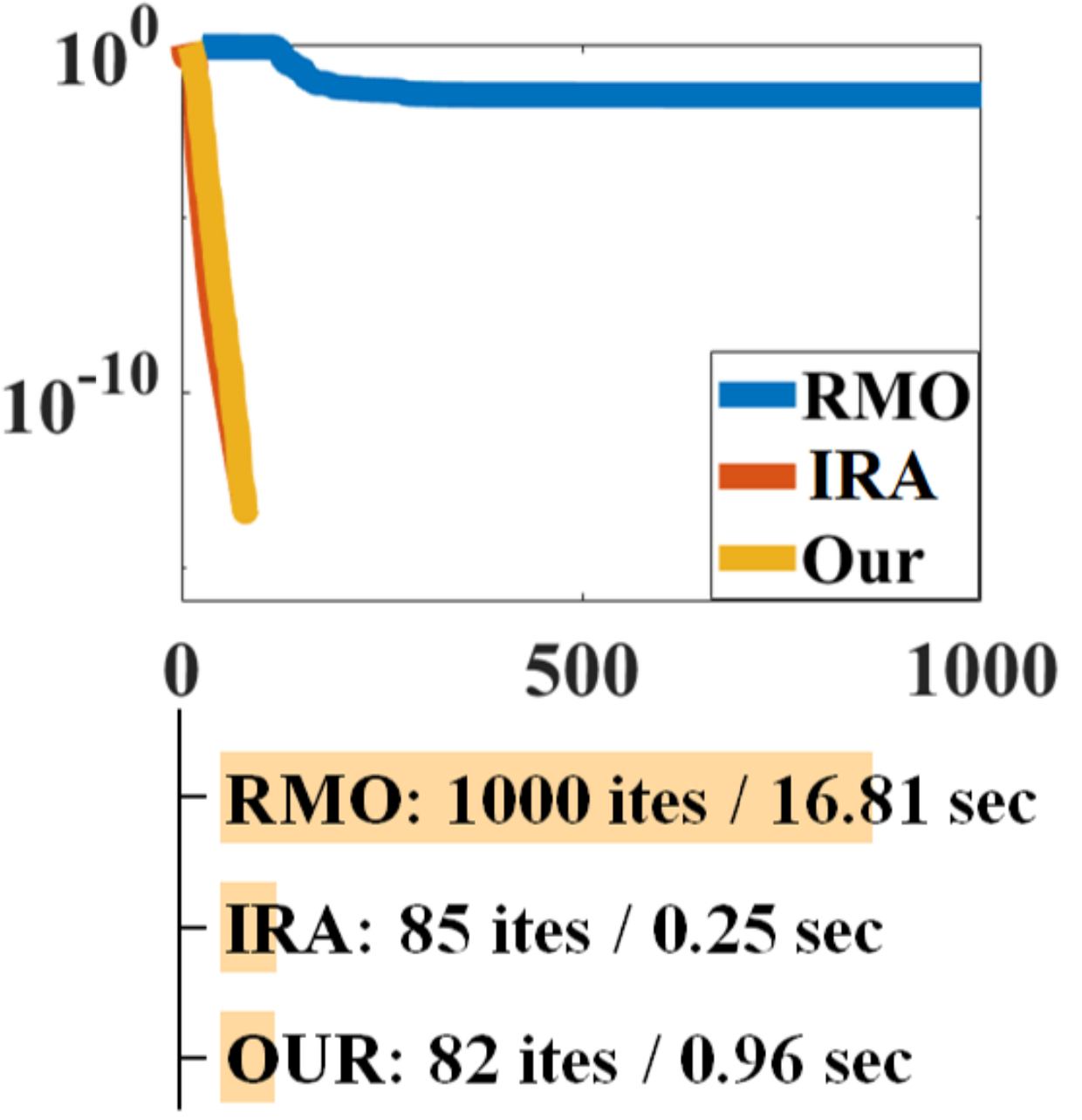}}
		\subfigure[GG1.5]{\includegraphics[width=0.18\textwidth]{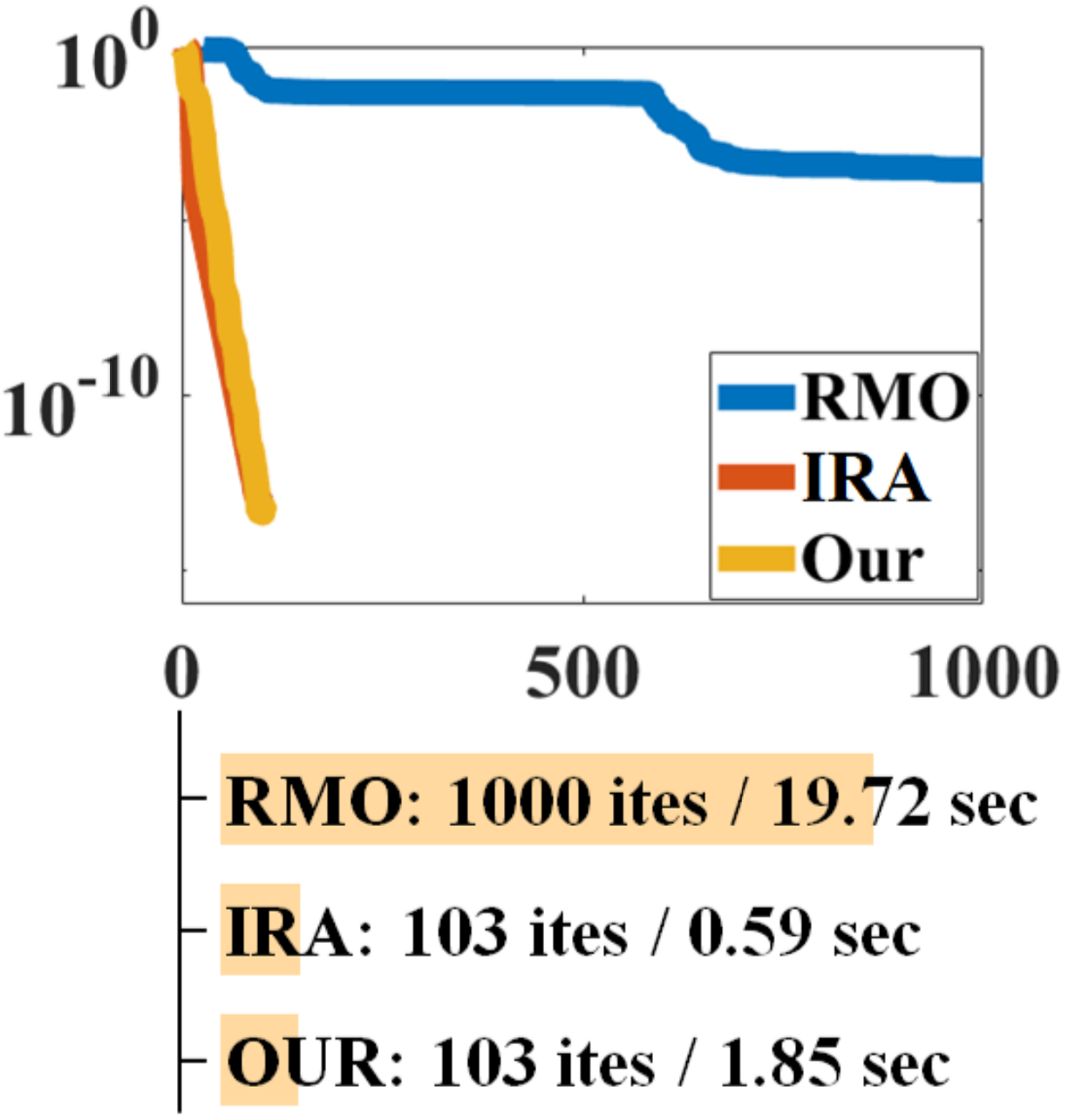}}
	\end{center}
	\vspace{-1em}
	\caption{Cost difference against the number of iterations of the 5 EMMs for Case 2. The top figures show the cost difference against the iterations optimised via Our, IRA and RMO methods. The bottom quantities are the final convergence speed in terms of the number of iterations and execution time (ites/sec).}\label{mixite}
	\vspace{-1em}
\end{figure*}
We further plot the iteration numbers against the average cost difference of Our, IRA and RMO methods when optimising the two cases in Fig. \ref{noisyite} and \ref{mixite}. Note that for the purpose of illustrations \cite{hosseini2015matrix}, the cost difference is defined by the absolute difference between the cost of each iteration and the relatively ``optimal'' cost, which was obtained by choosing the lowest value among the final costs of the three methods. From the two figures, the IRA achieved a monotonic convergence, because it consistently increases the lower bound of the log-likelihood; our method, although occasionally fluctuating at the late stage of convergence, such as for the Cauchy distribution in Fig. \ref{mixite}, consistently achieved the lowest cost among all EMMs and converged with the least number of iterations. A further cautious choice of optimisers as well as line search methods can probably achieve a monotonic convergence. Moreover, although one iteration of our method takes longer time than that of the IRA method due to the line search, the overall computational time of our method is comparable to that of the IRA method. As shown in the next section, our method even performs much faster in terms of computational time than the IRA for higher dimensions and larger numbers of clusters ($M>2$ and $K>4$).

\begin{table*}[h]
	\caption{Overall results averaged across $M$ and $K$ for 9 types of EMMs optimised by Our, IRA and RMO methods.}\label{synoverall}
	\centering
	\scriptsize{
		\begin{tabular}{cc|ccccccccc}\hline \hline
			\multicolumn{2}{c}{$c=10$, $e=10$}  & {Gaussian}   & Cauchy  & Student-$t$ ($v=10$)   & GG1.5      & Logistic   & Weib0.9    & Weib1.1    & Gamma1.1 & Mix  \\\hline
			\multirow{2}{*}{Our} & Ite. / T. (s) & \textbf{108} / \textbf{40.3} & \textbf{135} / \textbf{34.5}    & \textbf{136} / \textbf{32.7} & \textbf{219} / \textbf{160}  & \textbf{113} / \textbf{43.2} & \textbf{158} / \textbf{37.8} & \textbf{139} / \textbf{56.5} & \textbf{138} / \textbf{38.3} & \textbf{132} /\textbf{ 51.8} \\
			& Cost      & 64.5       & \textbf{65.6}    &   \textbf{64.8}   & \textbf{69.8}       & 47.8       & \textbf{64.4}       & \textbf{64.4}       & \textbf{64.4}       & \textbf{67.5}       \\
			\multirow{2}{*}{IRA} & Ite. / T. (s) & 367 / 90.0 & ----   &  ----  & ----          & 428 / 249  & 310 / 46.8 & ----       & ----       & ----       \\
			& Cost      & 66.2       & ----   &  ----   & ----       & 49.0       & 64.5       & ----       & ----       & ----       \\
			\multirow{2}{*}{RMO} & Ite. / T. (s) & 658 / 1738 & 848 / 1590   & 776 / 1681 & 796 / 1435 & 744 / 1651 & 825 / 1806 & 773 / 1625 & 760 / 1550 & 711 / 1547 \\
			& Cost      & \textbf{64.3}       & 65.8  & 64.7 & 71.6       & \textbf{47.5}       & 64.4       & 64.5       & 64.5       & 66.8       \\\hline
			\multicolumn{2}{c}{$c=0.1$, $e=1$} & & & & & & & & &  \\\hline
			\multirow{2}{*}{Our} & Ite. / T. (s) &       \textbf{693} /  \textbf{201}              &   \textbf{424} / \textbf{99.9}    & \textbf{621} / \textbf{106}  & \textbf{581} / \textbf{100}      &     \textbf{734} / \textbf{242}     & \textbf{686} / 119  & \textbf{655} / \textbf{111}  & \textbf{664} / \textbf{94.6} & \textbf{470} / \textbf{65.7} \\
			& Cost      &     \textbf{40.0}       &        \textbf{41.0}             & \textbf{40.2} &   \textbf{39.8}      &      \textbf{23.5 }     & \textbf{39.1}       & \textbf{40.1}       & \textbf{40.1}       & \textbf{39.2}       \\
			\multirow{2}{*}{IRA} & Ite. / T. (s) &    ----                    &   ----      & ---- &   ----      &       ----     & 711 / \textbf{101}  & ----       & ----       & ----       \\
			& Cost      &    ----        &    ----   &  ----   &       ----              &      ----      & 40.6       & ----       & ----       & ----       \\
			\multirow{2}{*}{RMO} & Ite. / T. (s) &     1000 / 1707                 &   956 / 1661    &  951 / 1789 & 898 / 1206      &     963 / 1508     & 969 / 1632 & 975 / 1744 & 958 / 1561 & 917 / 1505 \\
			& Cost      &     40.4       &     41.5       & 40.8  &   42.1              &    23.8        & 40.3       & 40.5       & 40.4       & 40.4     \\\hline \hline 
			\multicolumn{2}{c|}{Note:} &\multicolumn{9}{l}{T. (s): Time (seconds); Ite.: Iteration numbers; ----: Singularity or infinity in either $M=8$, $M=16$ or $M=64$.}\\\hline
	\end{tabular}}
	\vspace{-1em}
\end{table*}

\begin{table*}[!h]
	\caption{Detailed comparisons among Our, IRA and RMO on optimising 3 EMMs in the case of $c = 10$ and $e = 10$}\label{detail}
	\scriptsize{\resizebox{\textwidth}{!}{
			\begin{tabular}{cc|ccc|ccc|ccc}\hline \hline
				\multirow{2}{*}{(M, K)}  & \multirow{2}{*}{} & \multicolumn{3}{c|}{Gaussian}                         & \multicolumn{3}{c|}{Cauchy}                          & \multicolumn{3}{c}{Logistic}                        \\
				&                   & Our             & IRA             & RMO              & Our             & IRA             & RMO             & Our             & IRA             & RMO             \\\hline
				\multirow{3}{*}{(8, 8)}  & T. (s)              & \textbf{3.70 $\pm$ 5.44} & 5.48 $\pm$ 4.13 & 4.23 $\pm$ 10.76 & \textbf{2.23 $\pm$ 1.67} & 3.04 $\pm$ 2.45 & 20.4 $\pm$ 10.6 & \textbf{4.28 $\pm$ 4.60} & 4.71 $\pm$ 4.08 & 18.9 $\pm$ 17.3 \\
				& Ite.              & 165 $\pm$ 240   & 538 $\pm$ 403   & \textbf{122 $\pm$ 308}    & \textbf{107 $\pm$ 81.3}  & 340 $\pm$ 275   & 640 $\pm$ 334   & \textbf{177 $\pm$ 191}   & 463 $\pm$ 398   & 537 $\pm$ 489   \\
				& Co./Fa.           & 19.4 / 0$\%$    & 19.6 / 0$\%$    & \textbf{19.3} / \textbf{0$\%$}     & \textbf{20.3} / \textbf{0$\%$}    & 20.4 / 0$\%$    & 20.3 / 0$\%$    & \textbf{14.9} / \textbf{0$\%$}    & 15.0 / 0$\%$    & 14.9 / 0$\%$    \\	\hline	
				\multirow{3}{*}{(16, 16)} & T.  (s)            & \textbf{15.8 $\pm$ 13.4} & 28.8 $\pm$ 15.2 & 221 $\pm$ 78.3   & \textbf{38.1 $\pm$ 34.5} & 40.9 $\pm$ 18.2 & 260 $\pm$ 1.84  & \textbf{15.7 $\pm$ 7.64} & 23.6 $\pm$ 14.6 & 193 $\pm$ 115   \\		
				& Ite.              & \textbf{115 $\pm$ 94.0}  & 400 $\pm$ 207   & 853 $\pm$ 304    & \textbf{272 $\pm$ 243}   & 570 $\pm$ 254   & 1000 $\pm$ 0.00 & \textbf{115 $\pm$ 54.0}  & 325 $\pm$ 201   & 734 $\pm$ 433   \\	
				& Co./Fa.           & \textbf{37.7} / \textbf{0$\%$}    & 38.0 / 0$\%$    & 37.8 / 0$\%$     & \textbf{38.7} / \textbf{0$\%$}    & 38.8 / 0$\%$    & 39.0 / 0$\%$    & \textbf{28.6} / \textbf{0$\%$}    & 28.8 / 0$\%$    & 28.6 / 0$\%$    \\	\hline	
				\multirow{3}{*}{(64, 64)} & T. (s)             & \textbf{101 $\pm$ 18.0}  & 236 $\pm$ 0.00  & 4988 $\pm$ 152   & \textbf{63.2 $\pm$ 11.7} & ----            & 4491 $\pm$ 1549 & \textbf{110 $\pm$ 19.3}  & 720 $\pm$ 343   & 4741 $\pm$ 475  \\
				& Ite.              & \textbf{45.8 $\pm$ 5.67} & 163 $\pm$ 0.00  & 1000 $\pm$ 0.00  & \textbf{26.9 $\pm$ 3.70} & ----            & 902 $\pm$ 309   & \textbf{49.3 $\pm$ 5.98} & 497 $\pm$ 235   & 960 $\pm$ 121   \\
				& Co./Fa.           & \textbf{136} / \textbf{0$\%$ }    & 141 / 90$\%$    & 136 / 10$\%$     & \textbf{138} / \textbf{0$\%$}     & -- / 100$\%$    & 138 / 0$\%$     & 99.7 / \textbf{0$\%$}    & 103 / 70$\%$    & \textbf{99.1} / 10$\%$  \\ \hline \hline
				\multicolumn{2}{c|}{Note:} & \multicolumn{9}{l}{T. (s): Time (seconds); Ite.: Iteration numbers; Co.: Final cost; Fa.: Optimisation fail ratio; ----: Singularity or infinity}\\\hline
	\end{tabular}}}
	\vspace{-1em}
\end{table*}
\subsection{Evaluations over the synthetic datasets}\label{syntheticeva}
\vspace{-.5em}
We next systematically evaluated our method based on the synthetic dataset described in Section \ref{parameterssetiing}. For each dataset, we had $8\times 3=24$ test cases, i.e., 9 types of EMMs for the 3 methods. The 9 types of EMMs include 7 types of elliptical distributions and the other one (denoted as Mix) is composed by half the number of Cauchy distributions and the other half of Gaussian distributions. 
Table \ref{synoverall} shows the overall result averaged across dimensions $M$ and $K$ for the 9 EMMs. As can be seen from Table \ref{synoverall}, our method exhibits the fastest convergence speed in terms of both the number of iterations and computation time, and it also obtains the minimum cost. It can also be found that datasets with more overlaps (i.e., $\{c=0.1,e=1\}$) take a longer time to optimise, whereby iterations and computational time increase for all the 3 methods. On the other hand, by comparing the results of our method and those of the RMO method, we can clearly see a significant improvement in both convergence speed and final minimum, which verifies the effectiveness of our reformulation technique.

We provide further details of the comparisons of different $M$ and $K$ in Table \ref{detail}, where 3 EMMs and $\{c=10,e=10\}$ were reported due to the space limitation and the fact that similar results can be found for other EMMs and settings of Table \ref{synoverall}. We can see from this table that the superior performance of our method is consistent over different dimensions $M$ and cluster numbers $K$. Table \ref{detail} shows that our method requires the minimum number of iterations as well as least computational time. More importantly, the standard deviations for the iterations and computational time of our method are almost the lowest, which means that our method is able to stably optimise the EMMs. With an increase in $M$ and $K$, our method consistently achieves the best performance of the average costs with $0\%$ fail ratio. In contrast, the IRA can become extremely unstable. One reason is due to the fact that, as mentioned in Section \ref{relatedworks}, the IRA cannot converge for the non-geodesic convex distributions such as the \textit{Weib1.1} and \textit{Gamma1.1} in Table \ref{synoverall} \cite{wiesel2012geodesic}. Another perspective is that it even failed on geodesic convex distributions in Table \ref{detail} (e.g., $90\%$ fail ratio for the Gaussian and $100\%$ fail ratio for the Cauchy when $M\!=\!64,K\!=\!64$). This may be due to the separate updating scheme on $\bm{\mu}_k$ and $\mathbf{\Sigma}_k$, which has been mentioned in Section \ref{FEMM}. Although we see an enhanced stability when using manifold optimisation, this separate updating scheme, we believe, is also the reason that the RMO requires extremely large computational complexity to converge ($>900$ iterations and $>4000$ seconds when $M\!=\!64,K\!=\!64$) and in several cases it even failed to converge altogether. Our method, on the one hand, re-parametrises the parameters to perform a simultaneous update on $\bm{\mu}_k$ and $\mathbf{\Sigma}_k$ via $\mathbf{\tilde\Sigma}_k$; the re-parametrised EMMs also have a fixed zero mean that prevents potential clusters to move to the data boundary. These two aspects enable our method to consistently and stably optimise EMMs with the fastest convergence speed and lowest cost.

\begin{figure*}[!htb]
	\begin{center}
		\subfigure{\includegraphics[width=0.16\textwidth]{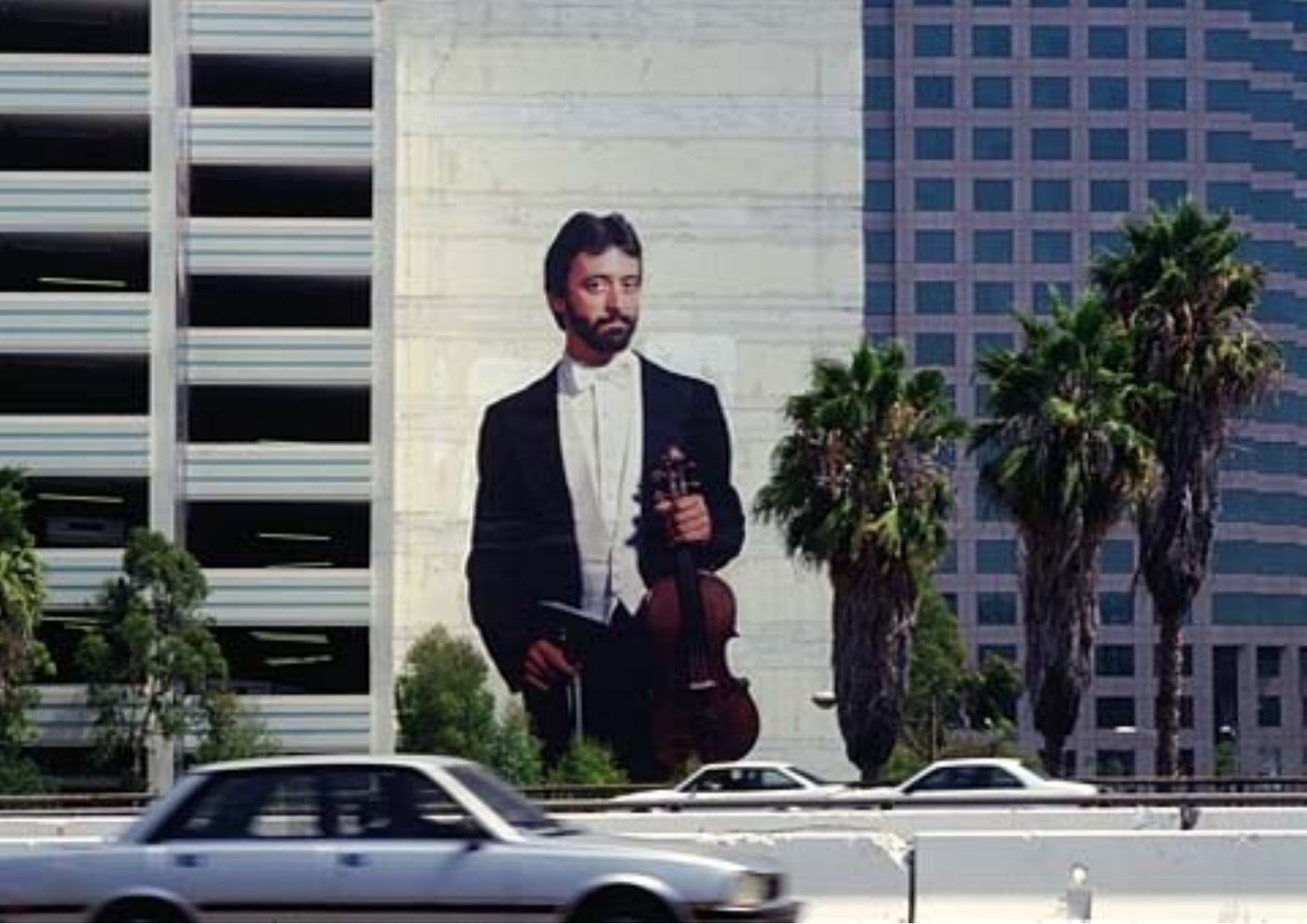}}
		\subfigure{\includegraphics[width=0.16\textwidth]{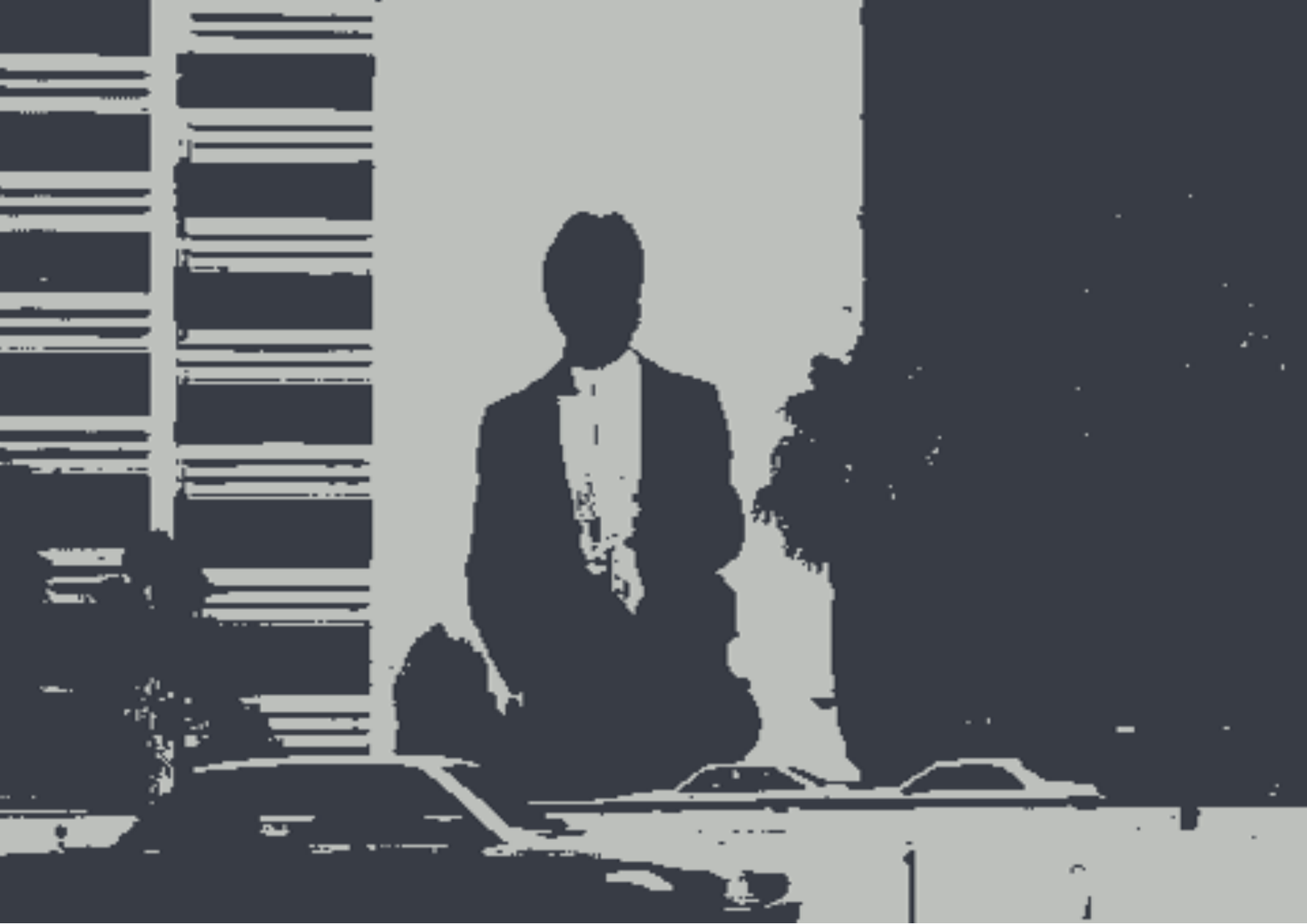}}
		\subfigure{\includegraphics[width=0.16\textwidth]{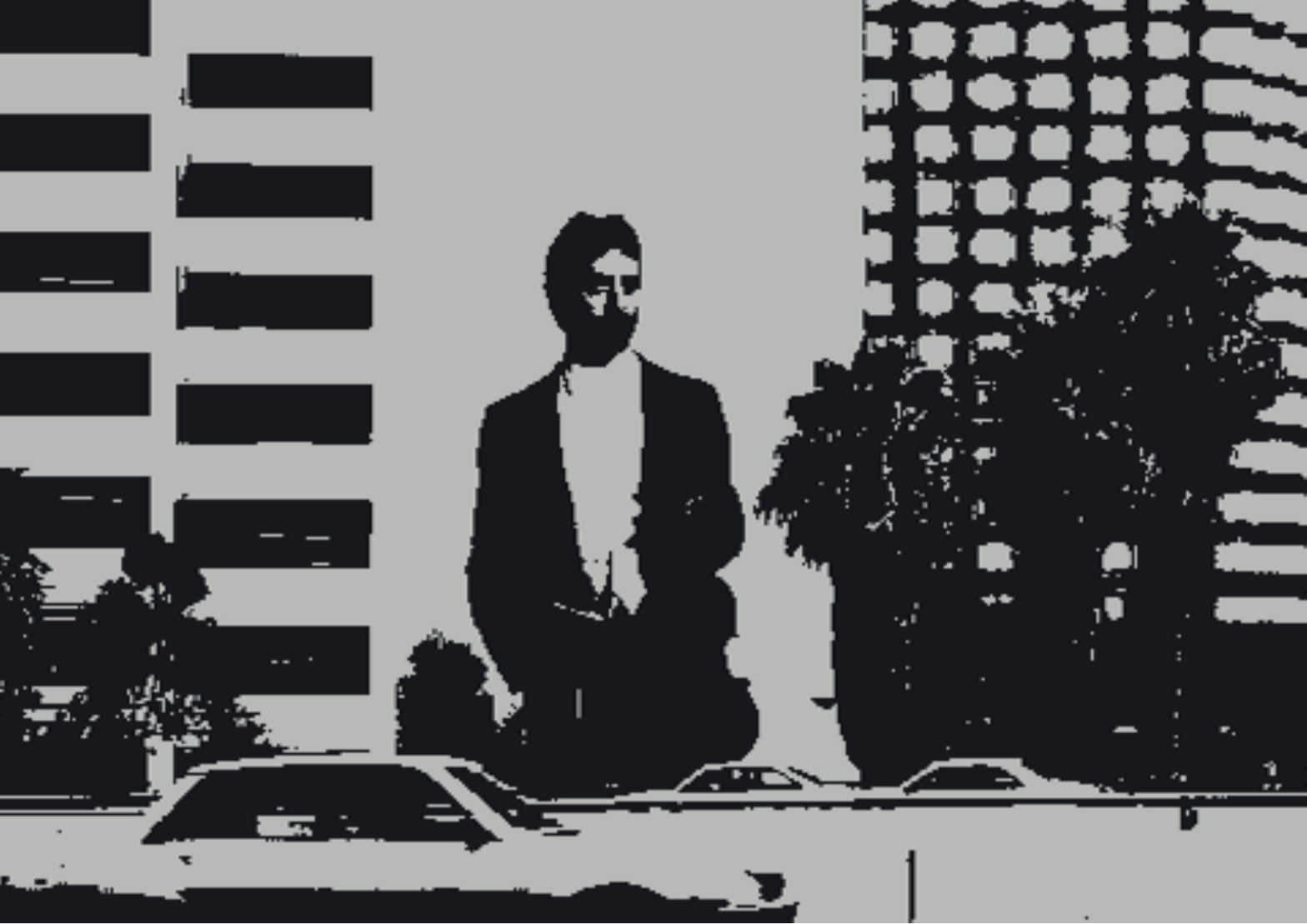}}
		\subfigure{\includegraphics[width=0.16\textwidth]{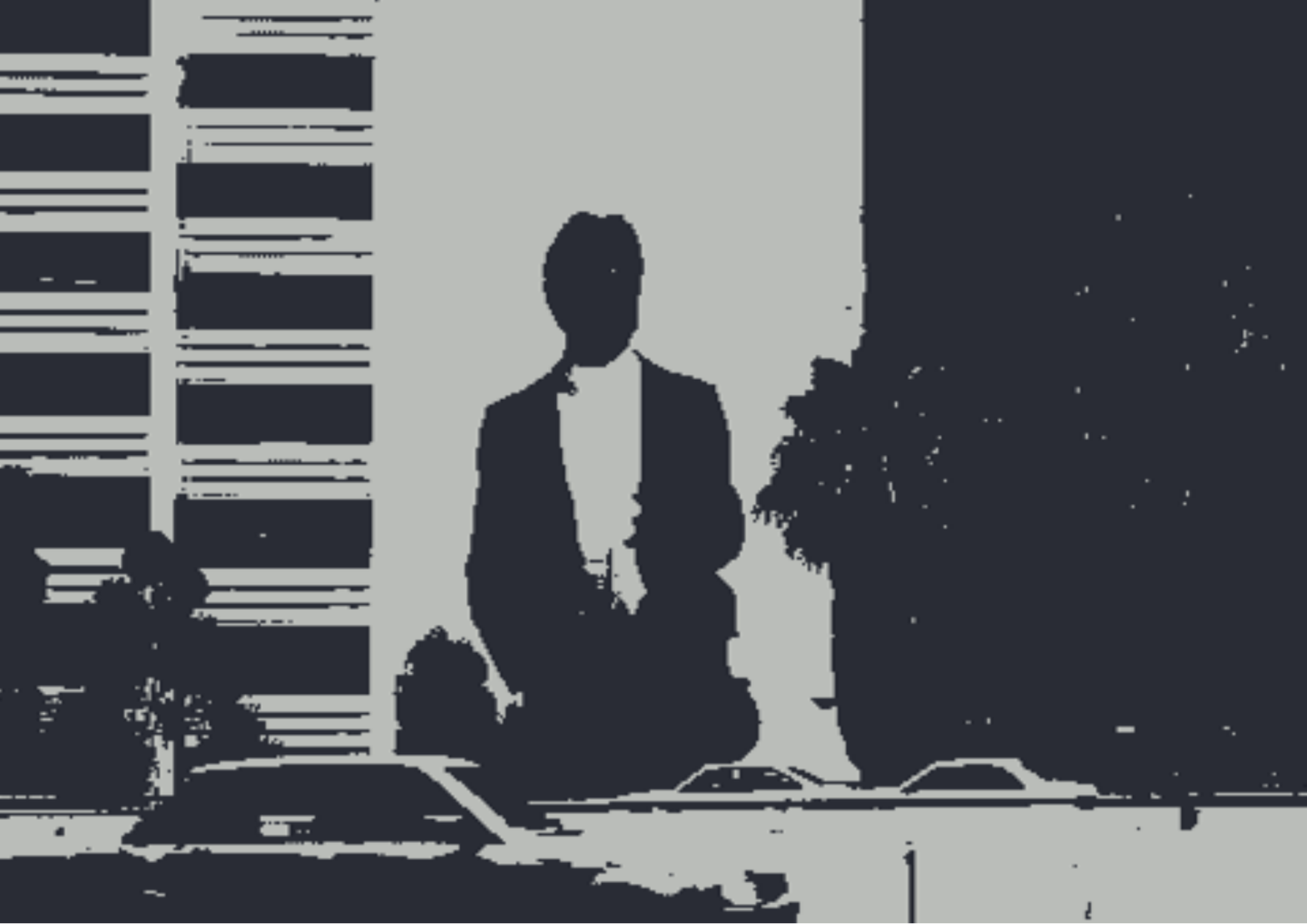}}
		\subfigure{\includegraphics[width=0.16\textwidth]{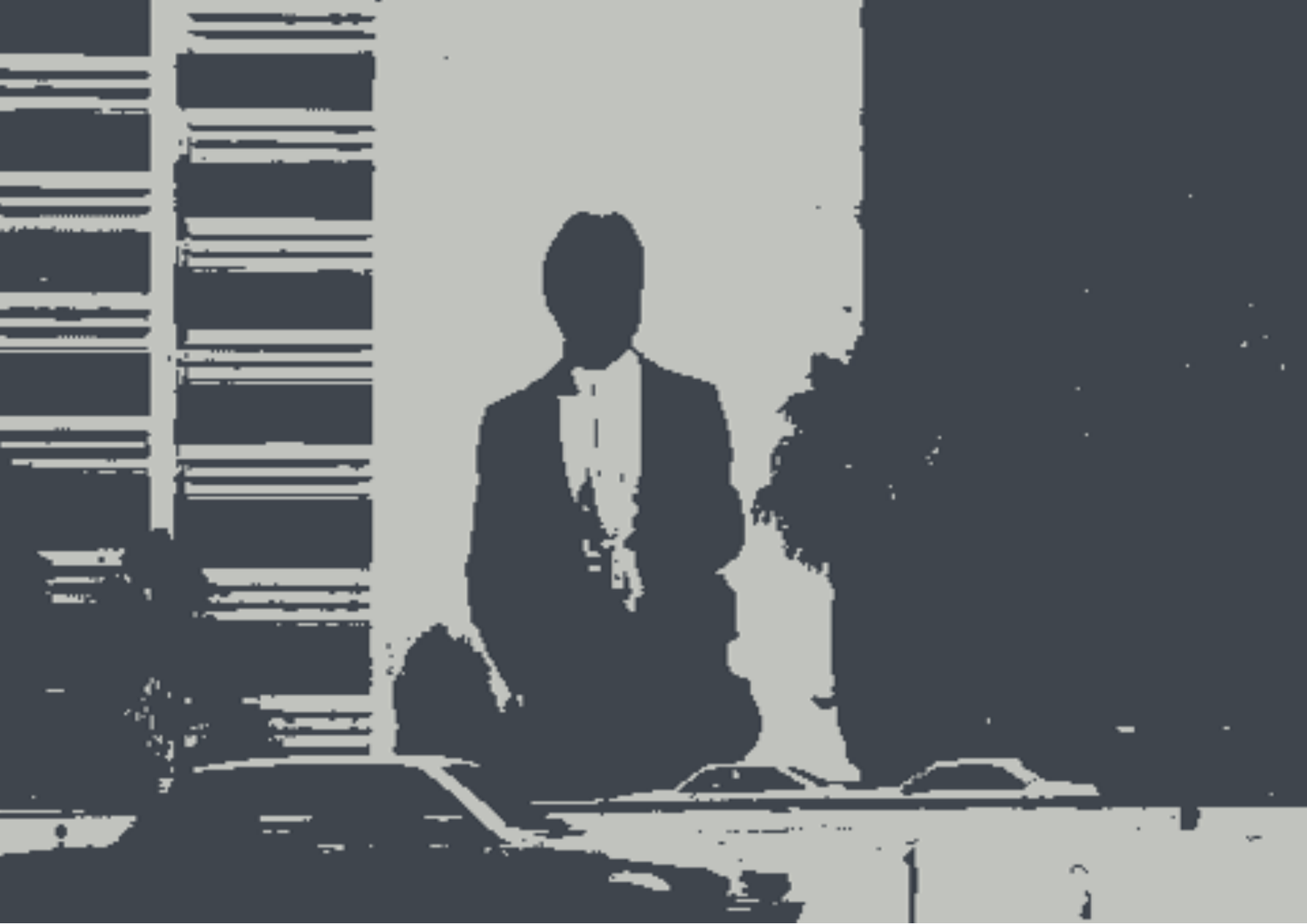}}
		\subfigure{\includegraphics[width=0.16\textwidth]{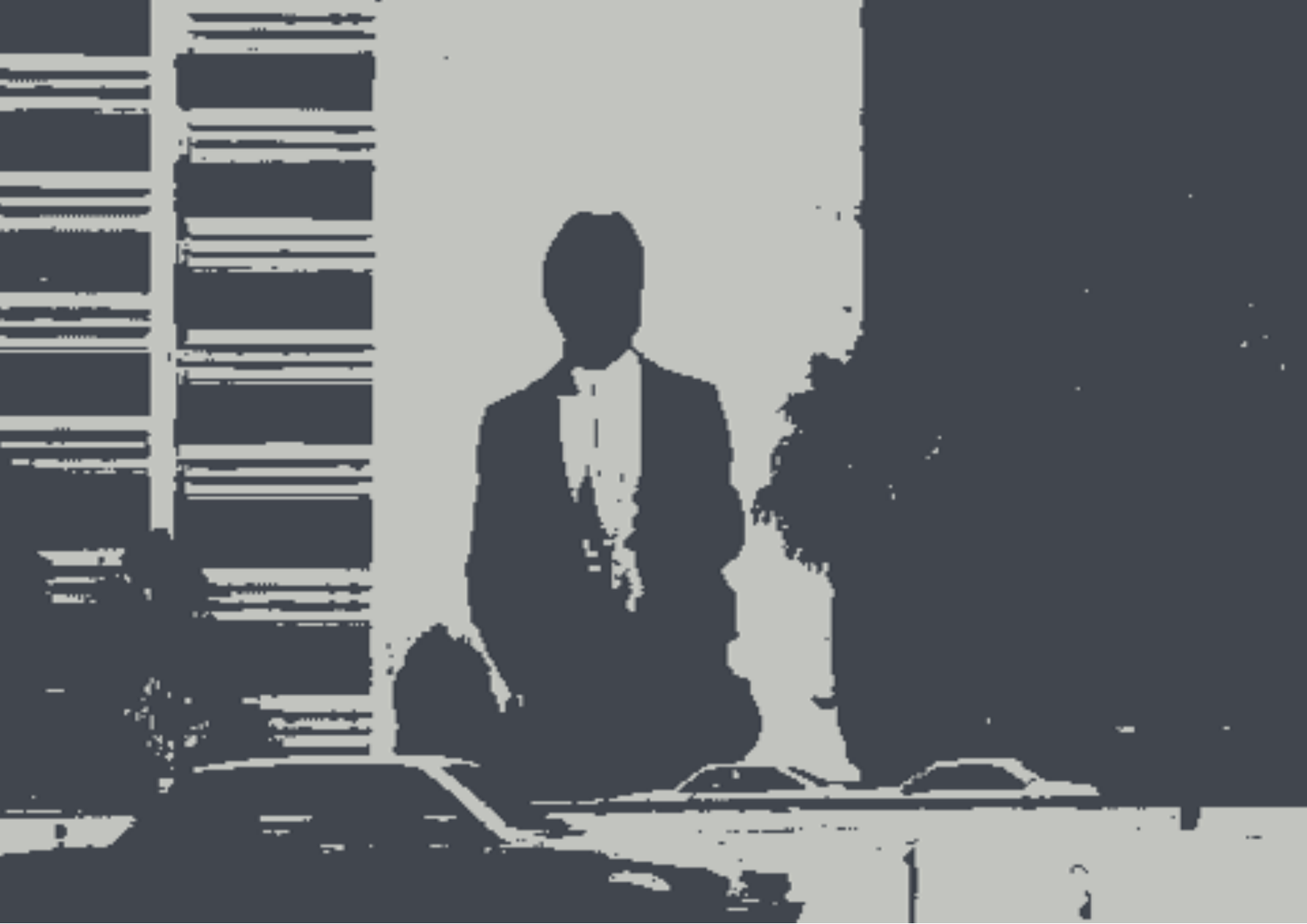}}
		\subfigure{\includegraphics[width=0.16\textwidth]{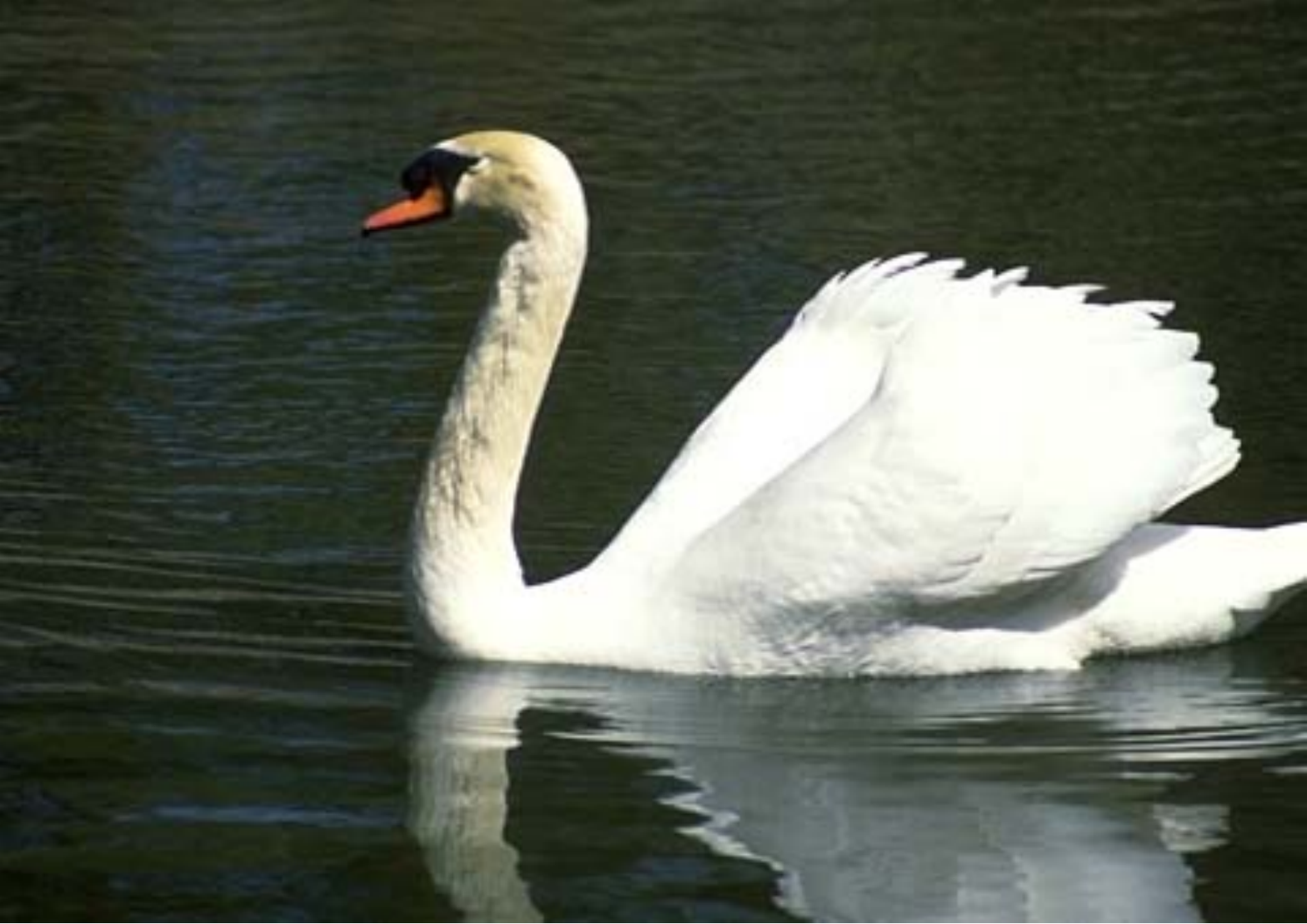}}
		\subfigure{\includegraphics[width=0.16\textwidth]{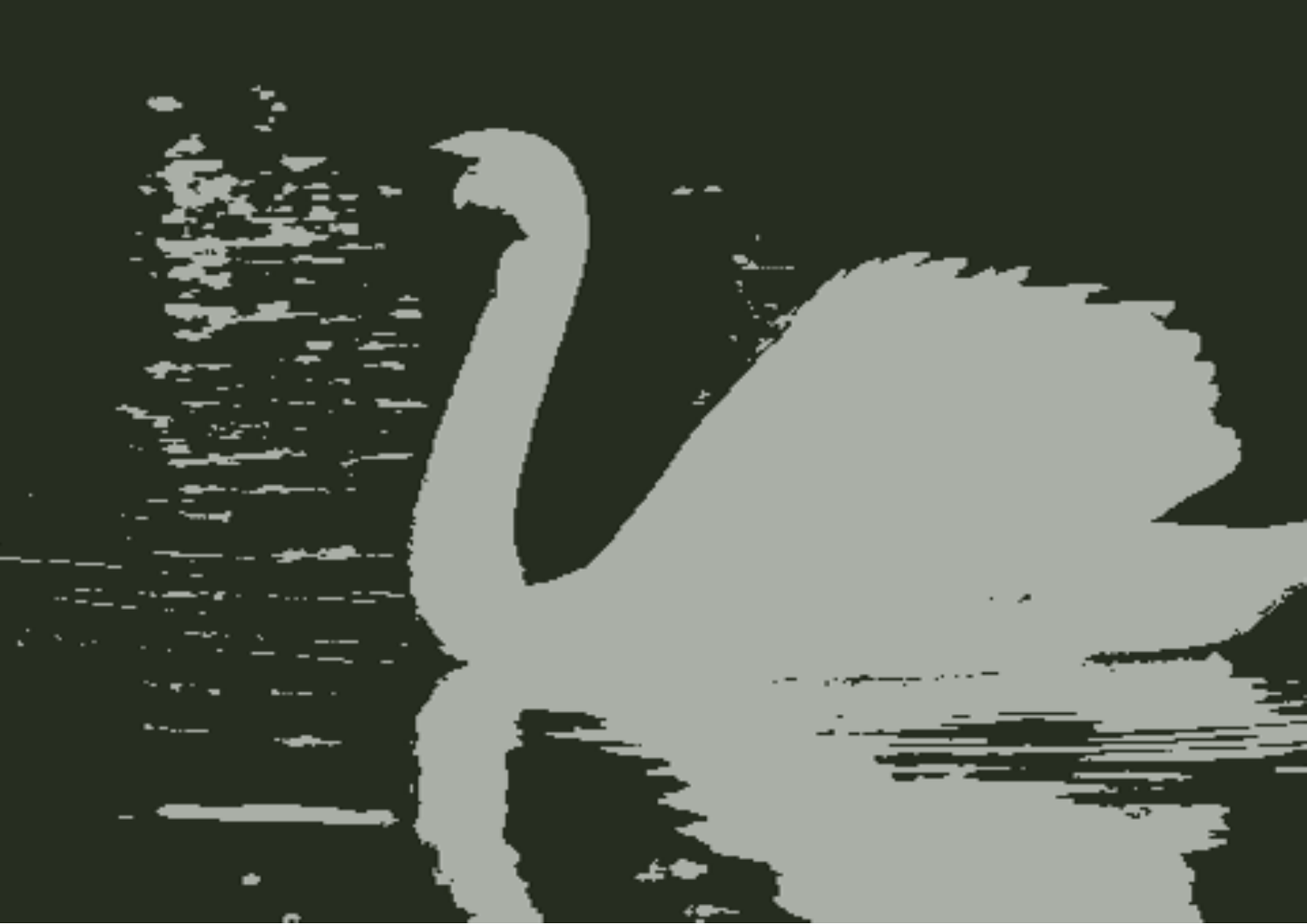}}
		\subfigure{\includegraphics[width=0.16\textwidth]{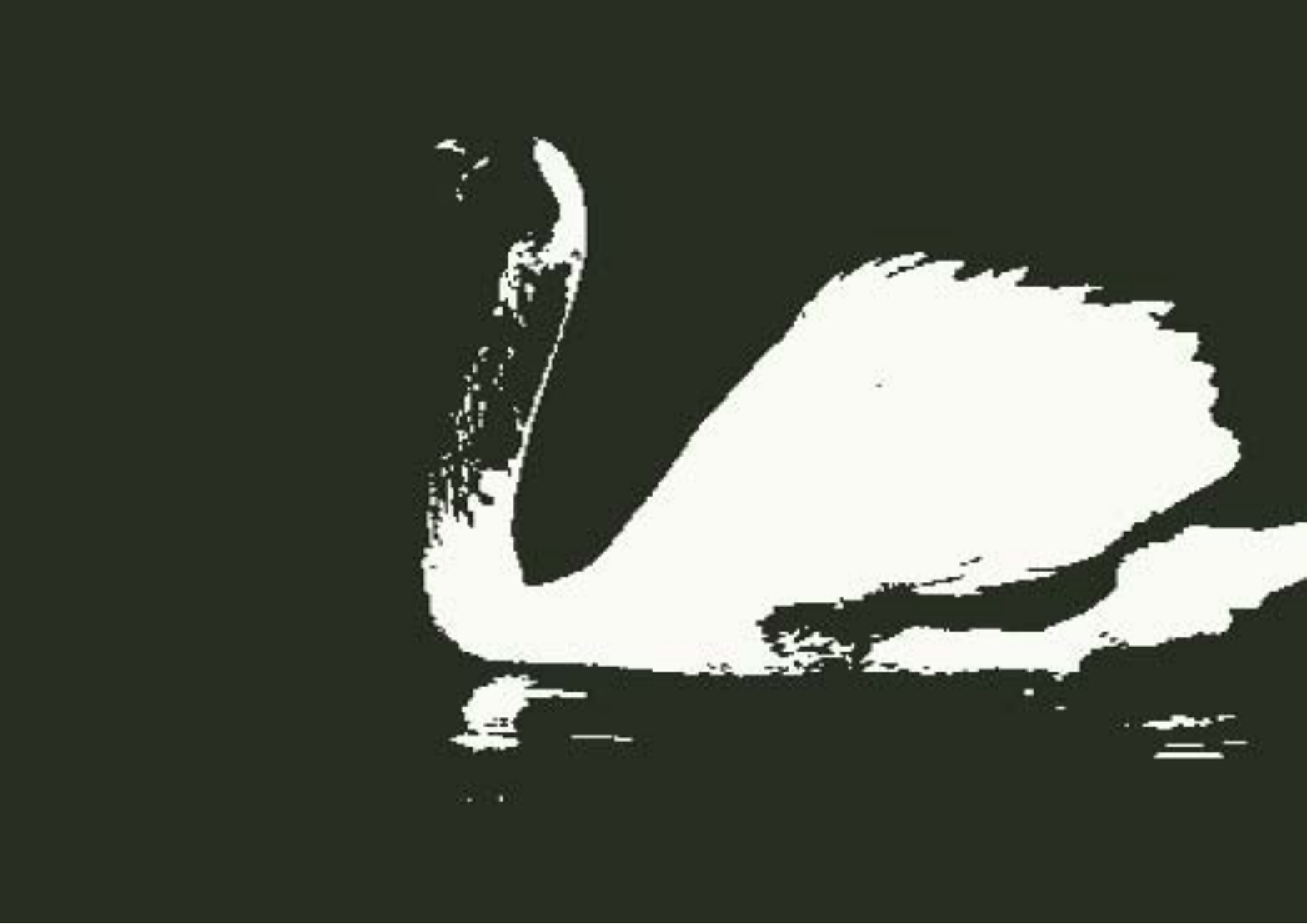}}
		\subfigure{\includegraphics[width=0.16\textwidth]{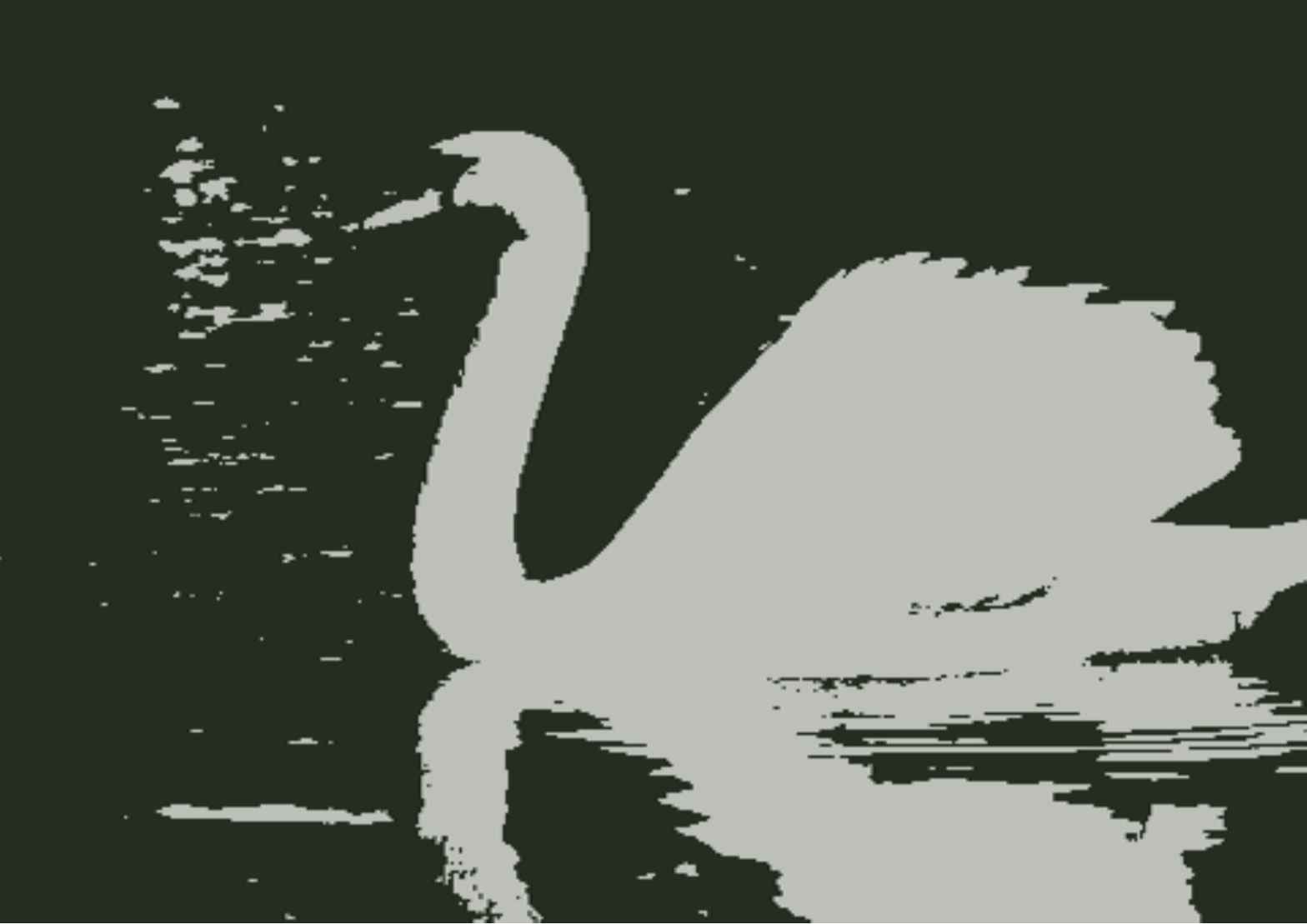}}
		\subfigure{\includegraphics[width=0.16\textwidth]{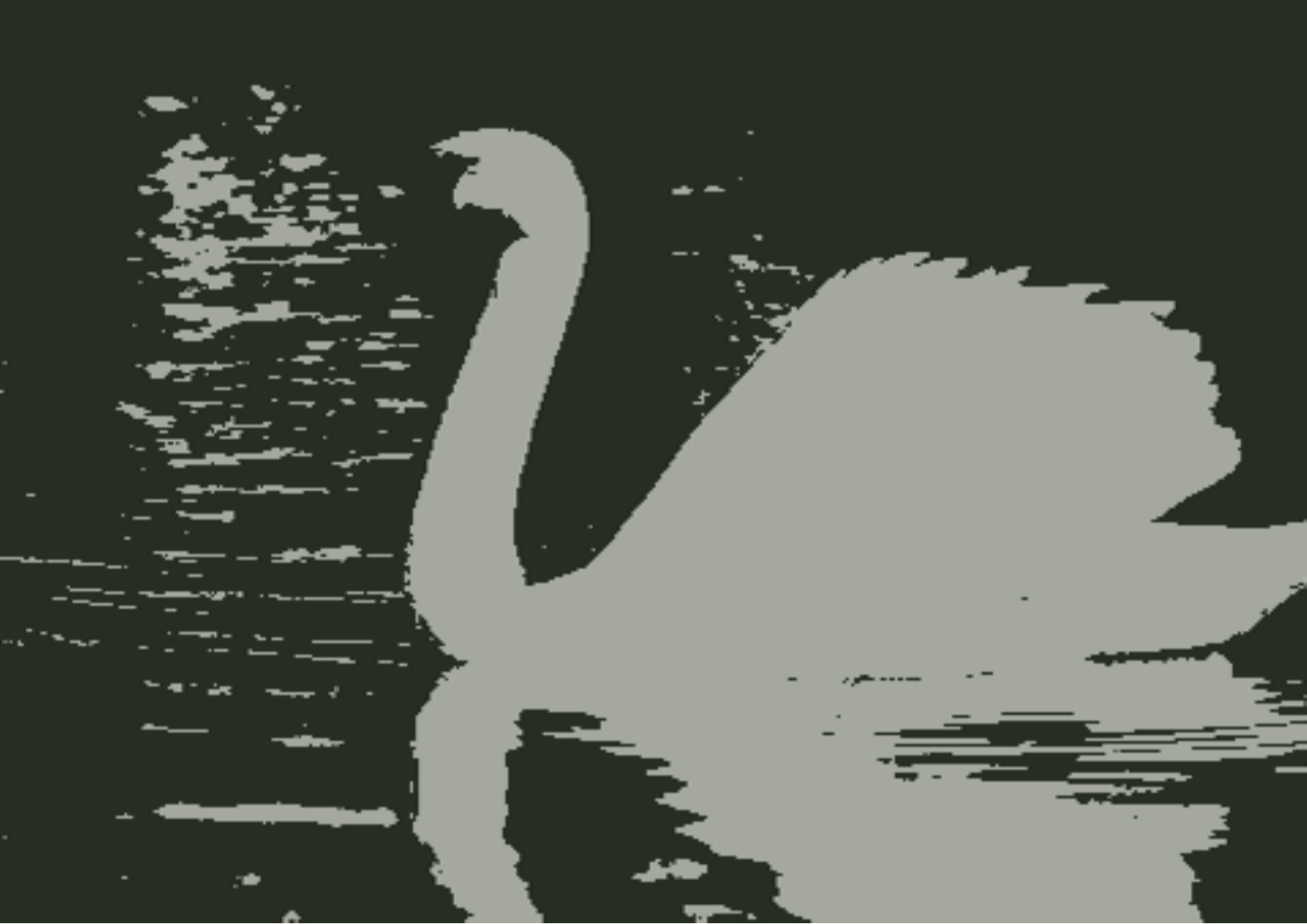}}
		\subfigure{\includegraphics[width=0.16\textwidth]{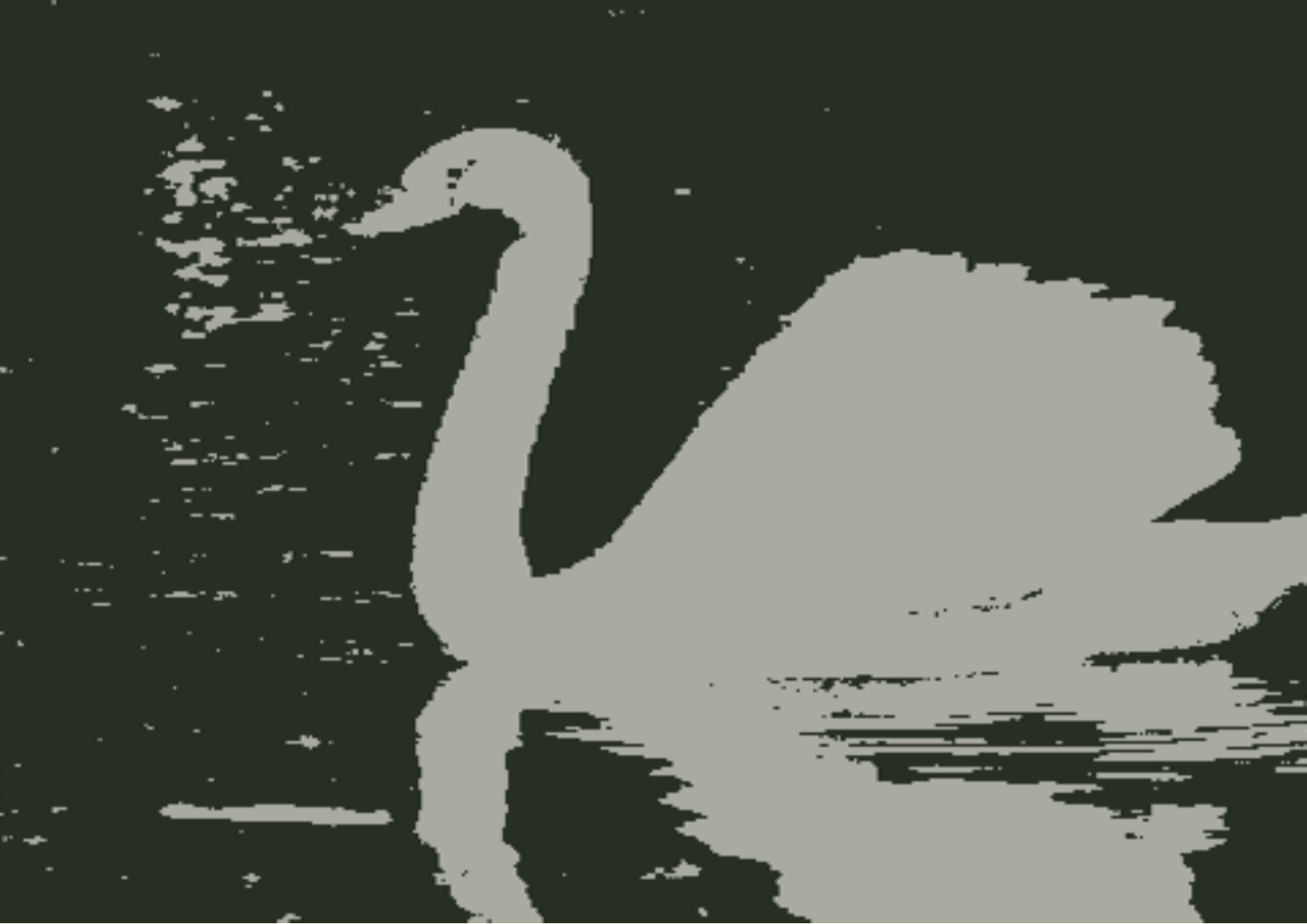}}
		\subfigure{\includegraphics[width=0.16\textwidth]{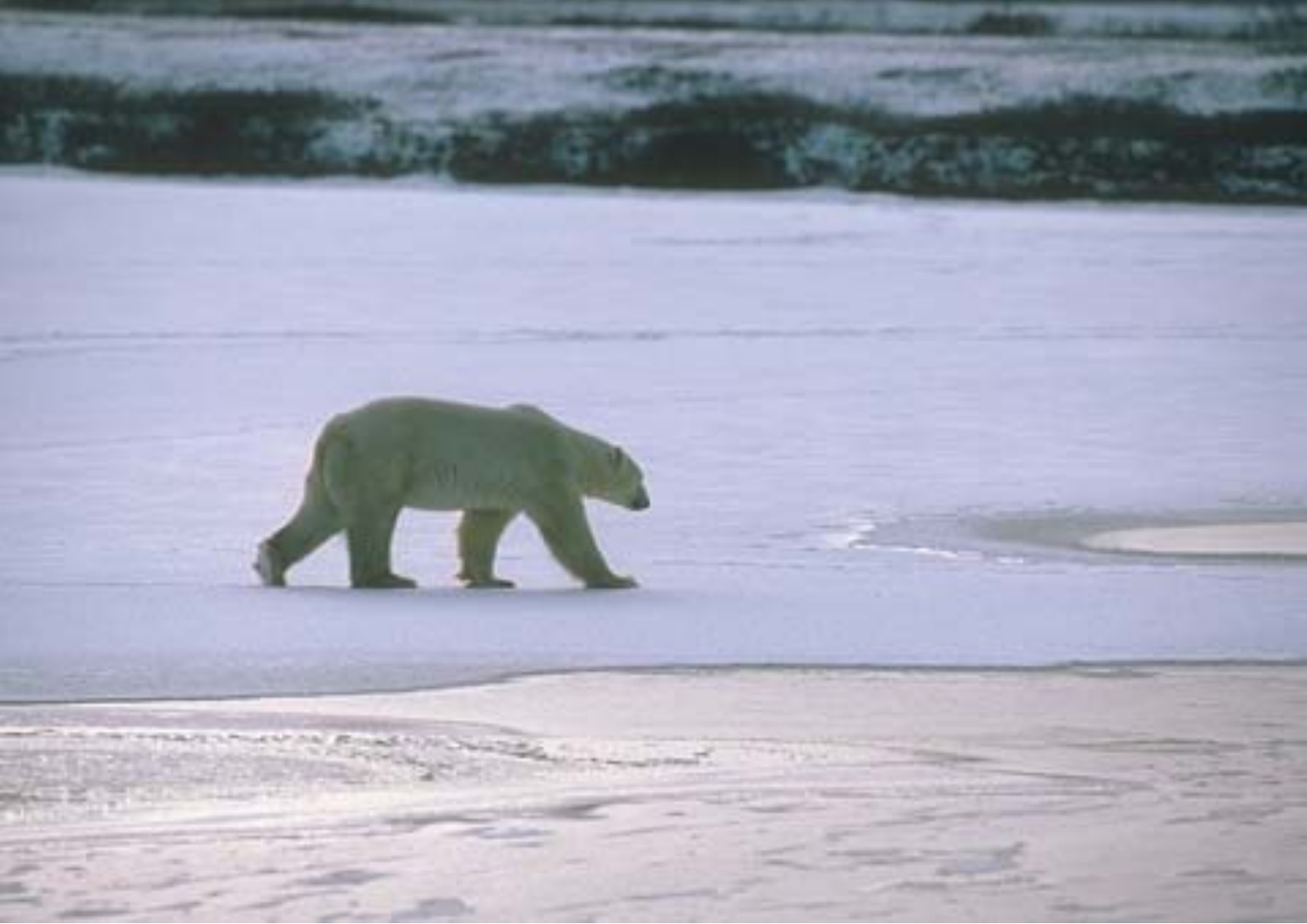}}
		\subfigure{\includegraphics[width=0.16\textwidth]{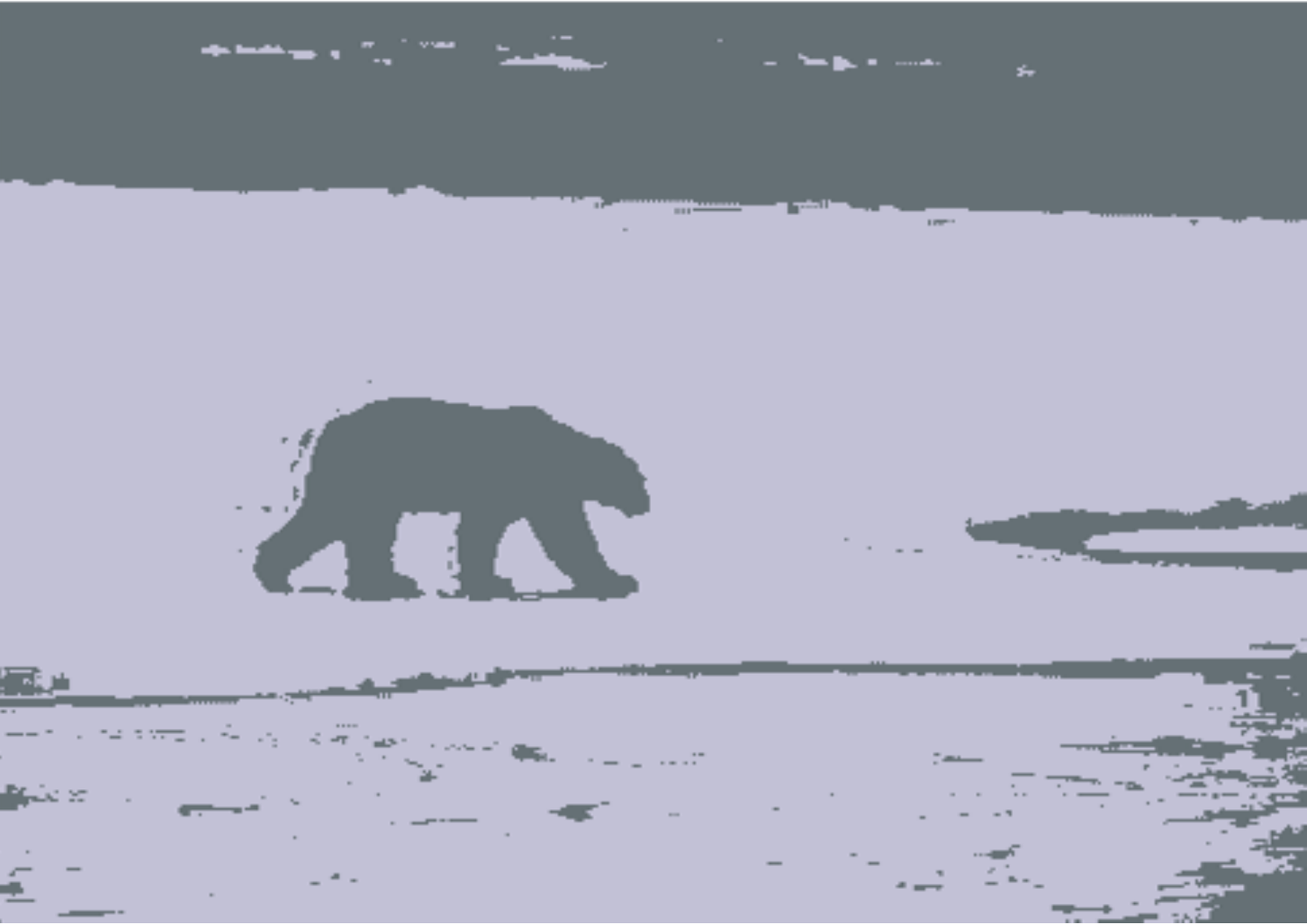}}
		\subfigure{\includegraphics[width=0.16\textwidth]{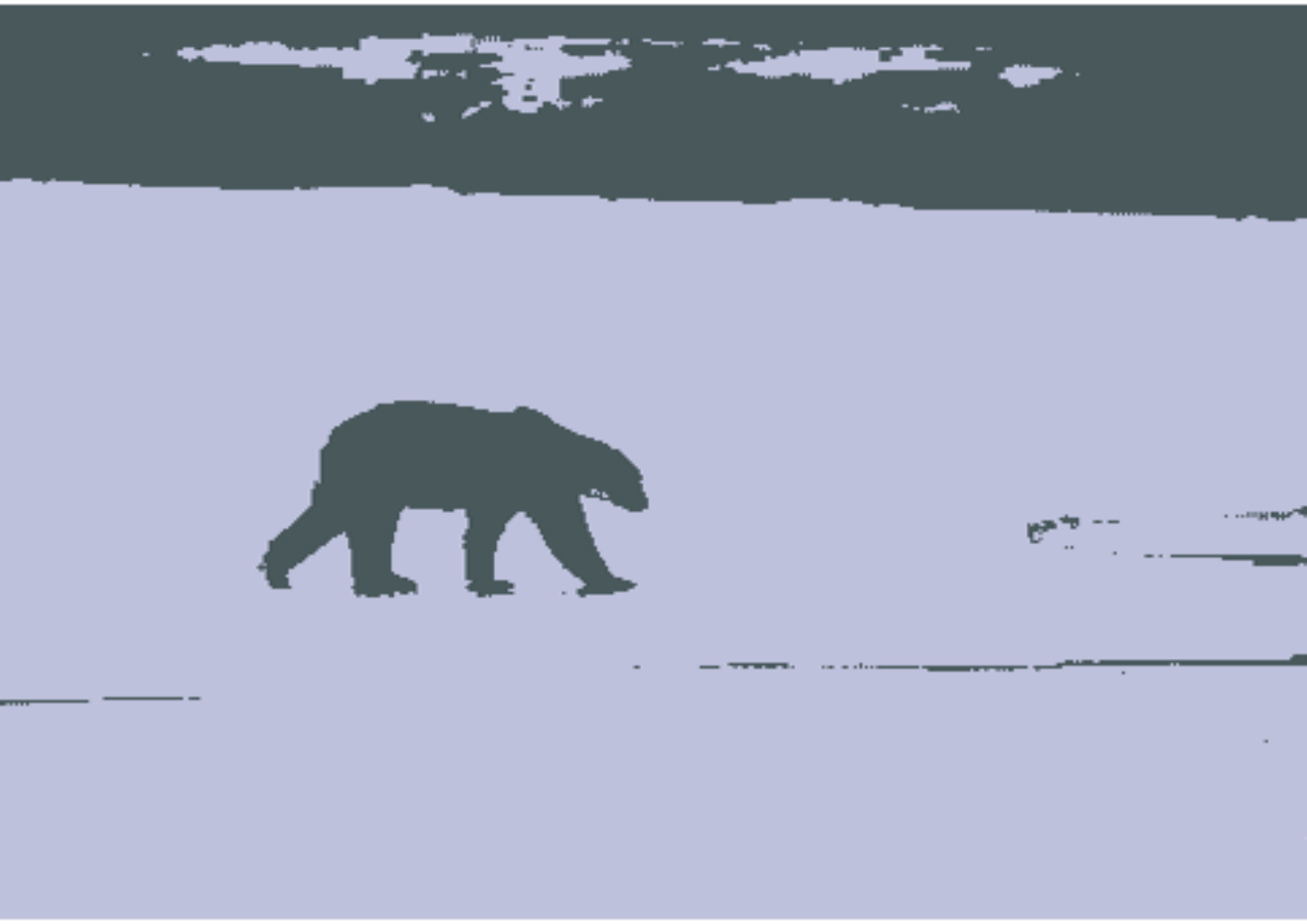}}
		\subfigure{\includegraphics[width=0.16\textwidth]{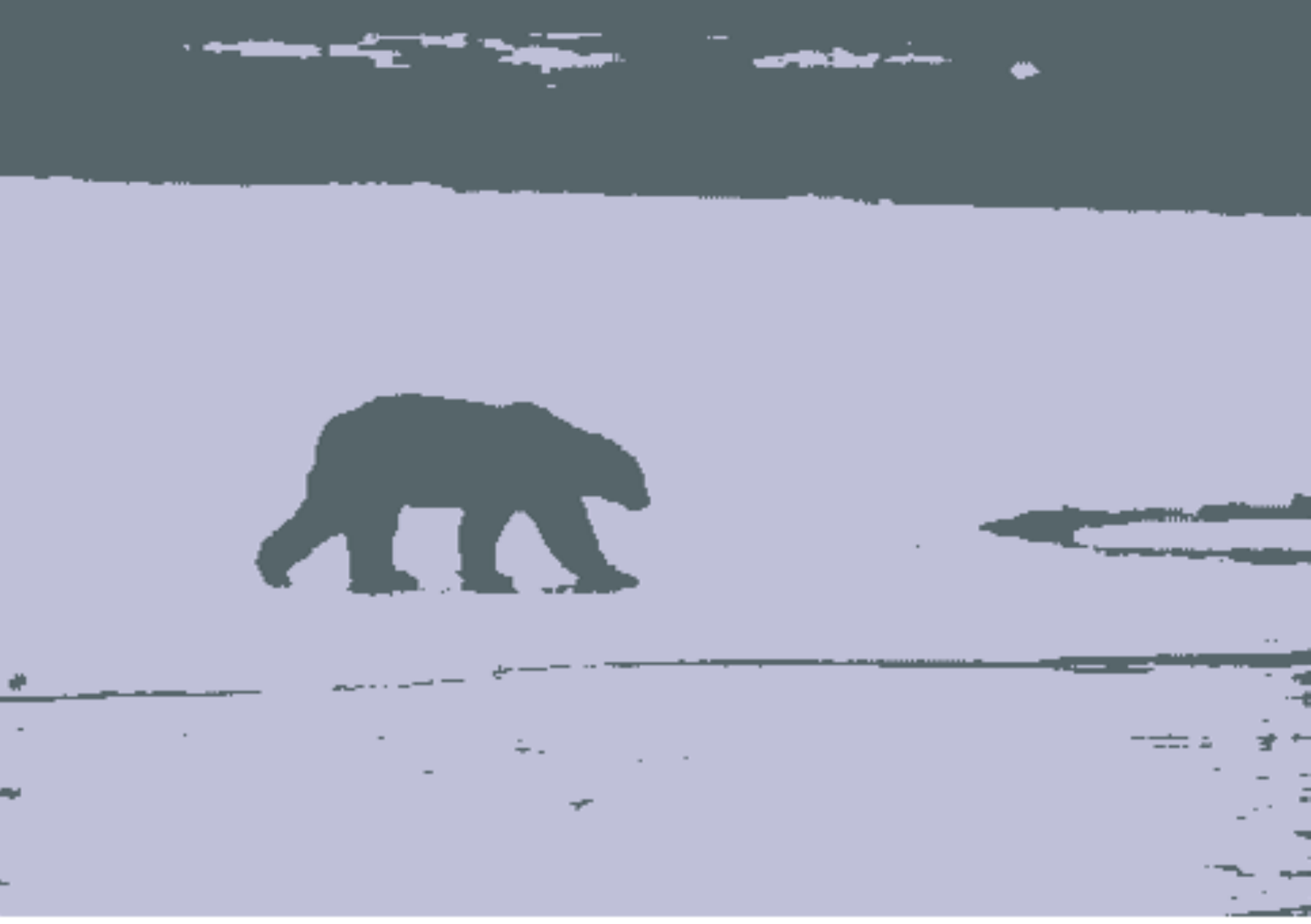}}
		\subfigure{\includegraphics[width=0.16\textwidth]{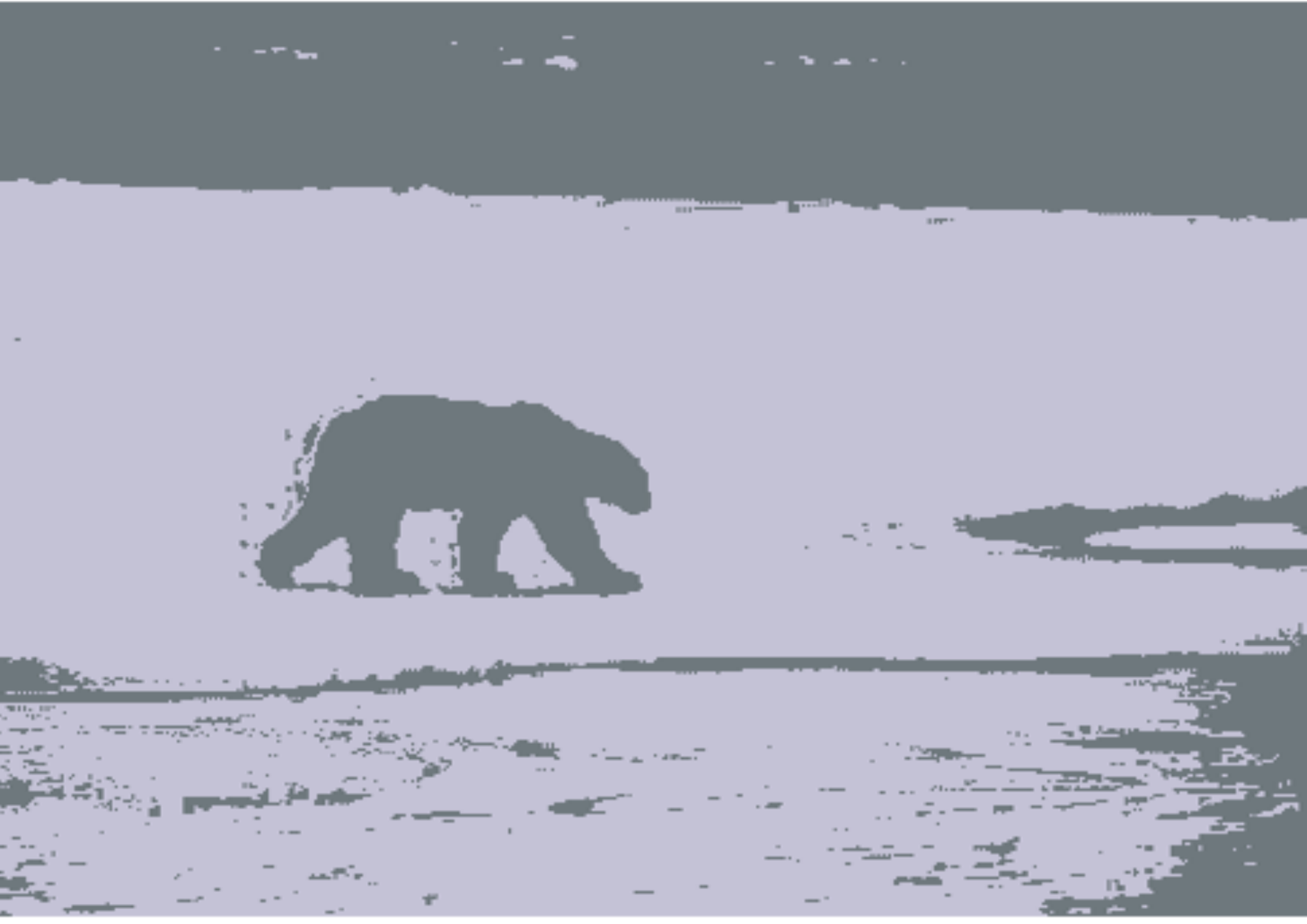}}
		\subfigure{\includegraphics[width=0.16\textwidth]{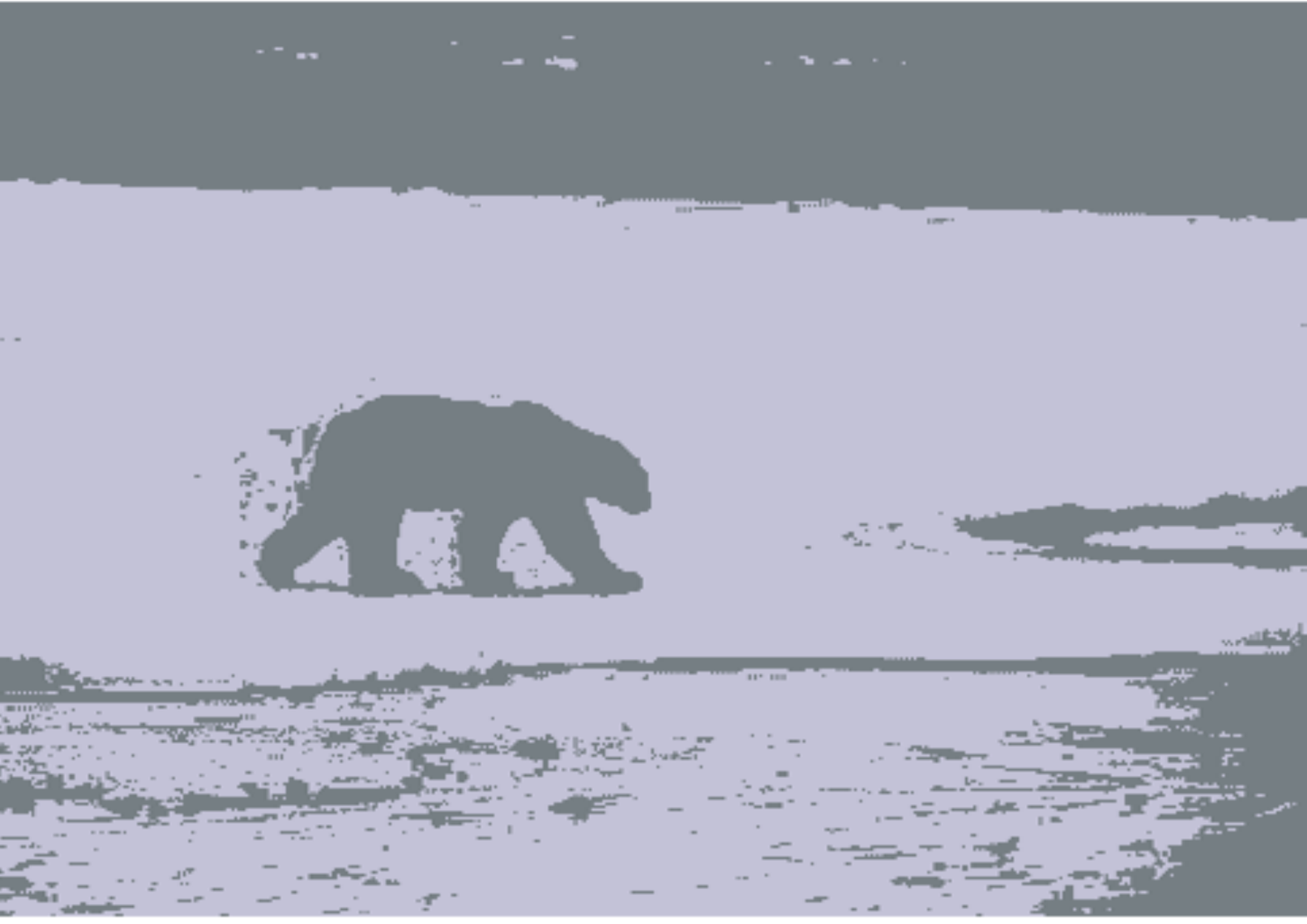}}	
		\setcounter{subfigure}{0}	
		\subfigure[Original]{\includegraphics[width=0.16\textwidth]{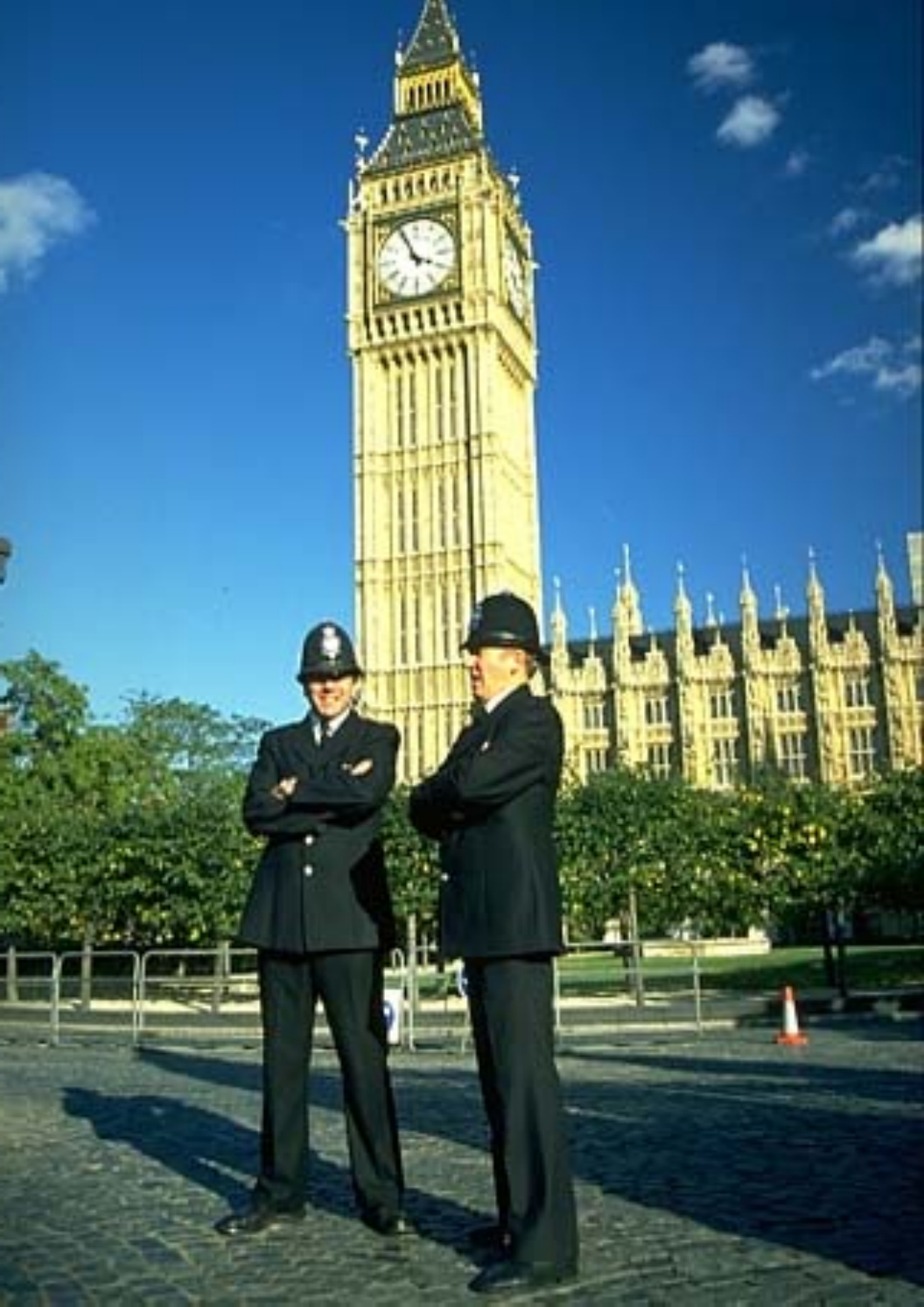}}
		\subfigure[Gaussian]{\includegraphics[width=0.16\textwidth]{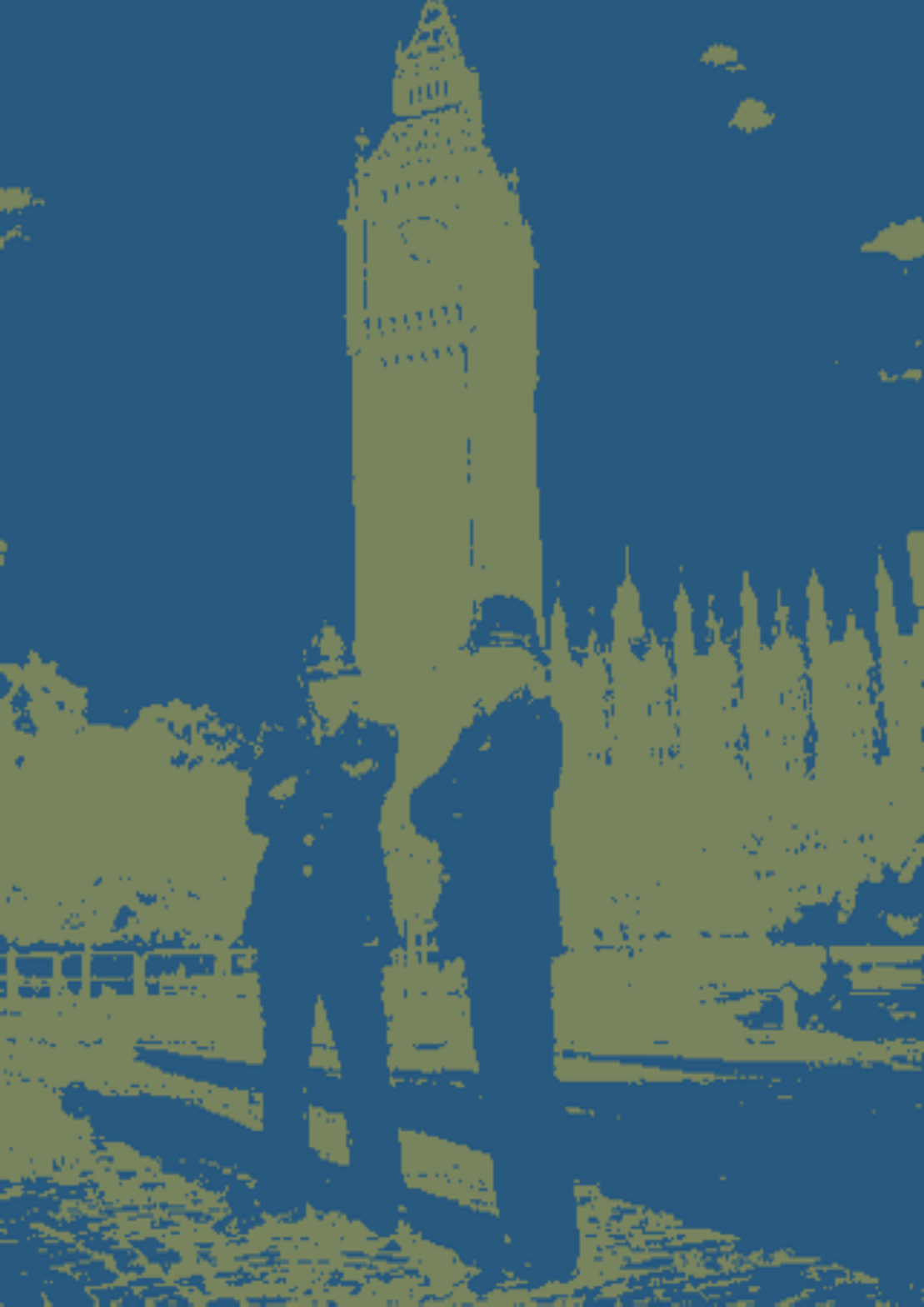}}
		\subfigure[Cauchy]{\includegraphics[width=0.16\textwidth]{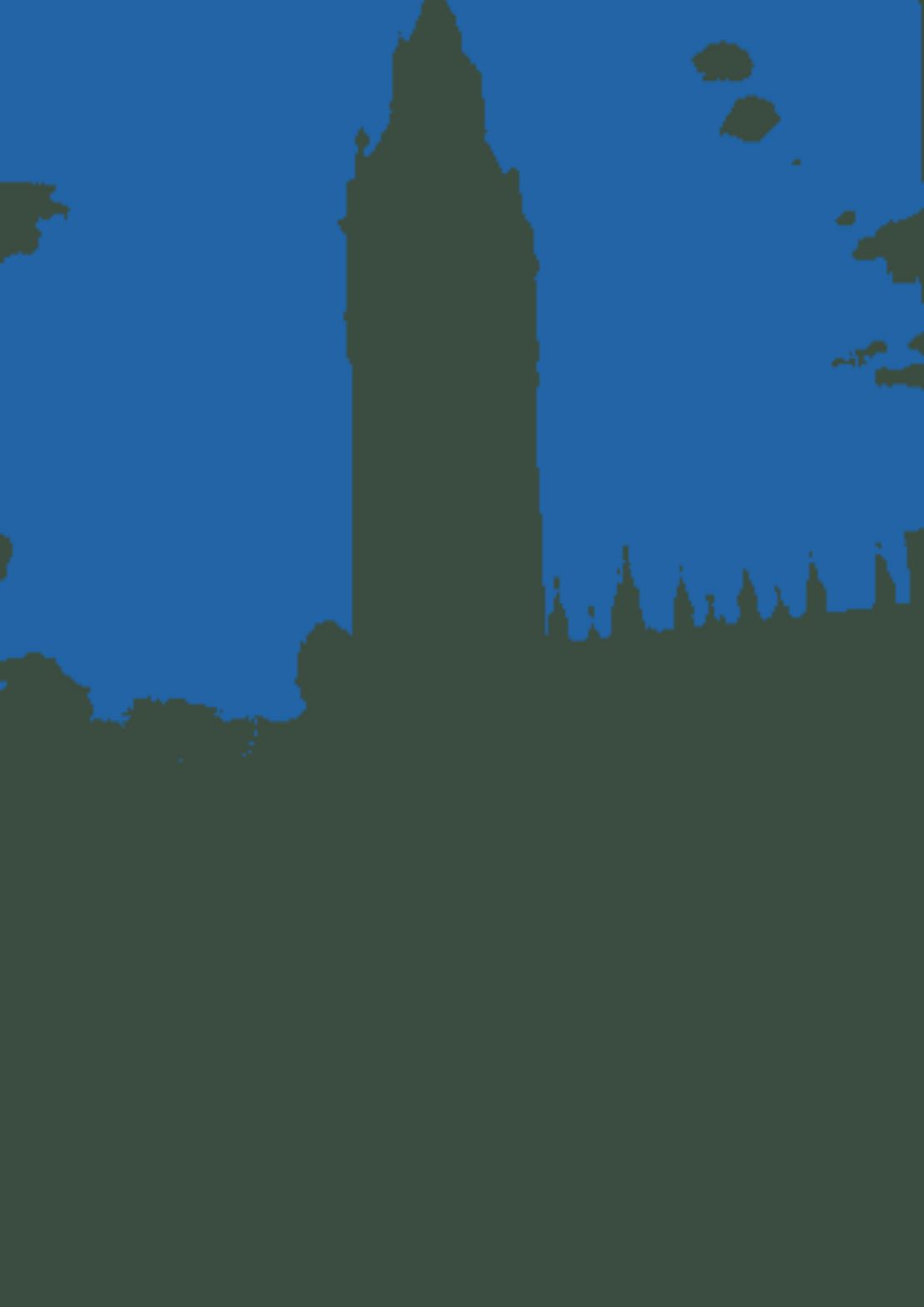}}
		\subfigure[Laplace]{\includegraphics[width=0.16\textwidth]{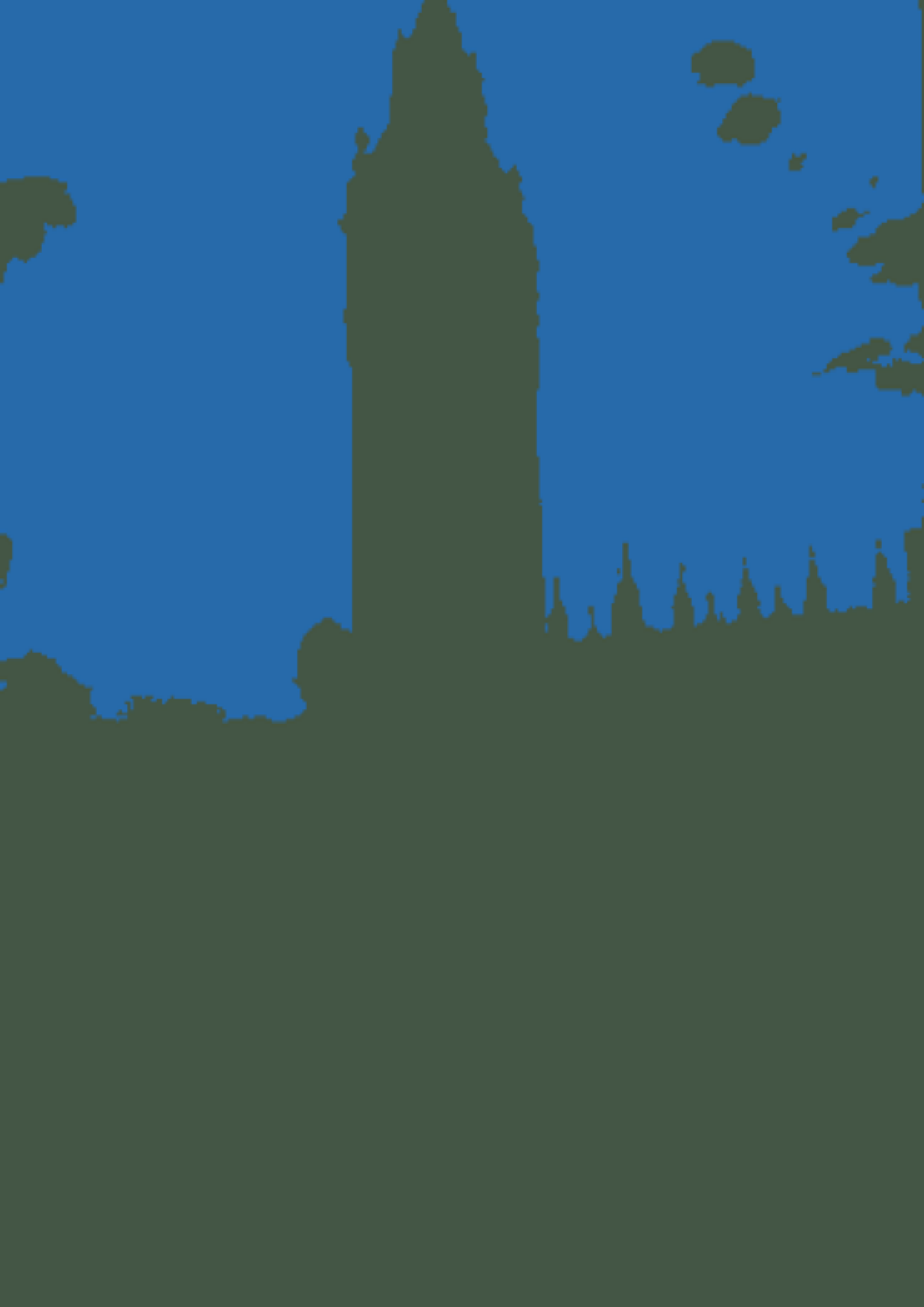}}
		\subfigure[Logistic]{\includegraphics[width=0.16\textwidth]{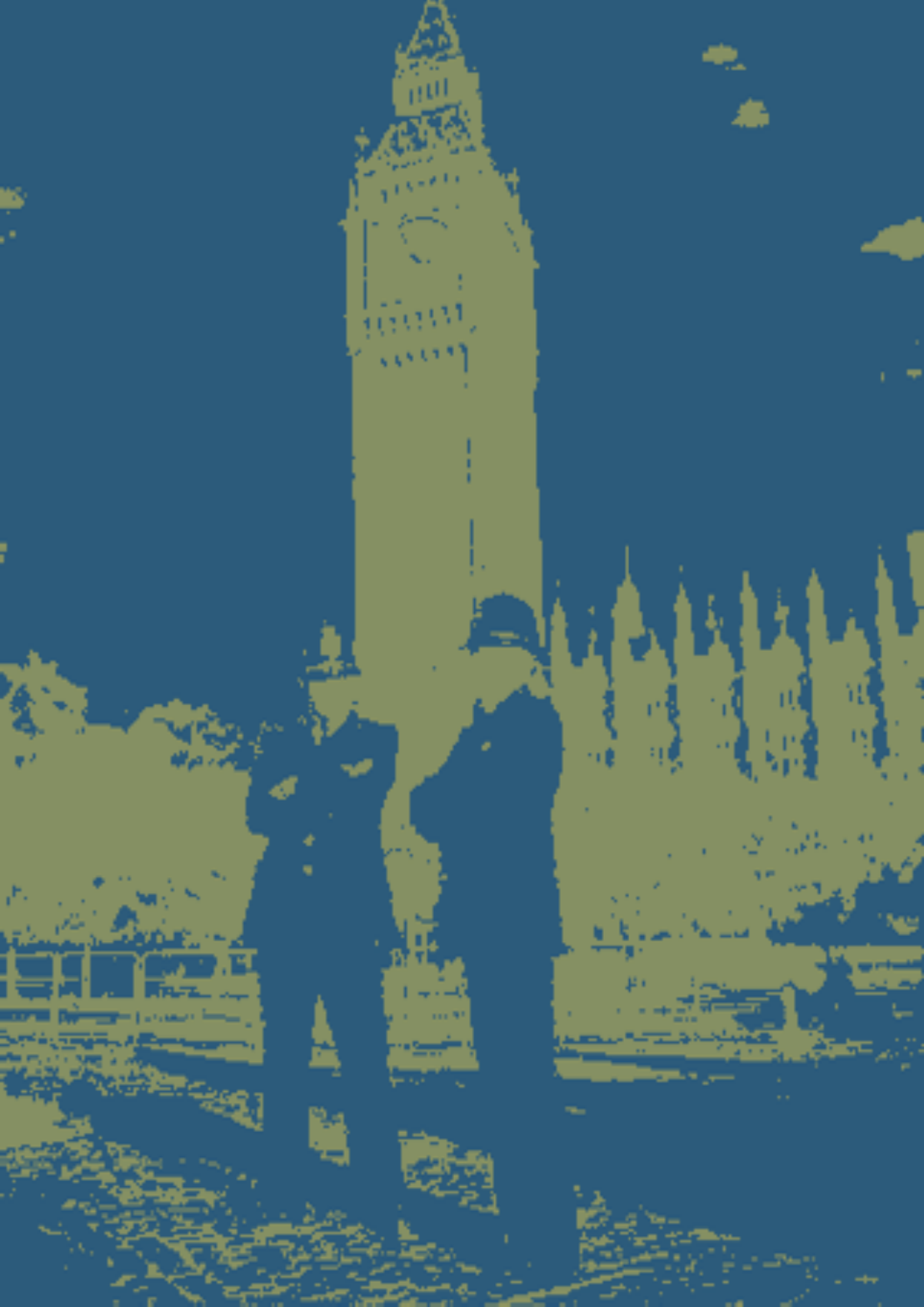}}
		\subfigure[GG1.5]{\includegraphics[width=0.16\textwidth]{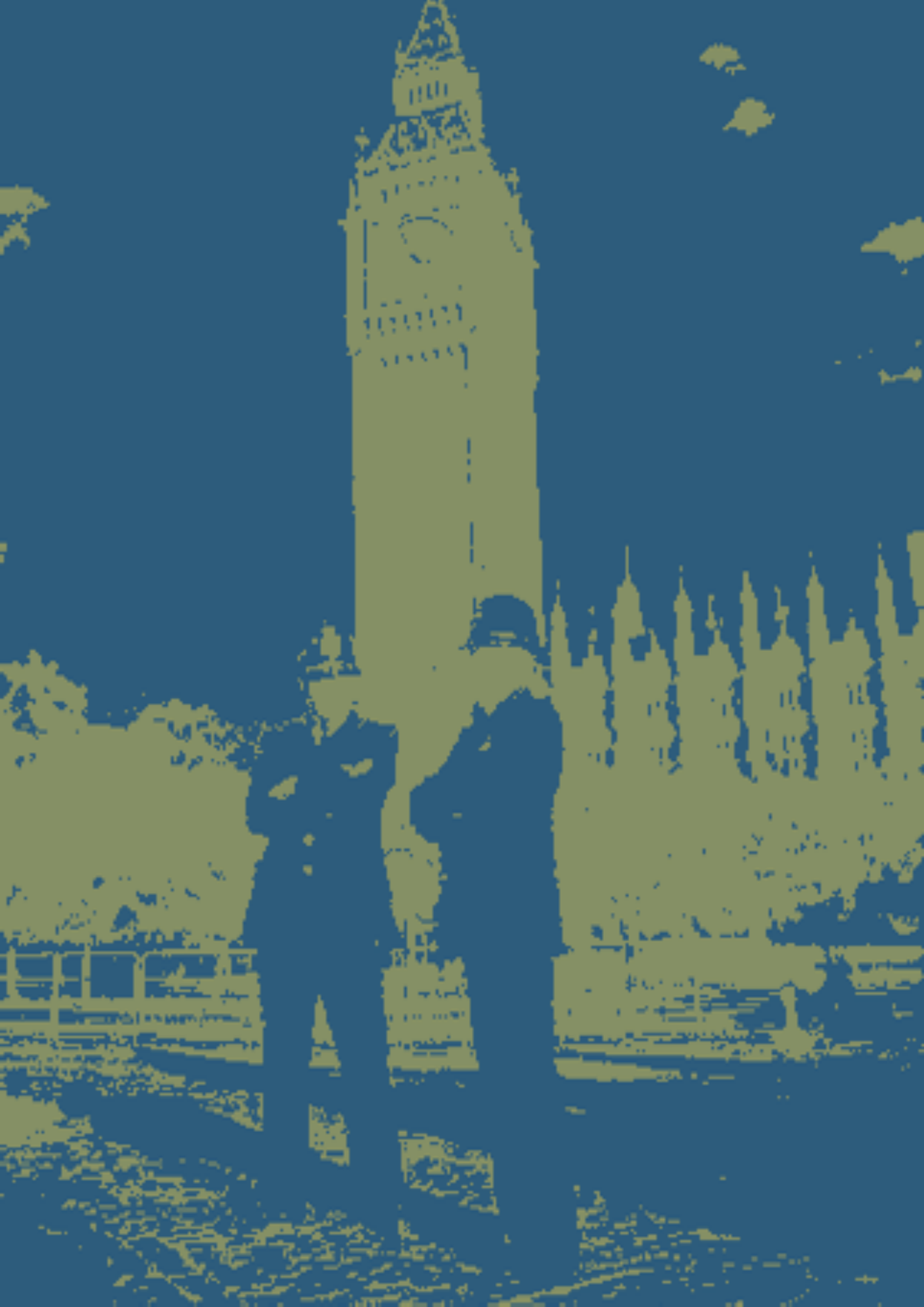}}
	\end{center}
\vspace{-1em}
	\caption{Reconstructed images by Our method when optimising the 5 EMMs.}\label{esifigus}
	\vspace{-1em}
\end{figure*}

%\begin{figure}[!htb]
%	\begin{center}
%		\subfigure[Gaussian]{\includegraphics[width=0.15\textwidth]{BSDGaussian.png}}
%		\subfigure[Cauchy]{\includegraphics[width=0.15\textwidth]{BSDCauchy.png}}
%		\subfigure[Laplace]{\includegraphics[width=0.15\textwidth]{BSDLaplace.png}}
%		\subfigure[GG1.5]{\includegraphics[width=0.15\textwidth]{BSDGG15.png}}
%		\subfigure[Logistic]{\includegraphics[width=0.15\textwidth]{BSDLogistic.png}}
%		\subfigure[Overall]{\includegraphics[width=0.15\textwidth]{BSDOverall.png}}
%	\end{center}
%	\caption{Average cost errors over 500 pictures against numbers of iterations for estimating 5 EMMs among the 3 methods.}\label{convergence}
%\end{figure}

\begin{table*}[!htb]
	\centering
	\caption{Comparisons averaged over 500 pictures among Our, the IRA and the RMO for estimating the 5 EMMs}\label{convtable}
	\scriptsize{
		\begin{tabular}{c|ccc|ccc|ccc|ccc|ccc}\hline\hline
			& \multicolumn{3}{c|}{Gaussian} & \multicolumn{3}{c|}{Cauchy} & \multicolumn{3}{c|}{Laplace} & \multicolumn{3}{c|}{GG1.5} & \multicolumn{3}{c}{Logistic} \\
			& Our      & IRA     & RMO     & Our     & IRA     & RMO    & Our     & IRA     & RMO     & Our     & IRA    & RMO    & Our      & IRA     & RMO     \\
			Iterations  & \textbf{56}       & 234     & 954     & \textbf{87}      & 390     & 939    & \textbf{101}     & 288     & 962     & \textbf{61}      & 224    & 982    & \textbf{58}       & 239     & 965     \\
			Time (s)  & \textbf{5.55}     & 11.8    & 165     & \textbf{7.12}    & 16.8    & 159    & \textbf{20.4}    & 42.2    & 486     & \textbf{8.32}    & 14.4   & 205    & \textbf{6.73}     & 13.3    & 177     \\
			Cost & \textbf{12.3}     & 12.3    & 12.3    & \textbf{12.1}    & 12.4    & 12.3   & \textbf{12.2}    & 12.2    & 12.3    & \textbf{12.4}    & 12.4   & 12.4   & \textbf{11.1}     & 11.1    & 11.2 \\\hline
			SSIM \cite{wang2004image} & \multicolumn{3}{c|}{0.5930} & \multicolumn{3}{c|}{0.6052} & \multicolumn{3}{c|}{0.6112} & \multicolumn{3}{c|}{0.5813} & \multicolumn{3}{c}{0.5792} \\\hline\hline   
	\end{tabular}}
\end{table*}
\begin{table*}[!htb]
	\caption{Comparisons among Our (with different manifold solvers), IRA and RMO methods on the BSDS500 image patches}\label{diffsolvers}
	\scriptsize{\resizebox{\textwidth}{!}{
			\begin{tabular}{c|ccccc|ccccc}\hline \hline
				& \multicolumn{5}{c|}{Gaussian}                         & \multicolumn{5}{c}{Cauchy}                           \\
				{$K, M=3, 27$} & Our             & Our (STP)             & Our (LBF)              & IRA             &   RMO           & Our             & Our (STP)           & Our (LBF)           & IRA   & RMO          \\\hline
				Time (s)    & \textbf{28.3$\pm$7.62} & 210$\pm$109 & 102$\pm$8.67 & 155$\pm$37.0 & 486$\pm$197 & \textbf{118$\pm$72.9} & 224$\pm$106 & 302$\pm$139 & ---- & 556$\pm$26.2\\
				Iterations       & \textbf{109$\pm$27.5}  &  749$\pm$390 & 318$\pm$37.3 & 930$\pm$221  & 844$\pm$342 & \textbf{416$\pm$254} & 817$\pm$387 & 670$\pm$253 & ---- & 983$\pm$54.4 \\
				Cost           &  72.8$\pm$1.30   & 73.6$\pm$6.30 & \textbf{72.0$\pm$1.84} & 73.9$\pm$2.01  & 75.5$\pm$6.72 & 70.7$\pm$1.53  & \textbf{69.5$\pm$4.63}  & 70.5$\pm$1.47 & ----  & 70.7$\pm$1.82 \\	\hline
				{$K, M=9, 75$}                   & Our             & Our (STP)             & Our (LBF)              & IRA             &   RMO           & Our             & Our (STP)           & Our (LBF)           & IRA   & RMO          \\\hline
				Time (s)    & 170$\pm$23.3 & \textbf{125$\pm$34.1} & 1609$\pm$74.7 & ---- &  4568$\pm$904&\textbf{595$\pm$473} &  2025$\pm$38.6& 898$\pm$157 &----  & 4389$\pm$1229\\
				Iterations       & 90.0$\pm$8.80  & \textbf{63.0$\pm$25.5}  & 998$\pm$7.07 & ----  & 921$\pm$177 &\textbf{295$\pm$230}  & 1e3$\pm$0 & 394$\pm$93.0 & ---- & 918$\pm$259 \\
				Cost           & 160$\pm$3.33   & 178$\pm$21.1 & \textbf{160$\pm$2.89} & ----  & 172$\pm$0.31 & \textbf{170$\pm$0.25}  & 171$\pm$3.32  & 171$\pm$2.37 & ----  & 173$\pm$4.31 \\	\hline
				\hline
	\end{tabular}}}
\vspace{-1em}
\end{table*}
\subsection{BSDS500 dataset:}\label{imagedata}
Convergence speed between Our, the IRA, and the RMO methods was compared over averaged results across the whole 500 pictures in BSDS500, as shown in Table \ref{convtable}. Again, our method converged with the fastest speed and achieved the minimum cost error among the RMO and the IRA methods. 
We further show in Fig. \ref{esifigus} four reconstructed images via our method when optimising the five EMMs. Observe that the images reconstructed from the non-robust distributions (e.g., the Gaussian and the GG1.5 distributions) were more ``noisy'' than those from the robust distributions (e.g., the Cauchy and Laplace distributions); however, these may capture more details in images. Thus, by our method, different EMMs can flexibly model data for different requirements or applications.

Finally, we evaluated our method on modelling the BSDS500 image patches, which is a challenging but important task in practice. More comprehensively, besides the manifold conjugate gradient solver adopted as a default of our method, we further evaluated the manifold steepest descent (denoted as Our (STP)) and the manifold LBFGS (denoted as Our (LBF)) as alternatives to solve the step descent in Algorithm \ref{Alg1}. The results are reported in Table \ref{diffsolvers}. From this table, we can see again that our method consistently achieved the lowest costs across different initialisations. For example, for the Gaussian mixtures, it obtained $72.8\pm1.30$ in the $50$ random initialisations, compared to $73.9\pm2.01$ of the IRA and $75.5\pm6.72$ of the RMO. Also, our method converged with the least computational time and was able, compared to the existing methods, to consistently provide fast, stable and superior optimisation. Furthermore, by employing different solvers, the results achieved by our method are comparable to one another; among the three solvers, the steepest descent solver took slightly more computational time due to its naive update rule, whilst the conjugate gradient solver performed slightly better than the other two solvers in balancing the computational time and cost. Overall, all the three solvers of our method achieved much better results than the IRA and the RMO methods, which also verifies the effectiveness of our re-designed cost. Therefore, we can conclude that for various scenarios and applications, our method was able to consistently yield a superior model with the lowest computational time.

\section{Conclusions}
We have proposed a general framework for a systematic analysis and universal optimisation of the EMMs, and have conclusively demonstrated that this equips EMMs with significantly enhanced flexibility and ease of use in practice. In addition to the general nature and the power of the proposed universal framework for EMMs, we have also verified both analytically and through simulations, that this provides a reliable and robust statistical tool for analysing the EMMs. Furthermore, we have proposed a general solver which consistently attains the optimum for general EMMs. Comprehensive simulations over both synthetic and real-world datasets validate the proposed framework, which is fast, stable and flexible.

\newpage
\appendix
\subsection{Proof of Lemma 1}\label{prooflemma1}

To calculate the Riemannian metric for the covariance matrix, we follow the work of \cite{hiai2009riemannian} to calculate the Hessian of the Boltzman entropy of elliptical distributions. The Boltzman entropy is first obtained as follows,
\begin{equation}
\begin{aligned}
H(\mathbf{x}|\mathbf{{\Sigma}}) &= \int_{\mathbf{x}}p_{\mathbf{x}}(\mathbf{x})\ln p_{\mathbf{x}}(\mathbf{x})d\mathbf{x} \\
&= \int_{\mathbf{x}} p_{\mathbf{x}}(\mathbf{x})[-\frac{1}{2}\ln|\mathbf{{\Sigma}}|+\ln c_M+\ln g(t)]d\mathbf{x}\\
&=-\frac{1}{2}\ln |\mathbf{\Sigma}|+\ln c_M + \int_{\mathbb{R}^{M}}p_\mathcal{R}(t)\ln g(t)dt.
\end{aligned}
\end{equation}
Because $(\ln c_M + \int_{\mathbb{R}^{M}}p_\mathcal{R}(t)\ln g(t)dt)$ is irrelevant to $\mathbf{{\Sigma}}$, the Hessian of $H(\mathbf{x}|\mathbf{{\Sigma}})$ can be calculated as
\begin{equation}
\frac{\partial H(\mathbf{x}|\mathbf{{\Sigma}}+s{\mathbf{\Sigma}_0}+h\mathbf{\Sigma}_1)}{\partial s \partial h}|_{s=0,h=0} = \frac{1}{2}\mathrm{tr}({\mathbf{\Sigma}_0} \mathbf{\Sigma}^{-1} {\mathbf{\Sigma}_1} \mathbf{\Sigma}^{-1}).
\end{equation}	
The Riemannian metric can thus be obtained as $ds^2 =\frac{1}{2} \mathrm{tr}(d\mathbf{\Sigma}\mathbf{{\Sigma}}^{-1}d\mathbf{\Sigma}\mathbf{{\Sigma}}^{-1})$, which is the same as the case for multivariate normal distributions and is the mostly widely used metric.

This completes the proof of Lemma 1.

\subsection{Proof of Theorem 1}\label{prooftheorem1}
	The proof of this property rests upon an expansion $\mathbf{y}_n^T \mathbf{\tilde{\Sigma}}_k^{-1} \mathbf{y}_n$ to become $(\mathbf{x}_n - \bm{\mu}_k)^T \mathbf{\Sigma}_k^{-1} (\mathbf{x}_n - \bm{\mu}_k) + \frac{1}{\lambda_k}$ within derivatives of $\tilde J$. By setting $\nicefrac{\partial \tilde J}{\partial \lambda_k}=0$ and $\nicefrac{\partial \tilde J}{\partial c_k}=0$, we then arrive at
\begin{equation}\label{lambdavalues}
\begin{aligned}
\lambda_k &= -2\frac{\sum_{n=1}^N\tilde{\xi}_{nk}\psi_k(t_{nk}+\frac{1}{\lambda_k}-c_k)}{\sum_{n=1}^N\tilde{\xi}_{nk}},\\
c_k &= -\frac{1}{2}\frac{\sum_{n=1}^N\tilde{\xi}_{nk}}{\sum_{n=1}^N\tilde{\xi}_{nk}\psi_k(t_{nk}+\frac{1}{\lambda_k}-c_k)},
\end{aligned}
\end{equation}
where 
\begin{equation}
\tilde{\xi}_{nk}\! =\! \frac{\pi_k\cdot c_M \cdot \sqrt{c_k\cdot\mathrm{det}(\mathbf{\tilde{\Sigma}}_k^{-1})} \cdot  g_k\left(\mathbf{y}_n^T \mathbf{\tilde{\Sigma}}_k^{-1} \mathbf{y}_n - c_k\right)}{\sum_{k = 1}^K \pi_k\cdot c_M \cdot \sqrt{c_k\cdot\mathrm{det}(\mathbf{\tilde{\Sigma}}_k^{-1})} \cdot  g_k\left(\mathbf{y}_n^T \mathbf{\tilde{\Sigma}}_k^{-1} \mathbf{y}_n \!-\! c_k\right)}.
\end{equation}

By inspecting $\lambda_k = \nicefrac{1}{c_k}$, $\mathrm{det}(\mathbf{\tilde{\Sigma}}_k) = \lambda_k\cdot\mathrm{det}(\mathbf{\Sigma}_k)$ and $\mathbf{y}_n^T \mathbf{\tilde{\Sigma}}_k^{-1} \mathbf{y}_n=(\mathbf{x}_n - \bm{\mu}_k)^T \mathbf{\Sigma}_k^{-1} (\mathbf{x}_n - \bm{\mu}_k) + \frac{1}{\lambda_k}$, we obtain $\tilde{\xi}_{nk} = \xi_{nk}$ and $\psi_k(t_{nk}+\frac{1}{\lambda_k}-c_k) = \psi_k(t_{nk})$.

To prove the equivalence of the optimum of $\mathbf{\tilde{\Sigma}}_k$ in \eqref{reformuJ} and optima of $\bm{\mu}_k$ and $\mathbf{\Sigma}_k$, we directly calculate $\nicefrac{\partial \tilde{J}}{\partial \mathbf{\tilde{\Sigma}}_k}$ in \eqref{reformuJ} and set it to $0$, to yield
\begin{equation}
\mathbf{\tilde{\Sigma}}_k = -2\frac{\sum_{n=1}^N \tilde{\xi}_{nk}\psi_k(t_{nk}+\frac{1}{\lambda_k}-c_k) \mathbf{y}_n\mathbf{y}_n^T}{\sum_{n=1}^N\tilde{\xi}_{nk}}.
\end{equation}
Again, as $\lambda_k = \nicefrac{1}{c_k}$ and $\tilde{\xi}_{nk} = \xi_{nk}$, we arrive at
\begin{equation}
\begin{aligned}
\mathbf{\tilde{\Sigma}}_k  &= -2\frac{\sum_{n=1}^N \xi_{nk}\psi_k(t_{nk}) \mathbf{y}_n\mathbf{y}_n^T}{\sum_{n=1}^N\xi_{nk}} \\
&= -2\frac{\sum_{n=1}^N \xi_{nk}\psi_k(t_{nk})}{\sum_{n=1}^N\xi_{nk}} \begin{bmatrix}
\mathbf{x}_n\mathbf{x}_n^T & \mathbf{x}_n \\
\mathbf{x}_n^T & 1 \\
\end{bmatrix}.
\end{aligned}
\end{equation}
By substituting $\bm{\mu}_k$ and $\mathbf{\Sigma}_k$ of \eqref{solution} and $\lambda_k$ in \eqref{lambdavalues}, we have
\begin{equation}
\mathbf{\tilde{\Sigma}}_k = \begin{bmatrix}
\mathbf{\Sigma}_k + \lambda_k\bm{\mu}_k \bm{\mu}_k^T & \lambda_k\bm{\mu}_k\\
\lambda_k\bm{\mu}_k^T & \lambda_k\\
\end{bmatrix},
\end{equation}
which means that the optimum value $\mathbf{\tilde{\Sigma}}_k$ is exactly the reformulated form by the optimum values of \eqref{solution}.

This completes the proof of Theorem 1.
\subsection{Proof of Lemma 3}\label{prooflemma3}
	We here denote the contaminated distribution $\mathcal{F} = (1-\epsilon)\mathcal{F}_{\mathbf{x}}+\epsilon\mathcal{F}_{\mathbf{x}_0}$, where $\epsilon$ is the proportion of outliers; $\mathcal{F}_{\mathbf{x}}$ is the true distribution of $\mathbf{x}$ and $\mathcal{F}_{\mathbf{x}_0}$ is the point-mass distribution at $\mathbf{x}_0$. For simplicity, we employ $t$ to denote $(\mathbf{x}-\bm{\mu}_j)^T\mathbf{\Sigma}_j^{-1}(\mathbf{x}-\bm{\mu}_j)$ and $t_0$ to denote $(\mathbf{x}_0-\bm{\mu}_j)^T\mathbf{\Sigma}_j^{-1}(\mathbf{x}_0-\bm{\mu}_j)$. Then, the maximum log-likelihood estimation on the $\mathbf{\Sigma}_j$ of $\mathcal{E}_j(\bm{\mu}_j, \mathbf{\Sigma}_j, g_j)$ becomes
\begin{equation}\label{appendIF1}
\begin{aligned}
&(1-\epsilon)\mathbb{E}[\xi_j(\mathbf{x})\psi_j(t)(\mathbf{x}-\bm{\mu}_j)(\mathbf{x}-\bm{\mu}_j)^T+\frac{1}{2}\xi_j(\mathbf{x})\mathbf{\Sigma}_j] \\
&+\epsilon\xi_j(\mathbf{x}_0)\psi_j(t_0)(\mathbf{x}_0-\bm{\mu}_j)(\mathbf{x}_0-\bm{\mu}_j)^T+\epsilon\frac{\xi_j(\mathbf{x}_0)}{2}\mathbf{\Sigma}_j = 0.
\end{aligned}
\end{equation}	

We first calculate the IF (denoted by $\mathcal{I}$ in the proof) when $\mathbf{\Sigma}_j=\mathbf{I}$ and $\bm{\mu}_j = \mathbf{0}$. Thus, we have $t = \mathbf{x}^T\mathbf{x}$ and $t_0 = \mathbf{x}_0^T\mathbf{x}_0$ in the following proof. Then, according to the definition of IF, we differentiate \eqref{appendIF1} with respect to $\epsilon$ and when it approaches $0$, we arrive at
\begin{equation}\label{appendIF2}
\begin{aligned}
&\mathbb{E}\big[\frac{\partial \xi_j(\mathbf{x})}{\partial \epsilon}|_{\epsilon = 0}\psi_j(t)\mathbf{x}\mathbf{x}^T \\
&\!+\! \xi_j(\mathbf{x})\frac{\partial\psi_j(\mathbf{x}^T\mathbf{\Sigma}_j^{-1}\mathbf{x})}{\partial \epsilon}|_{\epsilon=0}\mathbf{x}\mathbf{x}^T\!+\!\frac{1}{2}\frac{\partial \xi_j(\mathbf{x})}{\partial \epsilon}|_{\epsilon = 0}\mathbf{I}\!+\!\frac{1}{2}\xi_j(\mathbf{x})\mathcal{I}\big]\\
&+ \xi_j(\mathbf{x}_0)\psi_j(t_0)\mathbf{x}_0\mathbf{x}_0^T+\frac{1}{2}\xi_j(\mathbf{x}_0)\mathbf{I} = 0.
\end{aligned}
\end{equation}

In addition, we obtain
\begin{equation}\label{appendIF3}
\begin{aligned}
\frac{\partial \xi_j(\mathbf{x})}{\partial \epsilon}|_{\epsilon = 0} &= (\xi_j(\mathbf{x}) - \xi_j^2(\mathbf{x}))\cdot(-\frac{1}{2}\mathrm{tr}(\mathcal{I})-\psi_j(t)\mathbf{x}^T\mathcal{I}\mathbf{x}),\\
\frac{\partial \xi_j(\mathbf{x})}{\partial \epsilon}|_{\epsilon = 0} &= -\psi'_j(t)\mathbf{x}^T\mathcal{I}\mathbf{x}.
\end{aligned}
\end{equation}

By combining \eqref{appendIF2} and \eqref{appendIF3}, we arrive at
	\begin{equation}\label{appendIF4}
	\begin{aligned}
	\mathbb{E}\big[(\xi_j(\mathbf{x}) \!-\! \xi_j^2(\mathbf{x}))&\!\cdot\!(-\frac{1}{2}\mathrm{tr}(\mathcal{I})\!-\!\psi_j(t)\mathbf{x}^T\mathcal{I}\mathbf{x})\!\cdot\!(\psi_j(t)\mathbf{x}\mathbf{x}^T\!+\!\frac{1}{2}\mathbf{I})\\
	&-\xi_j(\mathbf{x})\psi'_j(t)(\mathbf{x}^T\mathcal{I}\mathbf{x})\mathbf{x}\mathbf{x}^T+\frac{1}{2}\xi_j(\mathbf{x})\mathcal{I}\big]\\
	&~~~~+ \xi_j(\mathbf{x}_0)\psi_j(t_0)\mathbf{x}_0\mathbf{x}_0^T+\frac{1}{2}\xi_j(\mathbf{x}_0)\mathbf{I} = 0.
	\end{aligned}
	\end{equation}

It should be pointed out that $(\mathbf{x}-\bm{\mu})^T\mathbf{\Sigma}^{-1}(\mathbf{x}-\bm{\mu})$ has the same distribution as $\mathcal{R}^2$. It is thus independent of $\frac{\Sigma^{-\nicefrac{1}{2}}(\mathbf{x}-\bm{\mu})}{\sqrt{(\mathbf{x}-\bm{\mu})^T\mathbf{\Sigma}(\mathbf{x}-\bm{\mu})}}$ (denoted by $\mathbf{u}$), which has the same distribution as $\mathcal{S}$ (i.e., uniform distribution). For mixing components, when data are well-separated, those $\mathbf{x}$ that do not belong to the $j$-th cluster have extremely low $\xi_j(\mathbf{x})$. In other words, the expectation in \eqref{appendIF4} is dominated by the expectation of the data which belong to the $j$-th cluster. Therefore, to calculate the expectation, we can treat the quadratic term $(\mathbf{x}-\mathbf{0})^T\mathbf{I}(\mathbf{x}-\mathbf{0})=\mathbf{x}^T\mathbf{x}$ (i.e., $t$) as independent of the normalised term $\frac{\mathbf{I}^{-\nicefrac{1}{2}}(\mathbf{x}-\mathbf{0})}{\sqrt{t}} = \frac{\mathbf{x}}{\sqrt{t}}$ (i.e., $\mathbf{u}$) of the $j$-th cluster.

Based on this approximation, we can rewrite \eqref{appendIF4} as
\begin{equation}\label{appendIF5}
\begin{aligned}
\mathbb{E}\big[(\xi_j(t) \!-\! \xi_j^2(t))&\!\cdot\!(\!-\!\frac{1}{2}\mathrm{tr}(\mathcal{I})\!-\!\psi_j(t)t\cdot\mathbf{u}^T\mathcal{I}\mathbf{u})\!\cdot\!(\psi_j(t)t\!\cdot\!\mathbf{u}\mathbf{u}^T\!+\!\frac{1}{2}\mathbf{I})\\
&\!-\!\xi_j(t)\psi'_j(t)t^2(\mathbf{u}^T\mathcal{I}\mathbf{u})\mathbf{u}\mathbf{u}^T\!+\!\frac{1}{2}\xi_j(t)\mathcal{I}\big]\\
&~~~~+ \xi_j(\mathbf{x}_0)\psi_j(t_0)\mathbf{x}_0\mathbf{x}_0^T+\frac{1}{2}\xi_j(\mathbf{x}_0)\mathbf{I} = 0.
\end{aligned}
\end{equation}	

Moreover, as $\mathbf{u}$ is uniformly distributed, we have $\mathbb{E}[\mathbf{u}\mathbf{u}^T]=\frac{1}{M}\mathbf{I}$, $\mathbb{E}[\mathbf{u}^T\mathcal{I}\mathbf{u}]=\frac{1}{M}\mathrm{tr}(\mathcal{I})$ and $\mathbb{E}[(\mathbf{u}^T\mathcal{I}\mathbf{u})\mathbf{u}\mathbf{u}^T]=\frac{1}{M(M+1)}(\mathcal{I}+\mathrm{tr}(\mathcal{I})\mathbf{I})$. Thus, we arrive at
\begin{equation}\label{appendIF6}
\begin{aligned}
&\!-\! \frac{(\mathbb{E}\big[(\xi_j(t) \!-\! \xi_j^2(t))\psi_j^2(t)t^2\big]+\mathbb{E}\big[\xi_j(t)\psi'_j(t)t^2\big])(\mathcal{I}+\mathrm{tr}(\mathcal{I})\mathbf{I})}{M(M+1)}\\
&\!-\!\frac{\mathbb{E}\big[(\xi_j(t) \!-\! \xi_j^2(t))\psi_j(t)t\big]\mathrm{tr}(\mathcal{I})}{2M}\mathbf{I} \!-\!\frac{\mathbb{E}[(\xi_j(t) \!-\! \xi_j^2(t))]\mathrm{tr}(\mathcal{I})}{4}\mathbf{I} \\
&- \frac{1}{2M}\mathbb{E}[(\xi_j(t) - \xi_j^2(t))\psi_j(t)t]\mathrm{tr}(\mathcal{I})\mathbf{I}+\frac{\pi_j}{2}\mathcal{I}\\
&+ \xi_j(\mathbf{x}_0)\psi_j(t_0)\mathbf{x}_0\mathbf{x}_0^T+\frac{1}{2}\xi_j(\mathbf{x}_0)\mathbf{I} = 0.
\end{aligned}
\end{equation}

With $w_1$ and $w_2$ in \eqref{defini}, we can re-write \eqref{appendIF6} as
\begin{equation}\label{appendIF7}
w_2\mathcal{I} = w_1 \mathrm{tr}(\mathcal{I})\mathbf{I}-\xi_j(\mathbf{x}_0)\psi_j(t_0)\mathbf{x}_0\mathbf{x}_0^T-\frac{1}{2}\xi_j(\mathbf{x}_0)\mathbf{I}.
\end{equation}

Then, by taking the trace on both sides of \eqref{appendIF7}, we have
\begin{equation}
\mathrm{tr}(\mathcal{I}) = \frac{\xi_j(\mathbf{x}_0)\psi_j(t_0)\mathbf{x}_0\mathbf{x}_0^T+\frac{1}{2}\xi_j(\mathbf{x}_0)\mathbf{I}}{Mw_1 - w_2}
\end{equation}

Thus, the IF at point $\mathbf{x}_0$ of the $j$-th cluster, when $\mathbf{\Sigma}_j = \mathbf{I}$, can be obtained as
\begin{equation}
\begin{aligned}
\mathcal{I}(\mathbf{x}_0) = &\left[\frac{2w_1\cdot \xi_j(\mathbf{x}_0)\psi_j(\mathbf{x}_0^T\mathbf{x}_0)\mathbf{x}_0^T\mathbf{x}_0 + w_2 \cdot\xi_j(\mathbf{x}_0)}{2(Mw_1-w_2)w_2}\right]\cdot \mathbf{I} \\
&- \frac{\xi_j(\mathbf{x}_0)\psi_j(\mathbf{x}_0^T\mathbf{x}_0)}{w_2}\mathbf{x}_0\mathbf{x}_0^T.
\end{aligned}
\end{equation}

The IF is then obtained at point $\mathbf{x}_0$ of the $j$-th cluster for general $\mathbf{\Sigma}_j$ and $\bm{\mu}_j$ according to its affine equivalence (i.e., $\mathcal{I}_{\mathbf{\Sigma}_j}(\mathbf{x}_0) = \mathbf{\Sigma}_j^{\nicefrac{1}{2}} \mathcal{I}(\mathbf{\Sigma}_j^{-\nicefrac{1}{2}}(\mathbf{x}_0-\bm{\mu}_j))\mathbf{\Sigma}_j^{\nicefrac{1}{2}}$).

This completes the proof of Lemma 3.
\subsection{Proof of Lemma 4}\label{prooflemma4}
	Similar to the proof of Lemma 3, we have the following equation for estimating $\bm{\mu}_j$ with contaminated distribution $\mathcal{F} = (1-\epsilon)\mathcal{F}_{\mathbf{x}}+\epsilon\mathcal{F}_{\mathbf{x}_0}$:
\begin{equation}\label{appendIFmean1}
\begin{aligned}
(1-\epsilon)\mathbb{E}[\xi_j(\mathbf{x})\psi_j(t)\mathbf{\Sigma}_j^{-1}&(\mathbf{x}-\bm{\mu}_j)]\\
&\! +\! \epsilon \xi_j(\mathbf{x}_0)\psi_j(t_0)\mathbf{\Sigma}_j^{-1}(\mathbf{x}_0-\bm{\mu}_j)=0
\end{aligned}
\end{equation}
Note that in \eqref{appendIFmean1}, for simplicity we also use $t$ to denote $(\mathbf{x}-\bm{\mu}_j)^T\mathbf{\Sigma}_j^{-1}(\mathbf{x}-\bm{\mu}_j)$ and $t_0$ to denote $(\mathbf{x}_0-\bm{\mu}_j)^T\mathbf{\Sigma}_j^{-1}(\mathbf{x}_0-\bm{\mu}_j)$. We here also utilise $\mathcal{I}$ to denote the IF in this proof.

In addition, by differentiating with $\epsilon$ and $\epsilon \rightarrow 0$, we arrive at
\small
\begin{equation}
\begin{aligned}
&\mathbb{E}\big[\frac{\partial \xi_j(\mathbf{x})}{\partial \epsilon}|_{\epsilon = 0}\psi_j(t)\mathbf{\Sigma}_j^{-1}(\mathbf{x}-\bm{\mu}_j)\\
&~~~~+\xi_j(\mathbf{x})\frac{\partial \psi_j(t)}{\partial \epsilon}|_{\epsilon = 0}\mathbf{\Sigma}_j^{-1}(\mathbf{x}-\bm{\mu}_j) - \xi_j(\mathbf{x})\psi_j(t)\mathbf{\Sigma}_j^{-1}\mathcal{I}\big]\\
&~~~~~~~~~~~~~~~~~~~~~~~~~~~ + \xi_j(\mathbf{x}_0)\psi_j(t_0)\mathbf{\Sigma}_j^{-1}(\mathbf{x}_0-\bm{\mu}_j) = 0.
\end{aligned}
\end{equation}
\normalsize

Besides, we also calculate 
\begin{equation}\label{appendIFmean2}
\begin{aligned}
\frac{\partial \xi_j(\mathbf{x})}{\partial \epsilon}|_{\epsilon = 0} & = (\xi_j(\mathbf{x})-\xi_j^2(\mathbf{x}))\cdot (-2\psi_j(t)(\mathbf{x}-\bm{\mu}_j)^T\mathbf{\Sigma}_j^{-1}\mathcal{I})\\
\frac{\partial\psi_j(\mathbf{x})}{\partial \epsilon}|_{\epsilon = 0} & = -2\psi'_j(t)(\mathbf{x}-\bm{\mu}_j)^T\mathbf{\Sigma}_j^{-1}\mathcal{I}.
\end{aligned}
\end{equation}

When data are well separated, we can assume that $t = (\mathbf{x}-\bm{\mu}_j)^T\mathbf{\Sigma}_j^{-1}(\mathbf{x}-\bm{\mu}_j)$ is independent of $\mathbf{u} = \frac{\mathbf{\Sigma}_j^{-\nicefrac{1}{2}}(\mathbf{x}-\bm{\mu}_j)}{\sqrt{t}}$. Therefore, \eqref{appendIFmean2} becomes 
\begin{equation}
\begin{aligned}
&2\!\cdot\!\mathbb{E}\big[(\xi_j(\mathbf{x})-\xi_j^2(\mathbf{x}))\psi_j^2(t)t\big]\!\cdot\!\mathbb{E}\big[\mathbf{u}^T\mathbf{\Sigma}_j^{-\frac{1}{2}}\mathcal{I}\mathbf{u}\big] \\
&\!+\!2\!\cdot\!\mathbb{E}\big[\xi_j(\mathbf{x})\psi'_j(t)t\big]\!\cdot\!\mathbb{E}\big[\mathbf{u}^T\mathbf{\Sigma}_j^{-\frac{1}{2}}\mathcal{I}\mathbf{u}\big]\! +\! \mathbb{E}[\xi_j(\mathbf{x})\psi_j(t)]\mathbf{\Sigma}_j^{-\frac{1}{2}}\mathcal{I}\\
& = \xi_j(\mathbf{x}_0)\psi_j(t_0)\mathbf{\Sigma}_j^{-\frac{1}{2}}(\mathbf{x}_0-\bm{\mu}_j),
\end{aligned}
\end{equation}
which yields $\mathbb{E}[\mathbf{u}^T\mathbf{\Sigma}_j^{-\frac{1}{2}}\mathcal{I}\mathbf{u}] = \frac{1}{M}\mathbf{\Sigma}_j^{-\frac{1}{2}}\mathcal{I}$. Thus, we arrive at 
\small
\begin{equation}
\mathcal{I}\! =\! \frac{\xi_j(\mathbf{x}_0)\psi_j(t_0)(\mathbf{x}_0-\bm{\mu}_j)}{\frac{2}{M}\mathbb{E}[(\xi_j(\mathbf{x})-\xi_j^2(\mathbf{x}))\psi_j^2(t)t]\! +\! \frac{2}{M}\mathbb{E}[\xi_j(\mathbf{x})\psi'_j(t)t]\! +\! \mathbb{E}[\xi_j(\mathbf{x})\psi_j(t)]}.
\end{equation}
\normalsize

This completes the proof of Lemma 4.
%\section{Properties of the distributions evaluated in experiments}

\bibliographystyle{IEEEtran}
\bibliography{IEEEfull,ShengxiLi}
\end{document}